\documentclass{article}

\usepackage[accepted]{icml2022}

\usepackage[ruled,vlined,algo2e]{algorithm2e}
\usepackage[utf8]{inputenc} % allow utf-8 input
\usepackage[T1]{fontenc} % use 8-bit T1 fonts
\usepackage[colorlinks]{hyperref}  % hyperlinks
\usepackage{url}   % simple URL typesetting
\usepackage{booktabs}  % professional-quality tables
\usepackage{amsfonts}  % blackboard math symbols
\usepackage{nicefrac}  % compact symbols for 1/2, etc.
\usepackage{microtype}  % microtypography
\usepackage{amsmath}
\usepackage{float}
\usepackage{graphicx}
\usepackage{subcaption}
\usepackage{amsthm}
\usepackage{verbatim}
\usepackage{csquotes}
\usepackage{color}
\usepackage{colordvi}
\usepackage{amssymb}
\usepackage{colordvi}
\usepackage{float}
\usepackage{graphicx}
\usepackage{subcaption}
\usepackage{enumerate}
\usepackage{enumitem}
\usepackage{physics}

\usepackage{bbm}
\usepackage{bigints}
\usepackage{soul}
\usepackage{pgffor}
\usepackage{tikz}
\usepackage[capitalize]{cleveref}

\usepackage{xcolor}
\definecolor{niceblue}{rgb}{0.10, 0.14, 0.76} 

\AtBeginDocument{%
\hypersetup{
    citecolor=niceblue,
    linkcolor=red,   
    urlcolor=niceblue}
}
\usepackage[%
    style=authoryear-comp,
    minnames=1,maxnames=99,maxcitenames=2,
    sorting=none,
    doi=false,url=false,
    giveninits=true,
    hyperref,natbib,backend=biber]{biblatex}

\renewbibmacro{in:}{%
  \ifentrytype{article}{}{\printtext{\bibstring{in}\intitlepunct}}}
\bibliography{sample}

\usepackage[hide=false,setmargin=true,marginparwidth=0.7in]{marginalia}

\newcommand{\dataset}{{\cal D}}
\newcommand{\reals}{\mathbb{R}}
\newcommand{\loss}{\mathcal{L}}
\newcommand{\Sphere}{\sqrt{d} \, \mathbb{S}^d}

\newcommand{\IL}{\mathcal{IL}}
\newcommand{\E}{\mathbb{E}}

\newcommand{\fracpartial}[2]{\frac{\partial #1}{\partial  #2}}
\newcommand{\bigO}{\mathcal{O}}
\newcommand{\appropto}{\mathrel{\vcenter{
  \offinterlineskip\halign{\hfil$##$\cr
    \propto\cr\noalign{\kern2pt}\sim\cr\noalign{\kern-2pt}}}}}

\usepackage{array}
\newcommand{\PreserveBackslash}[1]{\let\temp=\\#1\let\\=\temp}
\newcolumntype{C}[1]{>{\PreserveBackslash\centering}p{#1}}

\usepackage{algorithm}
\usepackage{algorithmic}
\usepackage{cases}
\usepackage{ifthen}
\theoremstyle{plain}

\newtheorem{theorem}{Theorem}

\newtheorem{lemma}{Lemma}
\newtheorem{corollary}{Corollary}
\newtheorem{definition}{Definition}

\newtheorem{approximation}{Approximation}

\usepackage[colorinlistoftodos]{todonotes}

\newtheorem{manthmin}{Theorem}
\newenvironment{manualthm}[1]{%
  \renewcommand\themanthmin{#1}%
  \manthmin
}{\endmanthmin}

\newtheorem{mancorin}{Corollary}
\newenvironment{manualcorollary}[1]{%
  \renewcommand\themancorin{#1}%
  \mancorin
}{\endmancorin}

\bibliography{sample}

\icmltitlerunning{Feature Learning and Signal Propagation in Deep Neural Networks}

\begin{document}

\twocolumn[

\icmltitle{Feature Learning and Signal Propagation in Deep Neural Networks}

\icmlsetsymbol{equal}{*}

\begin{icmlauthorlist}
\icmlauthor{Yizhang Lou}{stjc}
\icmlauthor{Chris Mingard}{dc,comp}
\icmlauthor{Yoonsoo Nam}{comp}
\icmlauthor{Soufiane Hayou}{sing}
\end{icmlauthorlist}

\icmlaffiliation{stjc}{St John's College, University of Oxford, Oxford, UK}
\icmlaffiliation{dc}{PTCL, University of Oxford, Oxford, UK}
\icmlaffiliation{comp}{Department of Physics, University of Oxford, UK}
\icmlaffiliation{sing}{Department of Mathematics, National University of Singapore}

\icmlcorrespondingauthor{Yizhang Lou}{yizhang.lou@sjc.ox.ac.uk}
\icmlcorrespondingauthor{Soufiane Hayou}{hayou@nus.edu.sg}


\icmlkeywords{Machine Learning, ICML}
\vskip 0.3in
]

\printAffiliationsAndNotice{}

\begin{abstract}
Recent work by \citet{baratin2021implicit} sheds light on an intriguing pattern that occurs during the training of deep neural networks: some layers \emph{align}
% \footnote{The \emph{alignment} is defined as the euclidean product of the tangent features matrix and the data labels matrix. See \cref{sec:tangent_features} for more details.}
much more with data compared to other layers
(where the \emph{alignment} is defined as the euclidean product of the tangent features matrix and the data labels matrix). 
The curve of the alignment as a function of layer index (generally) exhibits an ascent-descent pattern where the maximum is reached for some hidden layer.
%Modern Deep Neural Networks (DNNs) exhibit impressive generalization properties on a variety of tasks without any regularization method, suggesting the existence of hidden regularization effects. Recent work by \citet{baratin2021implicit} sheds light on an intriguing implicit regularization effect, showing that some layers are much more \emph{aligned} with data labels than other layers. 
In this work, we provide the first explanation for this phenomenon. We introduce the \emph{Equilibrium Hypothesis} which connects this alignment pattern to signal propagation in deep neural networks. Our experiments demonstrate an excellent match with the theoretical predictions.
% The empirical success of modern Deep Neural Networks (DNNs) has sparked a growing interest in the theoretical understanding of these models. Even without an additional regularization term, DNNs exhibit good generalization on a variety of tasks, suggesting the existence of hidden regularization effects. 
% An interesting topic in this direction is the study of \emph{implicit} regularization induced by the training algorithm. Recent work by \cite{baratin2021implicit} sheds light on an intriguing implicit regularization effect through the lens of tangent features \citep{jacot2018NTK}, and showed that some layers are much more \emph{aligned} with data labels compared to other layers. 
% This suggests that as the network grows in depth and width, an implicit \emph{layer selection} phenomenon occurs during training. However, no theoretical explanation for this effect can be found in the literature. In this work, we use a hybrid analysis that combines signal propagation theory and feature evolution during training to provide the first explanation for this \emph{alignment hierarchy}. We introduce and empirically validate the \emph{Equilibrium Hypothesis} which states that the layers that achieve some balance between forward and backward information loss are the ones with the highest alignment to data labels.
\end{abstract}
\vspace{-1em}

\section{Introduction}
The empirical success of modern Deep Neural Networks (DNNs) has sparked a growing interest in the theoretical understanding of these models. An important development in this direction was the introduction of the Neural Tangent Kernel (NTK) framework by \citet{jacot2018NTK}, which provides a dual view of Gradient Descent (GD) in function space. The NTK is the dot product kernel of Tangent Features (gradient features), given by
% \vspace{-0.1cm}
$$
K(x, x') = \nabla_{\theta} f(x) \nabla_{\theta} f(x')^T,
$$
%\vspace{-0.1cm}
%\\ &= \sum_{l=1}^L \nabla_{\theta_l} f_\theta(x) \nabla_{\theta_l} f_\theta(x')^T \in \mathbb R^{o \times o}.
where $f$ is the network output and $\theta$ is the vector of model parameters. The NTK has been the subject of an extensive body of literature, both in the NTK regime where the NTK remains constant during training, e.g. \citep{jacot2018NTK, jacot2020asymptotic, hayou_ntk, ghorbani2019investigation, yang_tensor3_2020, yang2019tensor_i}, and when the NTK changes during training, e.g \citep{yang2021tensor_iv, baratin2021implicit}. The latter is called the \emph{feature learning} regime since the tangent features $\nabla_{\theta} f(x)^T$ evolve during training in some data-dependent directions; we say that tangent features and the NTK \emph{adapt} to the data. A simple way to quantify \emph{how much these features adapt to data} is by measuring the \emph{alignment} between the tangent features and data labels \citep{baratin2021implicit}, which is given by the normalized euclidean product between the tangent kernel matrix $\hat{K}$ and data labels matrix $YY^T$ (see \cref{sec:tangent_features}). 
\begin{figure}[t]
\centering
\begin{subfigure}{.2\textwidth}
  \centering
  \includegraphics[width=\textwidth]{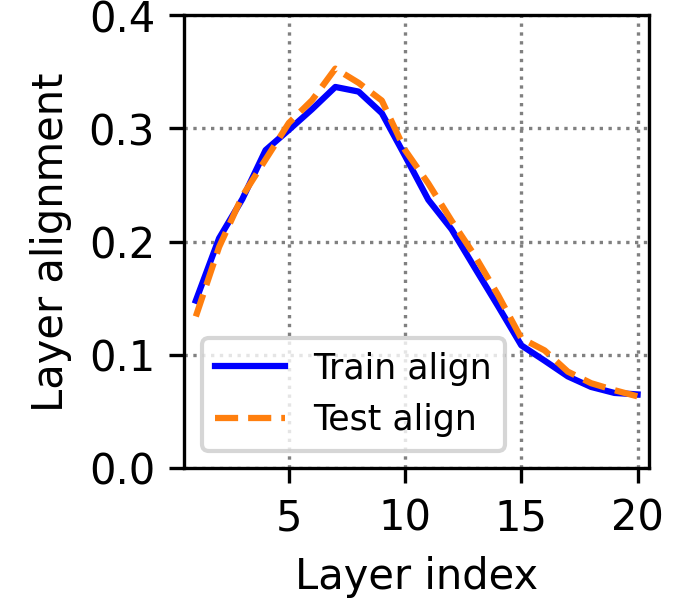}
\end{subfigure}
\begin{subfigure}{.2\textwidth}
  \centering
  \includegraphics[width=\textwidth]{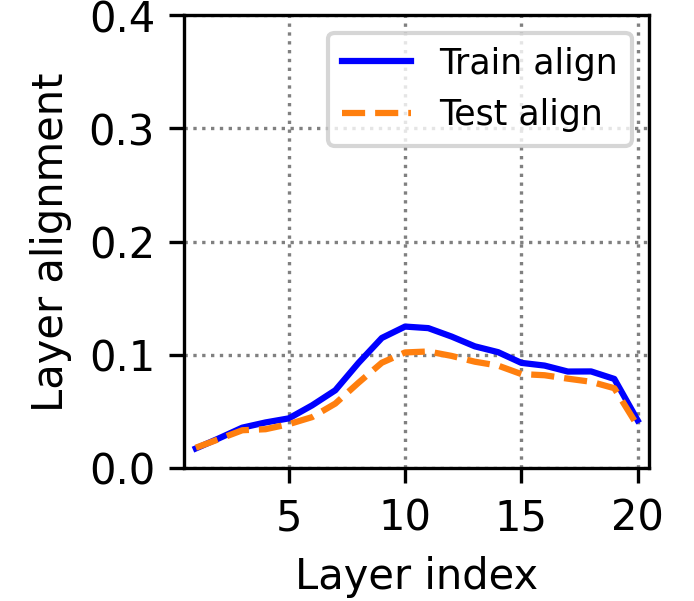}
\end{subfigure}
\caption{
%\small
\small{Example Alignment Hierarchies for 20 layer FFNNs trained on Fashion MNIST (L) and CIFAR10 (R).}}
\vspace{-0.5cm}
\label{fig:intro}
\end{figure}
% In the context of DNNs, another important aspect is how each layer adapts to the data during training. Indeed, for a depth $L$ network, the NTK can be expressed as a sum of the layer-wise NTK kernels (see \cref{sec:tangent_features})
% $
% \mathbf{K} = \sum_{l=1}^L \mathbf{K}^{l},
% $
% where $\mathbf{K}(x, x') = \nabla_{\theta_l} f(x) \nabla_{\theta_l} f(x')^T$ and $\theta_l$ is the collection of the parameters in layer $l$. 
\paragraph{Role of hidden layers.}
In the context of DNNs, hidden layers act as successive embeddings that evolve in data-dependent directions during training. However, it is still unclear how different layers contribute to model performance. 
A possible approach in this direction is the study of feature learning in each layer, i.e. how features adapt to data as we train the model. A simple way to do this is by measuring the change in the values of pre-activations as we train the network; this was used by \citet{yang2021tensor_iv} to engineer a network parameterization that maximizes feature learning. To capture the dynamics of GD, \citet{baratin2021implicit} studied the change in tangent features instead of pre-activations. For each layer $l$, they measured the alignment between the layer tangent kernel matrix $\hat{K}_l$ (where $K_l(x,x') = \nabla_{\theta_l} f(x) \nabla_{\theta_l} f(x')^T$, $\theta_l$ being the vector of parameters in the $l^{th}$ layer). The kernel $\hat{K}_l$ can be seen as the NTK of the network if only the parameters of the $l^{th}$ layer are allowed to change (other layers are frozen, see \cref{sec:quantifying_role} for a detailed discussion). The authors demonstrated the existence of an alignment pattern in which the tangent features of some hidden layers are significantly more aligned with data labels compared to other layers. We call this pattern the \emph{Alignment Hierarchy} (illustrated in \cref{fig:intro}). In general, the alignment reaches its maximum for some hidden layer and tends to be minimal in the external layers (first and last). 
To the best of our knowledge, no explanation has been provided for this phenomenon in the literature. 
In this paper, we propose an explanation based on the theory of signal propagation in randomly initialized neural networks. Our intuition is based on the observation that tangent kernels can be decomposed as an Hadamard product of two quantities linked to how signal propagates in DNNs.

\paragraph{Forward-Backward\hspace{0.1cm}Decomposition.} We show in \cref{sec:EH} that for a depth $L$ neural network, with proper normalization, the $l^{th}$ layer tangent kernel matrix can be written as
\begin{equation}\label{eq:intro_decomposition}
\hat{K}_l =  \overrightarrow{K}_l \circ \overleftarrow{K}_l
\end{equation}
where $\overrightarrow{K}_l$ is a kernel that depends only on the $l$ first layers, and $\overleftarrow{K}_l$ is a kernel that is essentially governed by the last $L-l+1$ layers. The $\circ$ denotes the Hadamard product. Intuitively, this decomposition suggests that tangent features are the result of the collaboration between the \emph{forward} kernel $\overrightarrow{K}_l$ and the \emph{backward} kernel $\overleftarrow{K}_l$. For DNNs, it is expected that depth has some layer-dependent effect on the kernels $\overrightarrow{K}_l$ and $\overleftarrow{K}_l$ and therefore on $K_l$. Understanding the effect of depth on how information propagates has been the subject of a large body of work which constitute what is known as the theory of signal propagation in DNNs. We briefly introduce this theory in the next paragraph (a more detailed discussion is provided in \cref{sec:EH}).
\vspace{-0.3cm}
\paragraph{Signal propagation in DNNs.} A recent line of research (e.g \citet{poole, samuel2017, lee_gaussian_process, yang_tensor3_2020, hayou2019impact, hayou_ntk, yang2021tensor_iv}) studied the dynamics of signal propagation in randomly initialized DNNs. Under mild conditions, in the infinite (layer) width limit, each neuron $y(.)$ in the network converges in distribution to a Gaussian process (GP) \citep{neal, matthews}. Hence, in this limit, the covariance kernel captures all the properties of the neurons at initialization \citep{neal, lee_gaussian_process}\footnote{GPs are fully characterized by their covariance kernel.}. From a geometric perspective, the correlation/covariance measures the angular distortion between vectors. Hence, the covariance represents a \emph{geometric information}\footnote{Here, the information refers purely to the covariance between two vectors, and is different from the information-theoretic definition of information.}. As shown in \cref{sec:EH}, the kernels $\overrightarrow{K}_l$ and $\overleftarrow{K}_l$ are covariance kernels; $\overrightarrow{K}_l$, resp. $\overleftarrow{K}_l$, represents a notion of forward, resp. backward, geometric information (covariance). Inspired by this observation, we provide an explanation of the feature alignment pattern based on how the geometric information, encoded in $\overrightarrow{K}_l$ and $\overleftarrow{K}_l$, changes with depth. Notably, we prove that these geometric information become degenerate in the limit of large depth which translates to the information being lost as we increase depth. We say that there is a \emph{information loss} as we increase depth and we characterize this loss in the case of fully-connected DNNs. Could this information loss be the reason behind the observed alignment pattern? More precisely, could some notion of balance between information loss in kernels $\overrightarrow{K}_l$ and $\overleftarrow{K}_l$ explain the alignment pattern? We formulate this intuition as the \emph{Equilibrium hypothesis} (EH) which we introduce in \cref{sec:EH}. 
% To explain these findings, we introduce the Equilibrium Hypothesis which conjectures that this effect is a result of a notion of balance between forward and backward information propagation at initialization.
\paragraph{Our contributions are three-fold.} Firstly, we introduce and empirically validate the Equilibrium Hypothesis, which provides an explanation for the alignment pattern. More precisely, we give an explanation for the fact that the alignment peaks at some hidden layer. Secondly, we provide a comprehensive analysis of this hypothesis in the case of fully-connected neural networks. Most notably, we prove that layers with indices $\mathbf{l = \Theta(L^{3/5})}$ achieve a notion of \emph{equilibrium} in geometric information in the limit of large depth $L$. Our experiments yield excellent match between theoretical and empirical results. Finally, we provide an empirical analysis of the connection between the Alignment Hierarchy (illustrated in \cref{fig:intro}) and the generalization error. 
% we discuss some algorithmic implications of the Equilibrium Hypothesis.
\section{Feature Learning in DNNs}\label{sec:tangent_features}
Consider a neural network model consisting of $L$ layers of widths $(N_l)_{1 \leq l \leq L}$, $N_0 = d$, and let $\theta = (\theta_l)_{1 \leq l \leq L}$ be the flattened vector of weights indexed by the layer's index, and $P$ be the dimension of $\theta$. Given an input $x \in \reals^d$, the network is described by the set of equations
$$
z_l(x) = \mathcal{F}_l(\theta_l, z_{l-1}(x)), \quad 1\leq l\leq L,
$$
where $\mathcal{F}_l$ is a mapping that defines the $l^{th}$ layer, e.g. fully-connected, convolutional, etc.\\
% a Lipschitz, twice differentiable nonlinearity function $\phi: \mathbb{R} \rightarrow \mathbb{R}$, with bounded second derivative. We focus on
% This paper focuses on the ANN realization function $F^{(L)}: \mathbb{R}^{P} \rightarrow \mathcal{F}$, mapping parameters $\theta$ to functions $f_{\theta}$ in function space $\mathcal{F}$.
The network output function $f$ is given by $f_\theta(x) = \nu(z_L(x)) \in \reals^o$ where $\nu : \mathbb{R}^{N_L} \rightarrow \mathbb{R}^o$ is a mapping of choice, and $o$ is the dimension of the output, e.g. the number of classes for a classification problem.\\
We consider a loss function $\mathbb{L}: \mathbb{R}^{o} \times \mathbb{R}^{o} \rightarrow \mathbb{R}$ and a dataset $\dataset = \{(x_1,y_1), \ldots, (x_{n},y_n)\}$. The network is then trained by minimizing the empirical loss $\loss: \mathbb{R}^P \to \mathbb{R}$ given by
\begin{align*}
\loss(\theta) = \frac{1}{n} \sum_{i =1}^{n} \mathbb{L}(f_{\theta}(x_i), y_i).
\end{align*}
% A single step full batch gradient descent with learning rate $\eta$ back propagated to the parameters $\theta$ could be written as:
% \begin{align*}
%     \theta(t+1) = \theta(t) - \eta \frac{\partial \mathcal{L}}{\partial \theta^T} 
% \end{align*}
% for time step $t = 0,1, \ldots$. 
\vspace{-0.5cm}
\paragraph{Tangent Features.} \cite{jacot2018NTK} introduced the Neural Tangent Kernel (NTK), which provides a dual view of the training procedure; it links gradient updates in parameter space to a kernel gradient descent in function space. The NTK is given by
\begin{equation}\label{eq:ntk_decomposition}
    \begin{aligned}
   K^L_{\theta}(x, x') &= \nabla_{\theta} f_\theta(x) \nabla_{\theta} f_\theta(x')^T\\ &= \sum_{l=1}^L \nabla_{\theta_l} f_\theta(x) \nabla_{\theta_l} f_\theta(x')^T \in \mathbb R^{o \times o}.
   \end{aligned}
\end{equation}
The tangent features are the feature maps of the NTK, given by the output gradients w.r.t the network parameters, namely
\begin{equation}
    \begin{aligned}
    \Psi_{\theta}(x):=\nabla_{\theta} f_{\theta}(x)^T \in \mathbb{R}^{P \times o}.
    \end{aligned}
\end{equation}
For the sake of simplicity, we remove  $\theta$ and $L$ in the kernel notation and define $\mathbf{\Psi} \in \mathbb{R}^{P \times on}$, the tangent feature matrix over the training dataset $\mathcal{D}$. $\mathbf{\Psi}$ is the horizontal concatenation of $\Psi(x_1), \ldots, \Psi(x_{n})$. The corresponding tangent kernel matrix is given by $\hat{\mathbf{K}}=\mathbf{\Psi}^T  \mathbf{\Psi} \in \mathbb{R}^{on \times on}$. 
\vspace{-0.3cm}
\subsection{Quantifying the role of each layer}\label{sec:quantifying_role}
\citet{lee2020finite,lee2019wide,valle2018deep,mingard2021sgd} demonstrated that neural network based kernel methods (e.g.\ infinite width NTK regime) can achieve near parity with finite width networks, and have near identical posterior distributions, over a range of architectures (e.g.\ LSTMs, WideResNet) and datasets (e.g.\ Cifar10).
However, optimizer hyperparameters are known to affect generalization, suggesting an extra layer of complexity \citep{bernstein2021implicit,mingard2021sgd}.
Furthermore, \citet{hayou_ntk} proved that the large depth limit of the NTK regime is trivial in the sense that the limiting NTK has rank 1. This suggests that this kernel regime, where tangent features are fixed at initialization, cannot explain the inductive bias of ultra deep neural networks, and that feature learning (tangent features evolve during training) could be the backbone of generalization in very deep networks. 
Finite width CNNs operating in the feature learning regime have also been shown to generalise better than their infinite width counterparts \citep{lee2020finite}.
A question that arises is that of the role of each layer in feature learning -- unfortunately there is no consensus on how feature learning should best be measured. One way to approach this question is by analyzing the behaviour of a network where only the parameters in a given layer are allowed to change with gradient updates. We call this approach `parameter freezing'.
\vspace{-0.3cm}
\paragraph{Parameter freezing.} 
The derivations of the results in this paragraph are provided in \cref{app:justif_tf}. We omitted the details in the main text to meet space constraints.\\
Consider a classification task with $o=k$ classes and assume that the dataset is balanced, i.e. $\frac{1}{n} \sum_{i=1}^n y_i \approx \frac{1}{k} \mathbf{1},$ where $\mathbf{1} \in \reals^k$ is the vector of ones.\\
\emph{\underline{Intuition}: One way to measure feature learning in the $l^{th}$ layer is by freezing the parameters in the other $L-1$ layers and tracking the change in network output when we update the parameters of the $l^{th}$ layer.}\\
Given a layer index $l$, suppose that we freeze all the weights in the other $L-1$ layers and allow the parameters in the $l^{th}$ layer to be updated with a gradient step. Consider the vector $f_t(X) \in \reals^{kn}$ which consists of the concatenation of the sequence $(f_t(x_i))_{1 \leq i \leq n}$ (here $X$ refers to the concatenation of $x_1, x_2, \dots, x_n$). Then, with one gradient step, the update $\delta f_t(X)$ is given by 
$$\delta f_t(X) = -\eta \hat{\mathbf{K}}_l (Z_t - Y),$$
where $\hat{\mathbf{K}}_l$ is the tangent kernel matrix for layer $l$, $Z_t: = (\textup{softmax}(f_t(x_i)))_{1 \leq i \leq n}, Y \in \reals^{on}$, and $\eta$ is the normalized learning rate (i.e. $\eta = \textup{LR}/n$). Hence, the kernel matrix $\hat{\mathbf{K}}_l$ controls the change in the vector $f_t(X)$. \\
To understand the interaction between $\hat{\mathbf{K}}_l$ and $Y$, let us see what happens at the first step of gradient descent. At initialization, the output function $f$ is random and has an average accuracy of an random classifier, i.e. a random guess with uniform probability $1/k$ for each class. In this case, the average update is given by
\begin{equation}\label{eq:dynamics_classif}
    \delta f_t(X) \approx -\eta \hat{\mathbf{K}}_l \left(\frac{1}{k} \mathbf{1} - Y\right) \approx \eta \hat{\mathbf{K}}_l \tilde{Y} ,
\end{equation}
where $\tilde{Y} = \left(I - \frac{1}{kn} \mathbf{1} \mathbf{1}^T\right) Y$. $\tilde{Y}$ is a centered version of $Y$. Using \cref{eq:dynamics_classif}, we obtain
$$
\|\delta f_t(X) \| \leq \eta \textup{Tr}(\hat{\mathbf{K}}_l) \|\tilde{Y}\|,
$$
with equality if and only if $\hat{\mathbf{K}}_l$ and $\tilde{Y}$ are aligned, i.e. $\hat{\mathbf{K}}_l \propto \tilde{Y} \tilde{Y}^T$. Hence, the maximum update of the network output is induced by a perfect alignment between $\hat{\mathbf{K}}_l$ and $\tilde{Y} \tilde{Y}^T$.\\
\emph{\underline{Conclusion}:} \emph{Assume that only parameters in the $l^{th}$ layer are updated with gradient descent. Then, the alignment between the tangent kernel matrix $\hat{\mathbf{K}}_l$ and the centered data labels matrix $\tilde{Y} \tilde{Y}^T$ controls the magnitude of change in $f_t(X)$ at early training.}\\
Although this analysis is performed at early training, we hypothesize that this alignment quantifies feature learning in each layer during training\footnote{ \cref{fig:align_example} shows that the increase in alignments occurs mostly during the first few epochs.}. Hence, we propose to use this alignment to quantify the role of each layer. This alignment can also be interpreted as a measure of how informative each layer's gradient update is. In fact, if we update all the layers (no parameter freezing), the change in function update projected on data labels vector is approximately
\begin{equation}\label{eq:dynamics_align}
    \langle\tilde{Y}, \delta f_t(X)\rangle \approx \eta \tilde{Y}^T \hat{\mathbf{K}} \tilde{Y} = \sum_{l} \eta \textup{Tr}(\hat{\mathbf{K}}_l \tilde{Y} \tilde{Y}^T),
\end{equation}
Hence, $\textup{Tr}(\hat{\mathbf{K}}_l \tilde{Y} \tilde{Y}^T)$ quantifies how gradient updates in each layer contribute to output function moving in the direction of the training target.

\subsection{Centered Kernel Alignment (CKA)} From the analysis in \cref{sec:quantifying_role}, we define the \emph{centered kernel alignment} between two kernel matrices $\mathbf{K}, \mathbf{K}^{\prime} \in \mathbb{R}^{on \times on}$ by
\begin{equation}
    \begin{aligned}
     A\left(\mathbf{K}, \mathbf{K}^{\prime}\right)=\frac{\operatorname{Tr}\left[\mathbf{K}_{c} \mathbf{K}_{c}^{\prime}\right]}{\left\|\mathbf{K}_{c}\right\|_{F}\left\|\mathbf{K}_{c}^{\prime}\right\|_{F}}
    \end{aligned}
\end{equation}
where $\mathbf{K}_{c} = \mathbf{C}\mathbf{K}\mathbf{C}$, $\mathbf{C} = \mathbf{I} - \frac{1}{on} \mathbf{1}\mathbf{1}^T$ is the centering matrix ($\mathbf{1}$ is a vector with all entries being $1$), and $\|.\|_F$ is the Frobenius norm. The CKA was used by \cite{baratin2021implicit} as a measure of feature learning.\\
\emph{Remark 1.} For all kernels $\mathbf{K}, \mathbf{K}^{\prime}$, we have $A\left(\mathbf{K}, \mathbf{K}^{\prime}\right) \in [0,1]$ with $A\left(\mathbf{K}, \mathbf{K}^{\prime}\right) = 1$ if and only if the kernel matrices are colinear.\\
To quantify the role of the $l^{th}$ layer ($1\leq l \leq L$), we use $\mathbf{K} = \hat{\mathbf{K}}_l = \mathbf{\Psi}^T_l \mathbf{\Psi}_l$ where $\mathbf{\Psi}_l \in \reals^{P_l \times on}$ are the tangent features of layer $l$ (the horizontal concatenation of the tangent features $(\Psi_l(x_i) = \nabla_{\theta_l}f_\theta(x_i))_{1\leq i \leq n}$), $P_l$ is the dimension of $\theta_l$, and $\mathbf{K}' = YY^T$ where $Y \in  \mathbb{R}^{on}$ is the horizontal concatenation of output vectors in the dataset $ \mathcal{D}$.
% we use  $\mathbf{K} = \mathbf{\Psi}^T \mathbf{\Psi}$, the tangent kernel matrix defined above, and $\mathbf{K}' = YY^T$ where $Y \in  \mathbb{R}^{on}$ is the horizontal concatenation of output vectors in the dataset $ \mathcal{D}$.
% % In this case, the CKA can be expressed in terms of dot products instead of the trace operator. 
% Let $\tilde{Y} = \mathbf{C}Y$ be the centered labels. Using the relation $\mathbf{C}^2 = \mathbf{C}$, the CKA is given by
% \begin{equation}
%     \begin{aligned}
%     A\left(\mathbf{K}, \mathbf{K}^{\prime}\right)
%     % =\frac{y^T \mathbf{C}\mathbf{\Psi}^T \mathbf{\Psi} \mathbf{C} y}{\left\|\mathbf{C} \mathbf{\Psi}^T \mathbf{\Psi} \mathbf{C}\right\|_{F}\left\|y\right\|^2} 
%     = \frac{\tilde{Y}^T \mathbf{\Psi}^T \mathbf{\Psi} \tilde{Y}}{\left\|\mathbf{C} \mathbf{\Psi}^T \mathbf{\Psi} \mathbf{C}\right\|_{F}\left\|\tilde{Y}\right\|^2}
%     \end{aligned}
% \end{equation}
% which simplifies to
% \begin{equation}
%     \begin{aligned}
%     A\left(\mathbf{K}, \mathbf{K}^{\prime}\right)=\frac{y^T \mathbf{\Psi}^T \mathbf{\Psi} y}{\left\|\mathbf{C} \mathbf{\Psi}^T \mathbf{\Psi} \mathbf{C}\right\|_{F}\left\|y\right\|^2}
%     \end{aligned}
% \end{equation}
\begin{equation}
    \begin{aligned}
     A_l:= A\left(\mathbf{K}_l, \mathbf{K}_l^{\prime}\right)=\frac{\tilde{Y}^T \mathbf{\Psi}_l^T \mathbf{\Psi}_l \tilde{Y}}{\left\|\mathbf{C} \mathbf{\Psi}_l^T \mathbf{\Psi}_l \mathbf{C}\right\|_{F}\left\|\tilde{Y}\right\|^2}
    \end{aligned}
\end{equation}
% This can also be written as 
% \begin{equation}
%     A_l =\frac{y^T \mathbf{\Psi}^T \mathbf{M}_l \mathbf{\Psi} y}{\left\|\mathbf{C} \mathbf{\Psi}^T \mathbf{M}_l \mathbf{\Psi} \mathbf{C}\right\|_{F}\left\|y\right\|^2}, \quad 1 \leq l \leq L,
% \end{equation}
% where $\mathbf{M}_l$ is a diagonal $0$-$1$ matrix, with the diagonal elements corresponding to the parameters in layer $l$ being equal to $1$ and $0$ otherwise. 
% In the following we will denote $\operatorname{CKA}$, $\operatorname{CKA}_l$ value at training step $t$ as $A(t)$, $A_l(t)$ respectively.\fTBD{$A_l$ you can just define the alignement using $A$ from the beginning instead of CKA}.
\begin{figure}[h]
\centering
\begin{subfigure}{.48\textwidth}
  \centering
  \includegraphics[width=\textwidth]{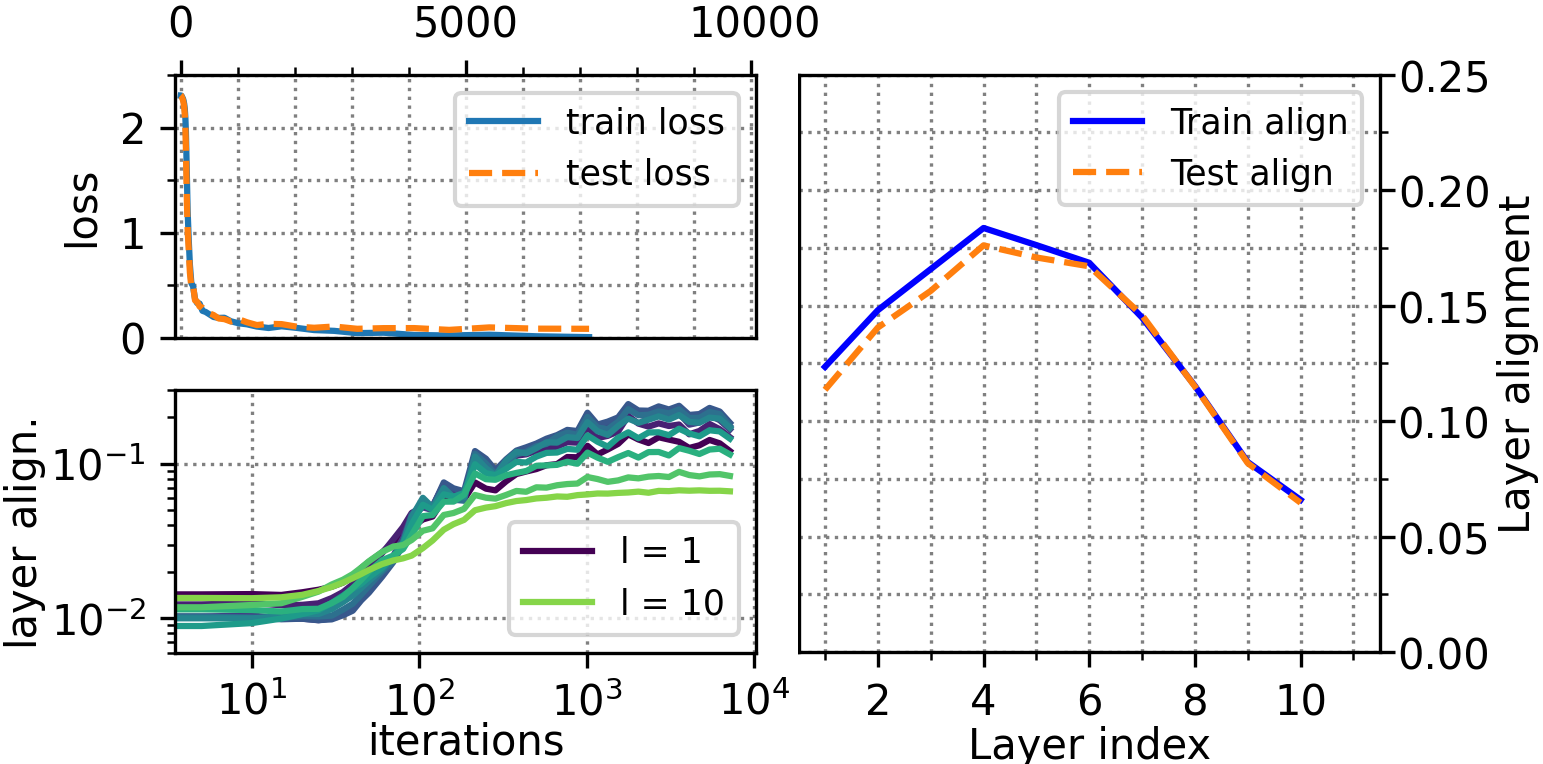}
  \caption{MNIST}
  \label{fig:sfig1_1}
\end{subfigure}
\begin{subfigure}{.48\textwidth}
  \centering
  \includegraphics[width=\textwidth]{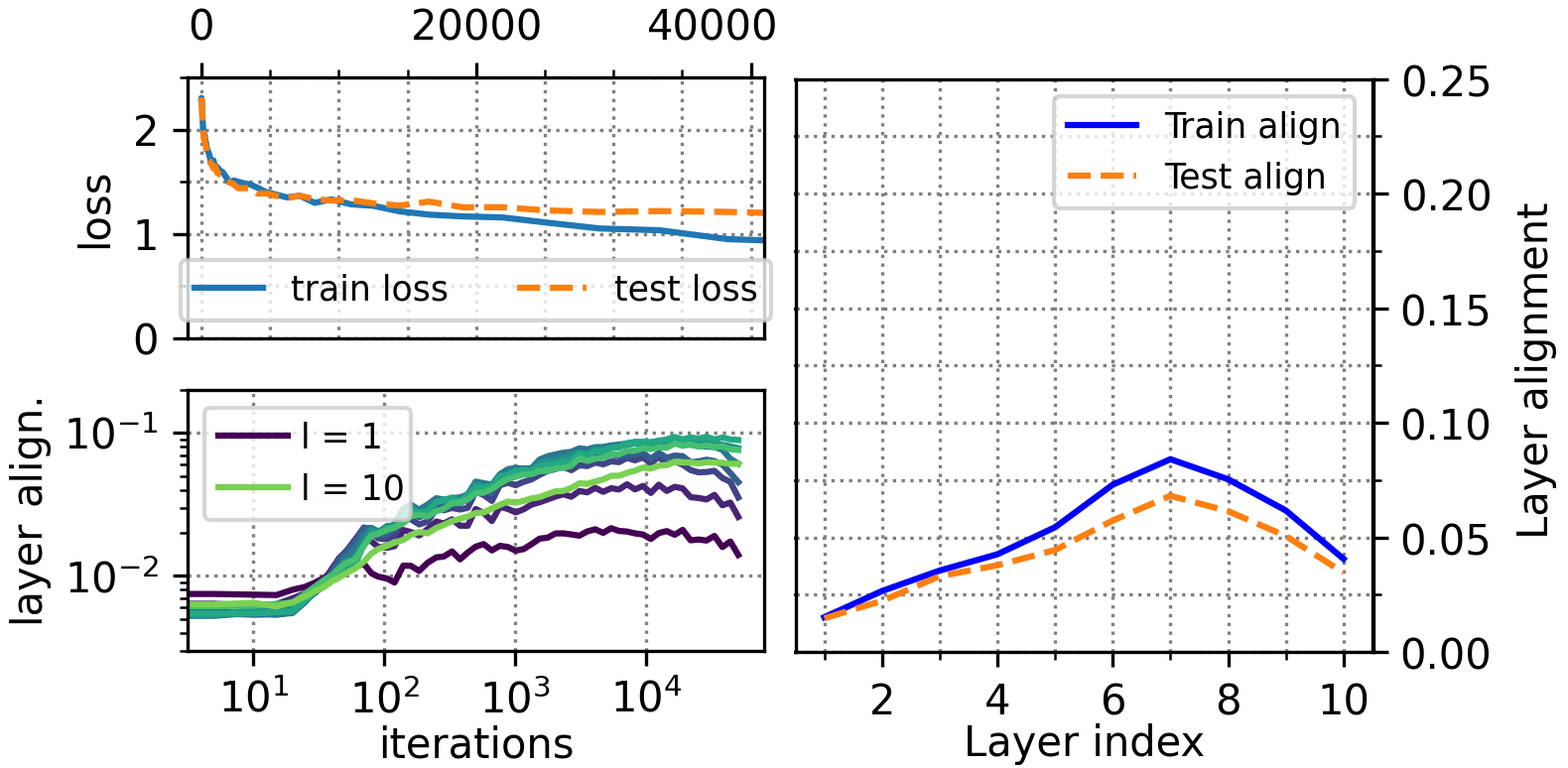}
  \caption{CIFAR10}
  \label{fig:sfig1_2}
\end{subfigure}
\caption{
%\small
\small{Layerwise alignment hierarchy for the MNIST and CIFAR10 datasets when trained on an FFNN with depth 10 and width 256. Left hand panels show progression of loss and layer alignment with iterations of SGD. Right hand panel shows layer alignment at the end of training.}}
\vspace{-0.5cm}
\label{fig:align_example}
\end{figure}

% \soufiane{Chris can you make this more clear? the conclusion is a bit unclear}
% \chris{I removed it as it distracts from your parameter freezing interpretation and spatial constraints}
% $\overrightarrow{K}_l$ has a straightforward interpretation -- it measures the angular separation of different inputs in the $l$'th layer's intermediate representation. Stronger `feature learning' would be associated with larger $A(\overrightarrow{K}_l,\tilde{Y}^T \tilde{Y})$. However,

% $K_l$ provides a fine-grained perspective of gradient descent: it measures how similar the direction (in parameter-space) of the greatest changes in $f(x_i)$ and $f(x_j)$. This measures, in a sense, how aligned $x_i$ and $x_j$ are -- if $K_l(x_i,x_j)$ is large, it means moving $\Delta\theta$ in parameter-space will mean $\Delta f(x_i;\theta)\sim \Delta f(x_j;\theta)$ -- when it is near zero it means $\Delta\theta$ exists where $\Delta f(x_i;\theta)$ can be much larger or smaller than $\Delta f(x_j;\theta)$.
% With this interpretation, `feature leaning' in the $l$'th layer means $l$'th layer parameters can be perturbed in directions which only affect one class.

\subsection{Alignment Hierarchy}(AH)
The alignment $A_l$ acts as a measure of how much layers contribute to the performance of the network (\cref{sec:quantifying_role}). \citet{baratin2021implicit} observed an interesting hierarchical structure in the alignments $A_l$ for different neural network architectures. During the course of training, the increase in $A_l$ for some middle layers is sharp and significantly larger than the alignments of other layers. We illustrate this pattern in
\cref{fig:align_example} on MNIST and CIFAR10 datasets with fully-connected networks. It appears that alignments of some layers increase much more effectively with gradient updates over others \footnote{This also suggests the existence of an implicit layer selection phenomenon during training}. 
We call this pattern the \emph{Alignment Hierarchy}, and aim to understand the reason why the alignment peaks at some hidden layer. For both datasets in \cref{fig:align_example}, the pattern is similar and shows large alignments for some middle layers. Further empirical results on K-MNIST and FashionMNIST datasets and VGG19/ResNet18 architectures are provided in \cref{app:further_experiments}. Motivated by this empirical findings, we formulate the Equilibrium Hypothesis in the next section, where we give an explanation of the Alignment Hierarchy using tools from the theory of signal propagation at initialization.
% \begin{itemize}
%     \item neural network model
%     \item define tangent features
%     \item show some examples of Tangent Alignment Hierarchy
% \end{itemize}
\section{The Equilibrium Hypothesis}\label{sec:EH}
In this section, we aim to understand a specific aspect of the AH: \emph{why does the alignment peak in some intermediate layer?} We argue that this is a result of the dynamics of signal propagation in DNNs at initialization. For the sake of simplicity, we restrict our theoretical analysis to fully-connected DNNs, although our results can in principle be extended to other architectures. 
\vspace{-0.3cm}
\paragraph{Fully-connected Feedforward Neural Network (FFNN).} 
Given an input $x \in \reals^d$, and a set of weights $(W_l)_{1 \leq L}$, we consider the following neural network model
\begin{equation}
\begin{aligned}
z_1(x) &= W_1 x\\
z_l(x) &= W_l \phi(z_{l-1}(x)), \quad 2 \leq l \leq L,
\end{aligned}
\label{eq:ffnn_1}
\end{equation}
where $W_{l} \in \mathbb{R}^{N \times N}$ with $o = N_{L} = 1$\footnote{For simplicity, we restrict our analysis to rectangular networks with 1D output networks.}, $W_1 \in \mathbb{R}^{N \times d}$, $N$ is the network width, and $\phi$ is the ReLU activation function given by $\phi(v) = (\max(v_i,0))_{1\leq i \leq p}$ for $v \in \reals^p$.
% The dimension of the parameter space is $P=\sum_{l=0}^{L-1}N_l N_{l+1}$ where we denote $N_0:=d$. 
For each layer, the weights are initialized with i.i.d Gaussian variables $W \sim \mathcal{N}(0,\frac{2}{\textup{fan\_in}})$, where `fan\_in' refers to the dimension of the previous layer. This standard initialization scheme is known as the He initialization \citep{he_relu} or the Edge of Chaos initialization \citep{poole,samuel2017, hayou2019impact}.
\subsection{Tangent Kernel decomposition} 
The tangent kernel at hidden layer $l$ can be expressed as
\begin{align*}
    K_l(x,x') &= \nabla_{\theta_l}f(x) \cdot \nabla_{\theta_l}f(x')\\
    &= \sum_{i,j} \phi(z_{l-1}^j(x))\phi(z_{l-1}^j(x')) \, \fracpartial{f}{z^{i}_l}(x)  \fracpartial{f}{z^{i}_l}(x').
\end{align*}
Since $K_l$ is a sum over $N^2$ terms, we consider the average kernel $\bar{K}_l$ given by 
$
\bar{K}_l = \frac{1}{N^2} K_l.
$
In matrix form\footnote{Bold characters in \cref{eq:kernel_forward_backward_decomposition} refer to kernel matrices and not kernels}, $\bar{K}_l$ can be written as the Hadamard product of two kernels 
\begin{equation}\label{eq:kernel_forward_backward_decomposition}
\bar{\mathbf{K}}_l = \overrightarrow{\mathbf{K}}_l \circ \overleftarrow{\mathbf{K}}_l,
\end{equation}
where $\overrightarrow{K}_l(x,x') = \frac{1}{N}\phi(z_{l-1}(x)) \cdot \phi(z_{l-1}(x'))$ is the \emph{forward} features kernel and $\overleftarrow{K}_l(x,x') = \frac{1}{N} \fracpartial{f_{l:L}}{z}(z_l(x)) \cdot \fracpartial{f_{l:L}}{z}(z_l(x'))$ is the \emph{backward} tangent features kernel, where $f_{l:L}$ is the function that maps the $l^{th}$ layer to the network output. The above decomposition illustrates the \emph{collaborative} roles played by kernels $\overrightarrow{K}$ and $\overleftarrow{K}$ in constructing the tangent features at layer $l$. To depict the role of each kernel, we use some tools from the theory of signal propagation at initialization.
\subsection{Signal propagation at initialization}
Consider an FFNN of type \eqref{eq:ffnn_1}. The weights $W$ are randomly initialized. Hence, the network neurons and output are random processes at initialization. Understanding the properties of such processes is crucial for both training and generalization \citep{samuel2017, hayou2019impact}.
It turns out that in the limit of infinite width $N \to \infty$, the neurons act as Gaussian processes. To see this, consider the simple case of a two layers FFNN. Since the weights are i.i.d, neurons $\{z_1^i(.)\}_{i\in[1:N]}$ are also iid Gaussian processes with covariance kernel given by
$\E_W[z_1^i(x) z_1^i(x')] = \frac{2 x \cdot x'}{d} $. Using the Central Limit Theorem, as $N \rightarrow \infty$, $z^i_{2}(x)$ is a Gaussian variable for any input $x$ and index $i \in [1:o]$. Moreover, the random variables $\{z^i_2(x)\}_{i \in [1:o]}$ are iid. Hence, the processes $z^i_{2}(.)$ can be seen as independent (across $i$), centered Gaussian processes with some covariance kernel $q_2$. This Gaussian process limit of FFNNs was first proposed by \citet{neal} in the single layer case and was extended to  multi-layer networks by \citet{matthews} where the authors showed that in $l^{th}$ layer, neurons become i.i.d Gaussian processes with covariance kernel $q_l$ in the limit $N \to \infty$. This result holds for all standard neural network architectures \citep{yang2019tensor_i}. A more complete review of this theory is provided in \cref{app:infinite_width_limit}. The covariance kernel $q_l(x,x')$ is a measure of the angular distortion between the vectors $z_l^i(x)$ and $z_l^i(x')$. Thus, the covariance kernel carries some information on how inputs propagate within the network. We formalize this notion of information in the next definition.
\begin{definition}[Geometric information]\label{def:geometric_information}
Given random weights $W$, we say that a kernel function $k$ is a geometric information if it can be expressed as $k(x,x') = \mathbb{E}_W[g(W,x)g(W,x')]$ for some function $g:\reals\times \reals^d \to \reals$.
\end{definition}
Hereafter, we simply use `information'\footnote{This is different from the information-theoretic definition of information.} to refer to the geometric information in \cref{def:geometric_information}. Recall the kernel decomposition given by \cref{eq:kernel_forward_backward_decomposition}
\begin{equation*}
\bar{\mathbf{K}}_l = \overrightarrow{\mathbf{K}}_l \circ \overleftarrow{\mathbf{K}}_l.
\end{equation*}
The kernels $\overrightarrow{K}_l$ and $\overleftarrow{K}_l$ depend on random weights $W$ and thus are random. We propose to study the average behaviour instead, where we consider the average kernels. For $\overrightarrow{K}_l$, the average kernel is given  $\E_W[\overrightarrow{K}_l(x,x')] = \E_W[\phi(z^1_{l-1}(x)) \phi(z^1_{l-1}(x'))]$ (since the neurons $z_l^j$ are identically distributed) which represents a geometric information as per \cref{def:geometric_information}. We call this average kernel the forward information. A standard result in signal propagation is that kernels $\overrightarrow{K}_l$ and $\overleftarrow{K}_l$ converge to their corresponding expected value in the limit of infinite width \citep{yang_tensor3_2020,samuel2017, hayou2019impact} which justifies our choice of the average kernel $\E_W[\overrightarrow{K}_l(x,x')]$. A similar result holds for $\overleftarrow{K}_l$. Let us formalize these definitions.
\begin{definition}
Given a layer index $l$, we define the forward information $I^f_{l,N}$ by
\begin{align*}
I^f_{l, N}(x,x') = \mathbb{E}_W\left[\phi(z_{l-1}^1(x))\phi(z_{l-1}^1(x'))\right],
\end{align*}
where the expectation is taken w.r.t $W$. Similarly, the backward information $I^b_{l,N}(x,x')$ is defined by 
\begin{align*}
I^b_{l, N}(x,x') = \mathbb{E}_W\left[\fracpartial{f_{l:L}}{z^1}(z_l(x))  \fracpartial{f_{l:L}}{z^1}(z_l(x'))\right].
\end{align*}
$I^f_{l, N}$ and $I^b_{l, N}$ are geometric information in accordance with \cref{def:geometric_information}.
\end{definition}

\subsection{Information loss in the large depth limit}
A classical result in the theory of signal propagation is that the information deteriorates with depth $L$ \citep{samuel2017, hayou2019impact} in the sense that the covariance kernels converge to trivial kernels (e.g. constant kernels) in the limit of infinite depth. This is a natural result of the randomness that adds to the neurons with each additional layer. This deterioration occurs with some rate (convergence rate to the trivial kernel w.r.t to $L$) which we call the information loss in the following definition. For two non-negative sequences $(a_n)_{n\geq 0}, (b_n)_{n\geq 0}$, we write $a_n = \Theta_n(b_n)$ if there exists two constant $M_1, M_2 > 0$ such that for all $n$, $M_1 \, b_n \leq a_n \leq M_2 \, b_n$.
\begin{definition}[Information Loss ($\IL$)]\label{def:il}
Let $(g_n(.))_{n \geq 0}$ be a sequence of real-valued functions defined on some set $\mathcal{C} \subset \mathbb{R}^m, m\geq 1$. Assume that $g_n$ converges uniformly to some constant $\kappa$ as $n \to \infty$ and that there exists a non-negative sequence $(r_n)_{n\geq 0}$ such that $\sup_{t \in \mathcal{C}}|g_n(t) - \kappa| = \Theta_n(r_n)$. We say that $g_n$ has an information loss of order $r_n$.
\end{definition}
The $\IL$ characterizes the rate at which the sequence $(g_n(t))_{n\geq 1}$ `forgets' the input $t$ since, by definition, the limiting value $\kappa$ is independent of $t$\footnote{Note that $\IL$ is unique up to a $\Theta$ factor, e.g. an $\IL$ of $n^{-1}$ is the same as an $\IL$ of $n^{-1} \times (2 + n^{-1})$}. 
In our case, we would expect similar behaviour of the forward information $I^f_{l, N}$ in the limit $l \to \infty$\footnote{Note that $I^f_{l, N}$ does not depend on the network depth $L$.}, and the backward information $I^b_{l, N}$ in some $(l,L)-$dependent limit(see \cref{sec:EH_FFNN}). For instance, assume that $L$ is large and consider a small $l$. Then, the forward information has minimal information loss (forward information loss occurs in the limit $l \rightarrow \infty$) while the backward information suffers from deterioration as it depends on the $L-l+1$ last layers, and thus it suffers from the accumulated randomness as it travels back through the network. The opposite happens when $l$ is large, e.g. $l \sim L$. This antagonistic roles of the forward/backward information is key in understanding the behaviour of the tangent kernel $K_l$: there exists a layer index $l_0$ for which the information loss for forward and backward information is comparable. By \cref{eq:intro_decomposition}, we expect $K_l$ to suffer from the deterioration that affects either the forward/backward information. Hence, we hypothesize that a layer with comparable forward/backward information loss is better conditioned to align with data.
\paragraph{The Equilibrium Hypothesis (EH).}\emph{ Let $\IL^f_{l,N}$, resp. $\IL^b_{l,N}$, be the information loss of the sequence $(I^f_{l,N})_{l\geq 1}$, resp. $(I^b_{l,N})_{l\geq 1}$. The layers with the highest alignments with data labels are the ones that satisfy the equilibrium property}$$\IL^f_{l,N} = \Theta\left(\IL^b_{l,N}\right).$$
The EH conjectures that balanced information loss between forward and backward information is related to high alignment with data. Our intuition is as follows: balance between forward and backward information at layer $l$ relates to informative updates of layer $l$'s parameter $\theta_l$, which corresponds to informative updates in the tangent feature $\mathbf{\Psi}_l$ leading to greater alignment with data. To see this for an FFNN, consider function update: 
\begin{align*}
 \delta f_t(X) &= -\eta \hat{\mathbf{K}} \nabla_{f}\mathbb{L}(f(X), Y) \\
 &= -\eta \sum_{l} \hat{\mathbf{K}_l} \nabla_{f}\mathbb{L}(f(X), Y),  
\end{align*}
where the contribution of updating $\theta_l$ to the overall change in $f_t(X)$ is  $- \eta \hat{\mathbf{K}_l} \nabla_{f}\mathbb{L}(f(X), Y)$, where $\eta =\textup{LR}/n$ is the normalized learning rate. Since the kernel matrices satisfy $\bar{\mathbf{K}}_l = \overrightarrow{\mathbf{K}}_l \circ \overleftarrow{\mathbf{K}}_l$, we expect that any deterioration of the kernels $\overrightarrow{K}_l$ and $\overleftarrow{K}_l$ would affect $K_l$. Our intuition is that the parameter update in layer $l$ benefits from the equilibrium property which guarantees that none of the kernels is more deteriorated than the other. This ensures that the function space update benefits from both forward and backward geometric information.  Note that at initialization, with the same analysis of \cref{sec:quantifying_role}, the function update due to $\theta_l$ can be approximated by $\delta f_t(X) \approx \eta \hat{\mathbf{K}_l} \tilde{Y}$. This suggest that the high alignment between tangent features $\mathbf{\Psi}_{l}(X)$ and data labels $Y$ is associated with informative parameter update. A more in-depth discussion of this result is provided in \cref{app:EH_early_training}. In addition, informative parameter update is associated with informative tangent feature update (\cref{app:EH_early_training}).

\begin{figure*}
\centering
\begin{subfigure}{0.32\linewidth}
  \includegraphics[width=\linewidth]{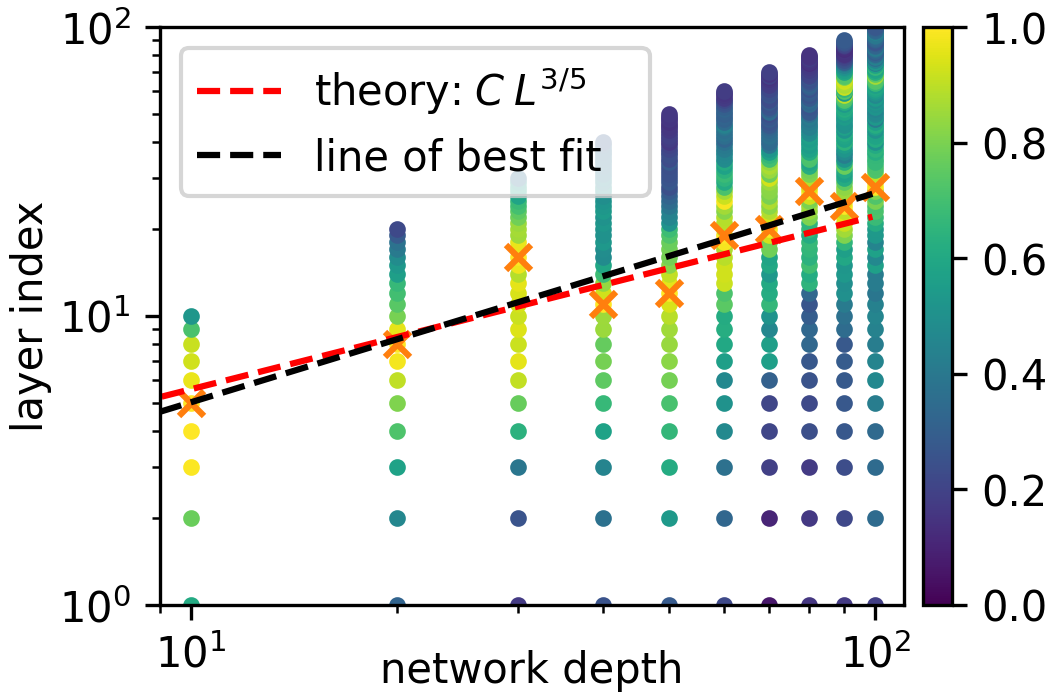}
  \caption{MNIST}\label{fig:MNIST_35}
\end{subfigure}
\begin{subfigure}{0.32\linewidth}
  \includegraphics[width=\linewidth]{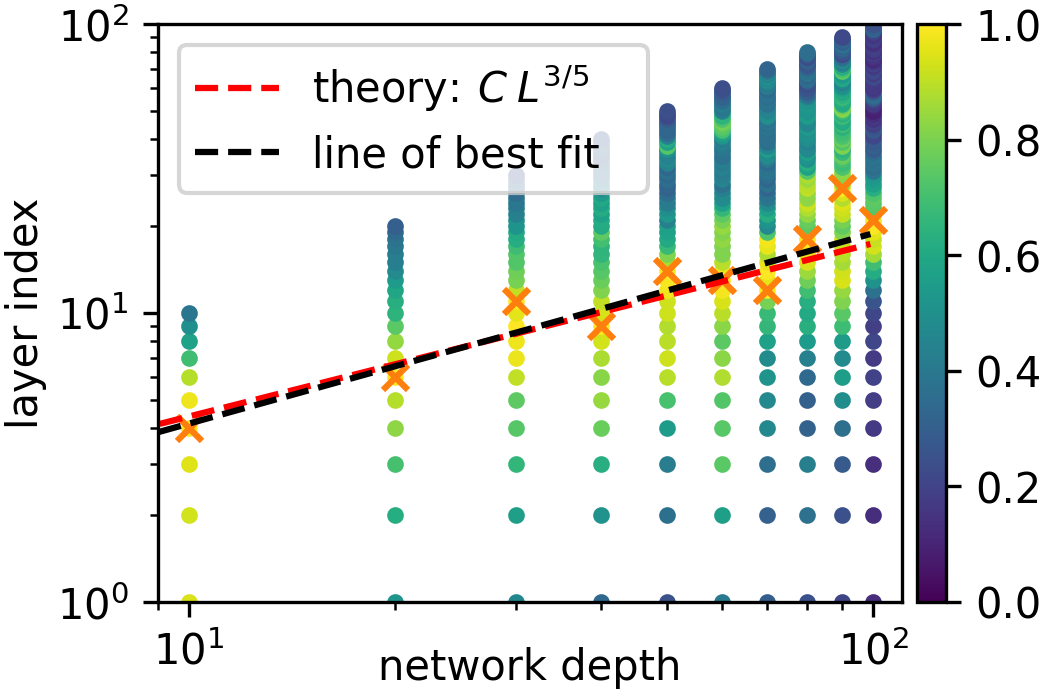}
  \caption{FashionMNIST}\label{fig:FMNIST_35}
\end{subfigure}
\begin{subfigure}{0.32\linewidth}
  \includegraphics[width=\linewidth]{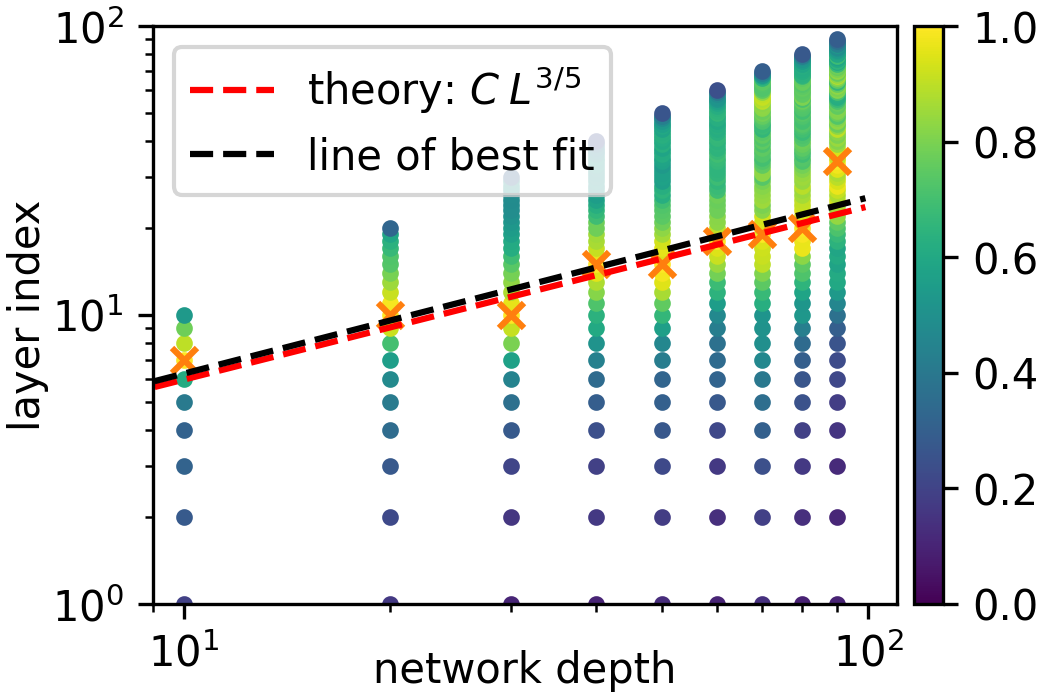}
  \caption{CIFAR10}\label{fig:CIFAR_35}
  \end{subfigure}
\caption{\small{Data with $x=10j$ in the plot corresponds to layer alignments for a FFNN with depth $10j$ trained on the MNIST/FashionMNIST/CIFAR10 datasets. The brighter the color, the closer the corresponding layer's alignment is to the maximum alignment across all layers. \textcolor{orange}{x} indicates the layer where largest alignment occurs. See \cref{fig:C_change_AH} for further experiments on Fashion MNIST showing larger learning rates decrease the $y$-intercept (alignment peaks in earlier layers).}}\label{fig:EH}
\vspace{-0.4cm}
\end{figure*}

\subsection{The Equilibrium in infinite width FFNNs}\label{sec:EH_FFNN}
For FFNNs, we provide a comprehensive analysis of the Equilibrium property in the infinite width limit. We characterize the layers where the equilibrium is achieved and we confirm our theoretical findings with empirical results. For the sake of simplicity, we restrict our theoretical analysis to the sphere $\Sphere=\{ x \in \reals^d, \|x\|=\sqrt{d}\}$ where $\|.\|$ is the euclidean norm. The generalization to $\reals^d$ is straightforward. To avoid issues with col-linearity in the dataset, we consider the set $E_\epsilon$, parameterized by $\epsilon \in (0,1)$, defined by 
\begin{equation}\label{eq:set_E}
E_\epsilon = \{(x,x') \in (\Sphere)^2: \frac{1}{d} x\cdot x' < 1 - \epsilon\}
\end{equation}
% To avoid dealing with different infinite width limits, we focus on the case of rectangular networks, i.e. layer widths are the same.
% \begin{assumption}\label{assump2}
% The widths are given by $N_l = N$ for $ 1 \leq l \leq L-1$ where $N$ is some integer.
% \end{assumption}
The next result characterizes the information loss $\IL$ of the forward/backward information defined above in the limit of infinite width ($N \to \infty$). In this limit, the forward information has an information loss of $l^{-2}$.
\begin{theorem}[Forward $\IL$]\label{thm:forward_loss}
Let $\epsilon \in (0,1)$, and consider the set $E_\epsilon$ as in \cref{eq:set_E}. Define $I^f_{l,\infty}(x,x') := \lim\limits_{N \to \infty} I^f_{l,N}(x,x')$ for all $x,x' \in \reals^d$. We have that 
$$
\sup_{(x,x') \in E_\epsilon}|I^f_{l,\infty}(x,x') - 1/2| = \Theta_l(l^{-2}).
$$
\end{theorem}
\cref{thm:forward_loss} is a corollary of a previous result that appeared in \citet{hayou_ntk}. The proof of the latter relies on an asymptotic analysis of the forward covariance kernel in the limit of large $l$, coupled with a uniform bounding of the convergence rate (See \cref{app:proofs} for more details).\\
The forward information $I^f_{l,N}$ does not depend on depth $L$. On the other hand, the backward information $I^b_{l,N}$ depends both on $l$ and the depth $L$. Therefore, in order to study the asymptotic information loss, we should specify how $l$ grows relatively to $L$. In the next result, we study the two cases where $l \ll L$ or $l =\lfloor \alpha L \rfloor$. 
\begin{theorem}[Backward $\IL$]\label{thm:backward_loss}
Let $\epsilon \in (0,1)$, and consider the set $E_\epsilon$ as in \cref{eq:set_E}. Define $I^b_{l,\infty}(x,x') := \lim\limits_{N \to \infty} I^b_{l,N}(x,x')$ for all $x,x' \in \reals^d$. Then, 
\begin{itemize}
    \item If $l = \lfloor \alpha L \rfloor$ where $\alpha \in (0,1)$ is a constant, then there exists a constant $\mu$ such that in the limit $L \to \infty$,
    $$
    \sup_{(x,x')\in E_\epsilon} |I^b_{l,\infty}(x,x') - \mu| = \Theta_L(\log(L) L^{-1})
    $$
    \item In the limit $l, L \to \infty$ with $l/L \to 0$, 
    $$
    \sup_{(x,x')\in E_\epsilon} |I^b_{l,\infty}(x,x')| = \Theta_{l,L}(\left(L/l\right)^{-3}).
    $$
\end{itemize}
\end{theorem}
A key ingredient in the proof of \cref{thm:backward_loss} in the so-called Gradient Independence assumption. In the literature on signal propagation at initialization (e.g. \citep{samuel2017,poole, hayou2019impact, yang2017meanfield, yang2019scaling, xiao_cnnmeanfield}), results on gradient backpropagation rely on the assumption that the weights used for backpropagation are independent from the ones used for forward propagation. \citet{yang_tensor3_2020} showed that this assumption yields exact computations of gradient covariance and NTK in the infinite width limit. 
% Intuitively, since the contribution of each weight to the next layer is of order $1/\sqrt{N}$, then as $N$ grows, we would naturally expect the dependence to weaken. 
We refer the reader to \cref{app:backprop} for more details.\\
In the case of infinite width FFNN, using Theorems \ref{thm:forward_loss} and \ref{thm:backward_loss}, we show which layers satisfy the equilibrium property.
% \begin{figure}
% \centering

% % \begin{subfigure}{\linewidth}
% %   \includegraphics[width=\linewidth]{fig/FMNIST_FFNN.png}
% %   \caption{FFNN on MNIST}
% % \end{subfigure}

% \begin{subfigure}{\linewidth}
%   \includegraphics[width=0.9\linewidth]{fig/CIFAR10_FFNN.png}
%   \caption{FFNN on CIFAR10}
%   \end{subfigure}

% \caption{A bar on $10j$ in the plot corresponds to layer alignments for a FFNN with depth $10j$ trained on the CIFAR10 dataset. The brighter the color, the closer the corresponding layer's alignment is to the maximum alignment across all layers. \textcolor{orange}{x} indicates the layer where largest alignment occurs. See \cref{app:further_experiments} for the same experiment on other datasets.}\label{fig:EH}
% \end{figure}
\begin{corollary}[Equilibrium]\label{cor:equilibrium_ffnn}
Under the conditions of Theorems \ref{thm:forward_loss} and \ref{thm:backward_loss}, the equilibrium for an FFNN is achieved for layers with index
$$
l = \Theta_L(L^{3/5})
$$
where $L$ is the network depth.
\end{corollary}
\cref{cor:equilibrium_ffnn} indicates that layers that satisfy the equilibrium property verify $l = \Theta_L(L^{3/5})$. In logarithmic scale, this implies  $\log(l) \in \left[\frac{3}{5} \log(L) + C_1,  \frac{3}{5} \log(L) + C_2\right] $ where $C_1, C_2$ are constants that depends on $\epsilon$ from Theorems \ref{thm:forward_loss} and \ref{thm:backward_loss}. \cref{fig:EH} shows the layer alignments $A_l$'s after training an FFNN (width 256) on different datasets. We fit the line $\log(l)=3/5 \log(L)+C$ by finding the constant $C$ that minimizes the squared error. We also perform a simple linear regression to see if the slope is close to $3/5$ (line of best fit). All experiments show an excellent match with the theoretical line $3/5 \log{L}+C$ (which was derived for infinite width networks).
\vspace{-0.3cm}
\section{Alignment and Hyperparameters}\label{sec:generalization}
Generalization and feature learning have been linked to optimizer hyperparameter choices (e.g.\ \citep{keskar2016large}). While this paper largely focuses on the $L^{3/5}$ scaling law in the EH ($l = \Theta(L^{3/5})$) it would be incomplete without a brief discussion of the constant term in this equation (the y-intercept in \cref{fig:EH}).
In this section, we observe that the layer index with the greatest alignment can be significantly affected by choice of optimizer hyperparameters, and good generalization is associated with a well defined peak away from the last layer. See \cref{fig:C_change_AH} for clear demonstration of the change in $C$ with learning rate.
% 1. Batch size learning rate connection to feature learning/constant in front of L3/5
% 2. Future work connection to generalisation via FIM
% 3. Short summary to suggest future work

To understand (1) the effect of optimizer hyperparameters on the alignment and (2) the impact of the Alignment Hierarchy on the generalization error, we trained multiple models with different datasets with a selection of batch sizes, learning rates and optimizers. We show the corresponding AH pattern, generalization error and generalization gap for CIFAR10 on Resnet18 and VGG19 in \cref{fig:egs_main}. 
The colour corresponds to the final test loss, and the number to the loss gap. 
\cref{fig:egs_main} weakly suggests that large alignments with data labels, especially for the middle layers, correlate with good generalization properties.
Note that the peak for the FFNN occurs in the early layers; VGG the middle layers, and for Resnet18 very near the last layer. If the EH is also valid for resnets, the peak should move towards the middle layers as depth increases. This is left for future work.
For further experiments, including on random labels, see \cref{app:further_experiments}.

\begin{figure*}[h]
\centering
\begin{subfigure}{.33\textwidth}
  \centering
  \includegraphics[width=\textwidth]{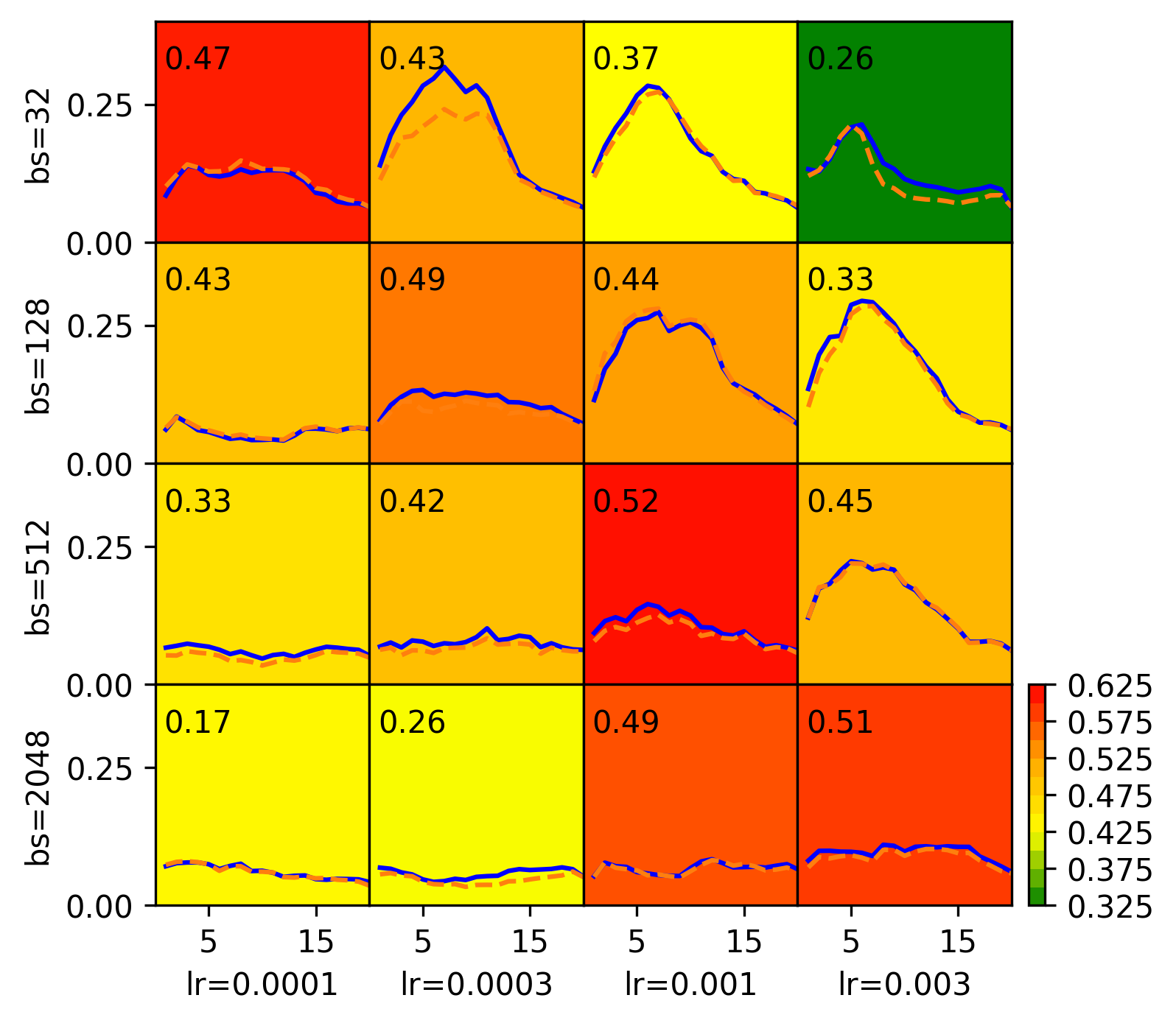}
  \caption{FFNN on Fashion MNIST}
  \label{fig:ffnn_eg}
\end{subfigure}
\begin{subfigure}{.33\textwidth}
  \centering
  \includegraphics[width=\textwidth]{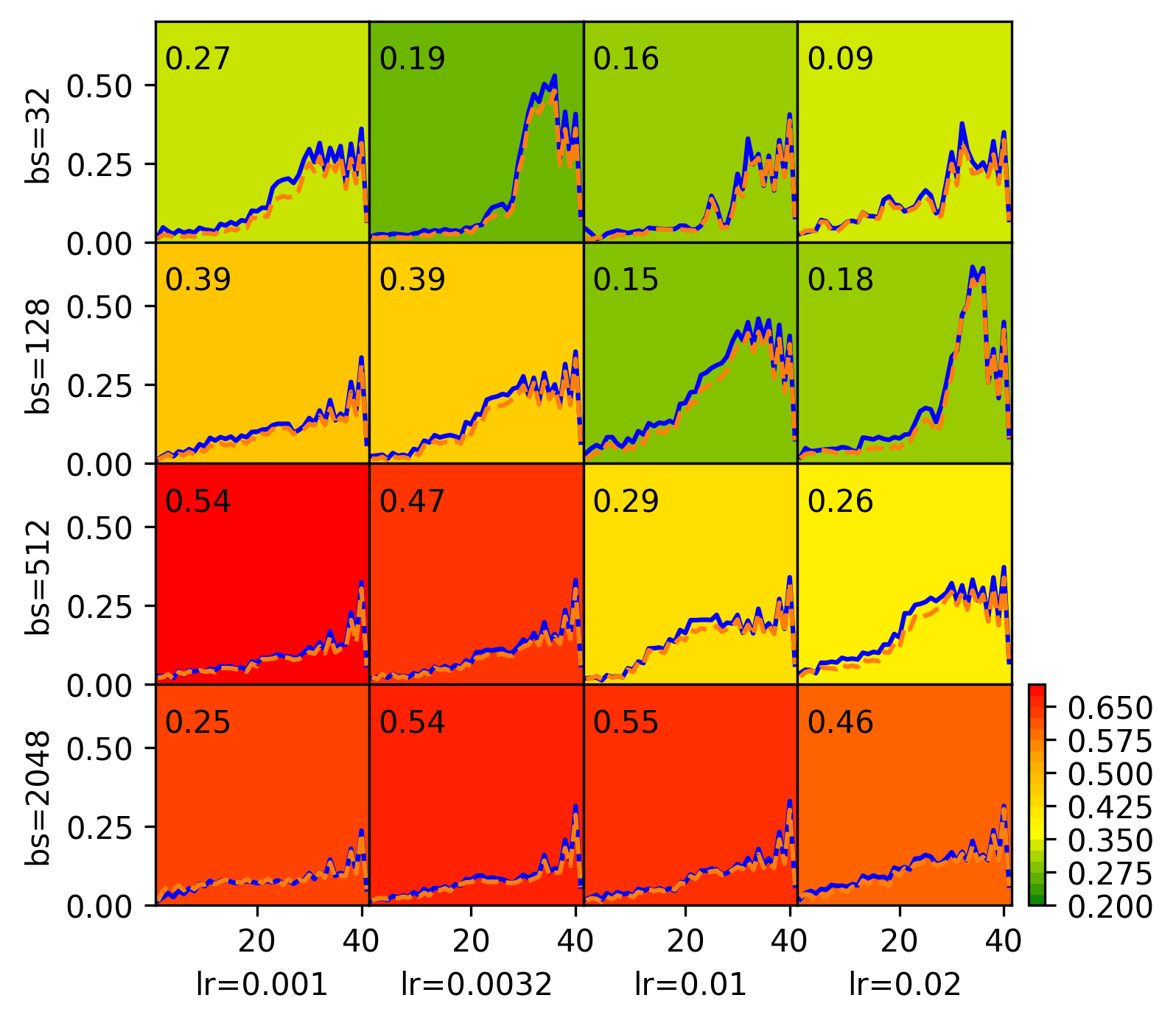}
  \caption{Resnet18 on CIFAR10}
  \label{fig:resnet_eg}
\end{subfigure}
\begin{subfigure}{.33\textwidth}
  \centering
  \includegraphics[width=\textwidth]{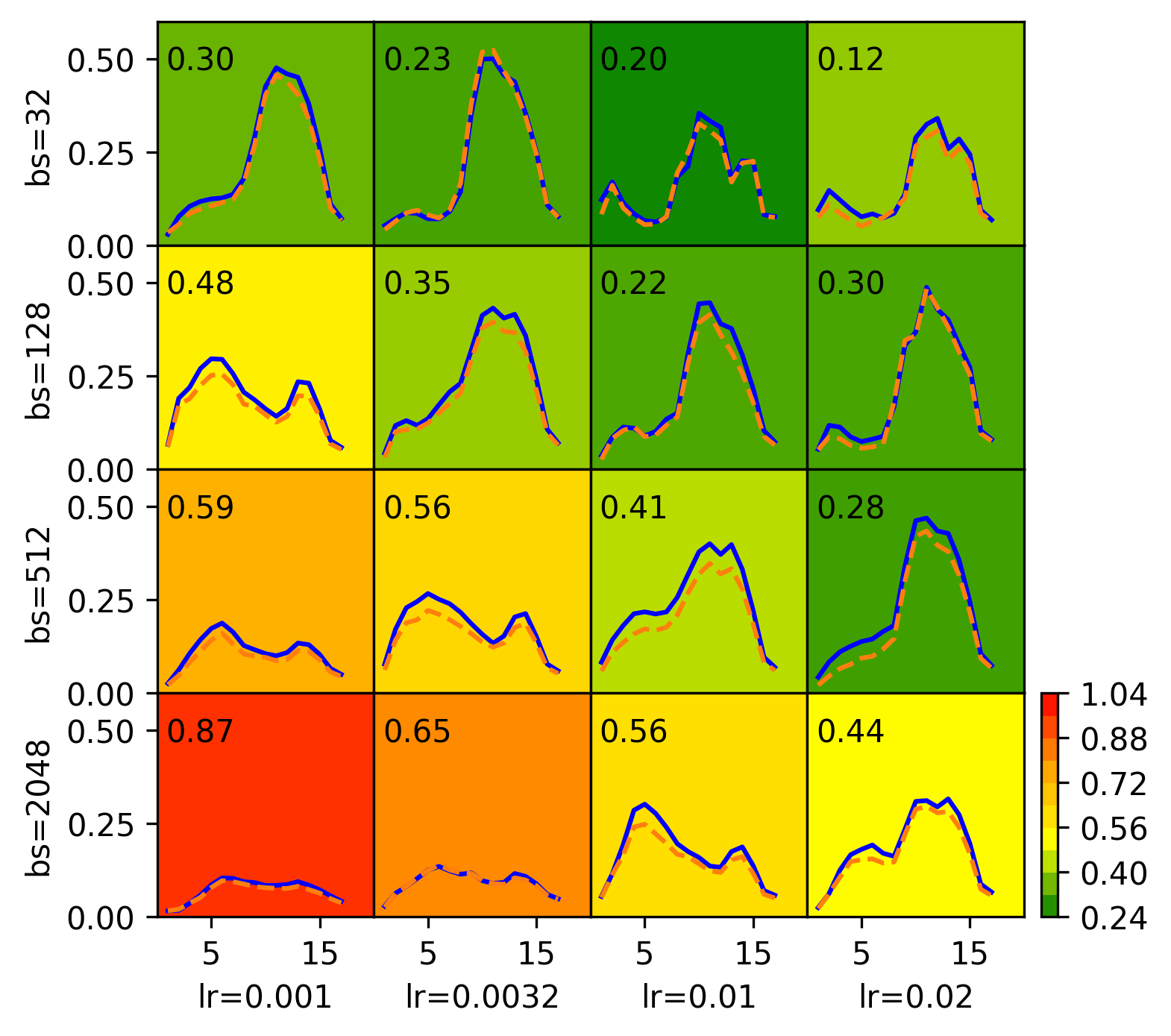}
  \caption{VGG19 on CIFAR10}
  \label{fig:vgg_eg}
\end{subfigure}
\caption{
\small{Alignment hierarchy as a function of batch size and learning rate for (a) Fashion MNIST on a FFNN (b) CIFAR10 on a Resnet18 and (c) CIFAR10 on a VGG19. The location of the peak is different for each dataset, architecture and hyperparameter combination -- note that AH cannot be seen here, as it predicts a scaling law with layers, but not the constants (see \cref{cor:equilibrium_ffnn}). The background colour denotes the test loss, and the number in the top left corner is the generalisation loss gap. Each model was trained stopping after a train loss of 0.1 or 500 epochs, whichever was sooner. There is a clear positive correlation between the height of the peak and the overall generalisation error for each example given here.}
}
\vspace{-0.5cm}
\label{fig:egs_main}
\end{figure*}
\vspace{-0.3cm}
\paragraph{Batch size and learning rate dependence}
\cref{fig:egs_main} shows that the AH is strongly affected by both batch size and learning rate. Typically, smaller batch sizes and larger learning rates lead to better generalisation where convergence is possible. 
They also lead to a larger peak in the AH, and the peak shifts towards the center. Although it is expected that learning rate and batch size affect training, we do not currently have a theoretical explanation for these effects on the AH.
\vspace{-0.3cm}
\paragraph{Connection to Stable Rank, Alignment \& Fisher Information}

\cite{baratin2021implicit,oymak2019generalization} observed a correlation between improved generalisation and (1) an increasing majority of the singular values of $\mathbf{\Psi}$ are very small with a few very large (2) the label vector $Y$ is increasingly aligned with the large singular directions in $\mathbf{\Psi}$. These are the conditions for maximsing CKA. In \cref{app:fim}, we extend these results to Layerwise CKA.
The layerwise CKA $A_l$ can be decomposed into a product of the inverse square root of its stable rank and the correlation term between the eigenvectors and the label vector $Y$. The stable rank measures an effective dimension of the internal representations of the neural network.
CIFAR10 with degrees of data randomisation was trained on VGG19, and showed a lower overall layerwise stable rank for the better generalising models. The alignment term varied the most across generalisation errors, with much greater alignment (in the later layers) for the best generalising model than the others. 
Larger overall alignment therefore correlates with lower stable rank; and an earlier peak appears to coincide with earlier alignment between $Y$ and large singular directions in $\mathbf{\Psi}$.
We might expect lower dimensional internal representations and earlier, larger alignment to correlate with good generalisation.
We also investigated Fisher Information Matrix $\mathbf{J}=\mathbf{\Psi}\mathbf{\Psi}^T$ (for MSE loss), and $\mathbf{K}$. 
We use observations from \citep{maddox2020effectiveparam} to link properties of $\mathbf{J}$ that coincide with larger CKA to explain why more alignment is likely to correlate with better generalisation.

% The Fisher information matrix $\mathbf{J}(\theta)$ measures how sensitive a likelihood function (which we take to be $p(y\mid x;\theta) \propto e^{-\|(y-f(x;\theta))\|^2}$)
% is to perturbations in $\theta$ (see \cref{app:fim} for more details).
% \cite{karakida2019pathological} showed that the empirical Fisher information matrix for neural networks is $\mathbf{J}=\mathbf{\Psi}\mathbf{\Psi}^T$\footnote{Assuming the above likelihood function, which corresponds to MSE loss (used for NTK). See \cref{sec:fim_intuition} for a discussion of other loss functions.} -- meaning it shares the same non-zero eigenspectrum as $\mathbf{K}=\mathbf{\Psi}^T\mathbf{\Psi}$.
% \citep{keskar2016large, wu2017towards, chaudhari2019entropy} have argued that `flatter' functions generalize better, and one can measure `flatness' with the eigenvalues of the FIM \citep{maddox2020effectiveparam}. 
% This link can also be used to argue that layers with higher kernel alignment are effectively lower dimensional, and thus higher layerwise kernel alignment should correlate with better generalisation.
% See \cref{app:fim} for a more complete discussion and experimental evidence.

\section{Related work}
To the best of our knowledge, the AH has only been discussed in \citep{baratin2021implicit}. However, the significance of AH goes beyond a simple feature learning pattern. The AH can be seen a \emph{structural} implicit regularization effect, i.e. a regularization effect that is purely induced by the depth.
Traditionally, implicit regularization refers to any hidden regularization effect induced by the training algorithm. For example, it is widely believed that the implicit regularization effect of Stochastic Gradient Descent (SGD) \citep{lecun1998sgd, he_relu, Krizhevsky_sgd} is mainly driven by the small-batch sampling noise \citep{jastrzebski2018finding}. 
However, recent empirical findings such as \citep{wu2017towards, geiping2021stochastic} demonstrated that DNNs can still achieve high accuracy on some image datasets with full-batch GD. \citet{goyal2017accurate} demonstrated that increasing batch size by several orders of magnitude on ImageNet does not affect generalization error significantly, suggesting that classical implicit regularization theories that rely on SGD noise to explain generalization are not sufficient to explain why neural networks generalize \textit{at all} (further backed up by the near parity achieved by kernel methods \citep{mingard2021sgd}).
These results suggest that implicit regularization might occur via other mechanisms than previously thought, one of them could this purely structural effect that results in the Alignment Hierarchy. 
Given interest in understanding how the amount of feature learning affects generalization and transfer learning, we believe this is a promising topic for future work.

\section{Conclusion and Limitations}
In this paper, we introduced the Equilibrium Hypothesis (EH) which connects information flow at initialization to tangent features alignment with data labels. the EH explains the alignment hierarchy, illustrated in \cref{fig:align_example}. Our empirical results showed an excellent match with the theoretical prediction $l = \Theta(L^{3/5})$ for FFNN on different datasets in \cref{fig:EH}, and presented empirical evidence that earlier alignment correlates with better generalisation \cref{fig:egs_main}.
Finally, we used connections between layerwise CKA, the stable rank and Fisher Information to present a theoretical case for this observation.
There are still multiple open questions to answer, e.g. the impact of the architecture on the alignment hierarchy. As demonstrated in \cref{fig:egs_main}, it seems that the alignment pattern is sensitive to the choice of the architecture. We leave this topic for future work.

\newpage
%\bibliography{sample}
\printbibliography
%\newpage
\appendix
\onecolumn
%\section*{Appendix}
\setcounter{equation}{0}
\setcounter{lemma}{0}
\setcounter{theorem}{0}
\setcounter{section}{0}
\setcounter{assumption}{0}

\section{Review of Signal propagation theory}\label{app:infinite_width_limit}
The signal propagation theory in the context of neural networks deals precisely with the distortion of the information carried by the output as it travels through the network. Most results in this theory (see e.g. \citet{poole, samuel2017, yang_tensor3_2020, hayou_ntk, hayou2019impact, jacot2018NTK, jacot2020asymptotic}) consider the infinite width limit as it allows the derivation of closed-form expressions. Infinite width networks are also naturally overparameterized (infinite number of parameters) and therefore might offer some theoretical insights on the overparameterized regime.

\paragraph{Fully-connected FeedForward Neural Networks (FFNN).} 
Given an input $x \in \reals^d$, and a set of weights and bias $(W_l, b_l)_{1 \leq L}$, the forward propagation is given by 
\begin{align}\label{eq:ffnn}
    z_1(x) &= W_1 x + b_1\\
    z_l(x) &= W_l \phi(z_{l-1}(x)), \quad 2 \leq l \leq L,
\end{align}
where $W_{l} \in \mathbb{R}^{N_{l} \times N_{l-1}}$, and $\phi$ is the ReLU activation function given by $\phi(v) = \max(v,0)$ for $v \in \reals$.
The dimension of the parameter space is $P=\sum_{l=0}^{L-1}\left(N_l+1\right) N_{l+1}$ where we denote $N_0:=d$. For each layer, the weights are initialized with i.i.d Gaussian variables $W_l^{ij} \sim \mathcal{N}(0,\frac{2}{N_{l-1}})$.

\subsection{Forward propagation}\label{subsection:aforward_prop}

When we take the limit $N_{l-1} \rightarrow \infty$ recursively over $l$, this implies, using Central Limit Theorem, that $z^i_{l} (x)$ is a Gaussian random variable for any input $x$. The convergence rate to this limiting Gaussian distribution is given $\mathcal{O}(1/\sqrt{N_{l-1}})$ (standard Monte Carlo error). More generally, an approximation of the random process $z^i_l(.)$  by a Gaussian process was first proposed by \cite{neal} in the single layer case and has been extended to the multiple layer case by \cite{lee_gaussian_process} and \cite{matthews}. The limiting Gaussian process kernels follow a recursive formula given by, for any inputs $x,x'\in \mathbb R^d$
\begin{align*}
\kappa_l(x,x') &= \mathbb{E}[z_l^i(x)z_l^i(x')]\\
&= 2 \, \mathbb{E}[\phi(z_{l-1}^i(x))\phi(z_{l-1}^i(x'))]\\
&= 2 \, \Psi_{\phi} (\kappa_{l-1}(x,x), \kappa_{l-1}(x,x'), \kappa_{l-1}(x',x')),
\end{align*}
where $\Psi_{\phi}$ is a function that only depends on $\phi$. This provides a simple recursive formula for the computation of the kernel $\kappa^l$; see, e.g., \cite{lee_gaussian_process} for more details.

\subsection{Gradient Independence}\label{section:gradient_independence}
In the literature of infinite width DNNs, a standard assumption in prior literature is that of the gradient independence which is similar in nature to the concept of feedback alignment \citep{timothy_feedback}. This assumption states that, for infinitely wide neural networks, if we assume the weights used for forward propagation are independent from those used for back-propagation. When used for the computation of Neural Tangent Kernel, this approximation was proven to give the exact computation for standard architectures such as FFNN, CNN and ResNets \cite{yang_tensor3_2020}.

\begin{lemma}[Corollary of Theorem D.1. in \citep{yang_tensor3_2020}]\label{lemma:gradient_independence}
Consider an FFNN with weights $\mathbf{W}$. In the limit of infinite width, we can assume that $\mathbf{W}^T$ used in back-propagation is independent from $\mathbf{W}$ used for forward propagation, for the calculation of Gradient Covariance and NTK.
\end{lemma}

This result has been extensively used in the literature as an approximation before being proved to yield exact computation for gradient covariance and NTK.

\paragraph{Gradient Covariance back-propagation.} Analytical formulas for gradient covariance back-propagation were derived using this result, in \citep{hayou2019impact, samuel2017, poole, xiao_cnnmeanfield, yang2019scaling}. Empirical results showed an excellent match for FFNN in \cite{samuel2017}, for Resnets in \cite{yang2019scaling} and for CNN in \cite{xiao_cnnmeanfield}. 

\paragraph{Neural Tangent Kernel.}The Gradient Independence approximation was implicitly used in \cite{jacot2018NTK} to derive the infinite width Neural Tangent Kernel (See \cite{jacot2018NTK}, Appendix A.1). The authors have found that the infinite width NTK computed with the Gradient Independence approximation yields excellent match with empirical (exact) NTK.

\subsection{Back-propagation}\label{app:backprop}

For FFNN layers, let $q_l(x):=q_l(x,x)$ be the variance of $z_l^1(x)$ (the choice of the index $1$ is not important since, in the infinite width limit, the random variables $(z_l^i(x))_{i \in [1:N_l]}$ are i.i.d). Let $q_l(x,x')$, resp. $c_l^1(x,x')$ be the covariance, resp. the correlation between $z_l^1(x)$ and $z_l^1(x')$. For Gradient back-propagation, let $\Tilde{q}_l(x,x')$ be the Gradient covariance defined by  $\Tilde{q}_l(x,x')= \mathbb{E}\left[ \frac{\partial \mathcal{L}}{\partial z_l^1}(x) \frac{\partial \mathcal{L}}{\partial z_l^1}(x')\right]$ where $\mathcal{L}$ is some loss function. Similarly, let $\Tilde{q}_l(x)$ be the Gradient variance at point $x$. We also define $\Dot{q}_l(x,x') = 2 \mathbb{E}[\phi'(z_{l-1}^1(x))\phi'(y_{l-1}^1(x'))] $.

Given two inputs $x,x' \in \mathbb{R}^d$, using Central Limit Theorem as in \cite{samuel2017}, we obtain
$$
q_l(x,x') = 2 \mathbb{E}\left[\phi\left(\sqrt{q_l(x)} Z_1\right)\phi\left(\sqrt{q_l(x')}(c^{l-1} Z_1 + \sqrt{1 - (c^{l-1})^2} Z_2\right)\right], \quad Z_1, Z_2 \overset{iid}{\sim} \mathcal{N}(0,1),
$$
with $c^{l-1} := c^{l-1}(x,x')$.\\
With ReLU, and since ReLU is positively homogeneous (i.e. $\phi(\lambda x) = \lambda \phi(x)$ for $\lambda\geq0$), we have that 
$$
q_l(x,x') = \sqrt{q^l(x)}\sqrt{q^l(x')} g(c^{l-1})
$$
where $g$ is the ReLU correlation function given by \cite{hayou_ntk}
\begin{align}\label{eq:correl_function}
    g(c) = \frac{1}{\pi}(c \arcsin{c} + \sqrt{1- c^2}) + \frac{1}{2}c.
\end{align}

\paragraph{Gradient back-propagation.}
The gradient back-propagation is given by 
\begin{align*}
    \frac{\partial f_{l:L}}{\partial y^l_i} &= \phi'(y^l_i) \sum_{j=1}^{N_{l+1}} \frac{\partial f_{l:L}}{\partial y^{l+1}_j} W^{l+1}_{ji}.
\end{align*}
where $f_{l:L}$ is the mapping from layer $l$ to the output. 
Using the Gradient Independence in the infinite width limit (Lemma \ref{lemma:gradient_independence}) and assuming all layer widths go to infinity at the same rate, a Central Limit Theorem argument yields (see e.g. Section 7.9 in the appendix in \cite{samuel2017})
\begin{equation*}
    \Tilde{q}_l(x, x') = \Tilde{q}_{l+1}(x,x') g'(c_l(x,x')),
\end{equation*}
where $g$ is the function defined in \cref{eq:correl_function}.\\

By telescopic product, we obtain

\begin{equation}
    \Tilde{q}_l(x, x') = \Tilde{q}_{L}(x,x')  \prod_{k=l}^{L-1} g'(c_k(x,x')) = \Tilde{q}_{L}(x,x') \frac{\zeta_L(x,x')}{\zeta_l(x,x')}.
\end{equation}
where $\zeta_m(x,x') = \prod_{k=1}^{m-1} g'(c_k(x,x'))$ for $m\geq 2$.
\subsection{Standard parameterization Vs NTK parameterization}\label{app:ntk_param}
Many papers that study the NTK consider the so-called NTK parameterization given by 

\begin{align}\label{eq:ffnn_ntk_param}
    z_1(x) &= \frac{2}{\sqrt{d}} W_1 x + b_1\\
    z_l(x) &= \frac{2}{\sqrt{N_{l-1}}} W_l \phi(z_{l-1}(x)), \quad 2 \leq l \leq L,
\end{align}
where the weights $W_l^{ij}$ are initialized with standard normal distribution $\mathcal{N}(0,1)$.
However, both parameterizations yield the same quantities for signal propagation at initialization, i.e. the covariance $q_l$ is the same for both parameterizations. In our proofs, we will refer to results in \cite{hayou_ntk} and \cite{samuel2017} that consider either the NTK or the standard paramaterization.

\section{Proofs}\label{app:proofs}

\subsection{Proof of \cref{thm:forward_loss}}
\begin{manualthm}{\ref{thm:forward_loss}}[Forward Information Loss]
Let $\epsilon \in (0,1)$, and consider the set $E_\epsilon$ as in \cref{eq:set_E}. Define $I^f_{l,\infty}(x,x') := \lim\limits_{N \to \infty} I^f_{l,N}(x,x')$ for all $x,x' \in \reals^d$. We have that 
$$
\sup_{(x,x') \in E_\epsilon}|I^f_{l,\infty}(x,x') - 1/2| = \Theta(l^{-2}).
$$
\end{manualthm}
\cref{thm:forward_loss} is a corollary of a previous result that appeared in \cite{hayou_ntk}. The proof techniques for the latter rely on an asymptotic analysis of a well defined covariance kernel in the limit of large $l$, coupled with a uniform bounding of the convergence rate.\\

\begin{proof}
Fix $(x,x') \in E$. From \cref{app:infinite_width_limit}, we know that $I^f_{l,\infty}(x,x') = \frac{1}{2} q_l(x,x')$  where $q_l$
is the covariance between $z_l^1(x), z_l^1(x')$ given by
$$
q_l(x,x') = \E[z_l^1(x) z_l^1(x')]
$$

Since $q_1(x,x)=q_1(x',x') = 1$, $q_1(x,x')$ can be seen as the correlation between $z_1^1(x)$ and $z_1^1(x')$. Recursively, it is straightforward that $q_l(x,x)=q_l(x',x') = 1$ for all $l$, suggesting that $q_l(x,x')$ can be seen as the correlation between $z_l^1(x)$ and $z_l^1(x')$.\\

From `Appendix Lemma 1' in \cite{hayou_ntk}, we have that 
$$
\sup_{(x,x') \in E_\epsilon}|q_l(x,x') - 1| = \Theta(l^{-2})
$$
which yields the desired result.
\end{proof}

\subsection{Proof of \cref{thm:backward_loss}}
We first prove a result that will be useful in the proof of \cref{thm:backward_loss}.
\begin{lemma}[Uniform Asymptotic Expansion]\label{lemma:uniform_expansion}
Let $a \geq 1$ be a positive integer. We define the sequence $(b_l)_{l \geq 0}$ by 
$$
b_l = \beta_l b_{l-1}, 
$$
where $(\beta_l)_{l\geq 0}$ is a sequence of reals numbers that satisfy $\beta_l = 1 - \frac{a}{l} + \kappa \frac{\log(l)}{l^2} + \bigO(l^{-2})$ where $\kappa \neq 0$ is a constant that does not depend on $\beta_0$. Assume that the $\bigO$ bound is uniform over $\beta_0$. Then, uniformly over $\beta_0$, we have that
$$
\log(b_l) = -a \log(l) +  \kappa \frac{\log(l)}{l} + \bigO(l^{-1})
$$
\end{lemma}

\begin{proof}
Let $r_l := b_l l^a$. We have that 
\begin{align*}
    r_l &= b_l r_{l-1} (1 + a l^{-1} + \bigO(l^{-2}))\\
    &= (1 + \kappa\frac{\log(l)}{l^2} + \bigO(l^{-2})) r_{l-1}
\end{align*}
which yields
$$
\log(r_l / r_{l-1}) = \kappa \frac{\log(l)}{l^2} + \bigO(l^{-2})
$$

Since the series on the right side converge, we have that
\begin{equation*}
\begin{aligned}
\sum_{k\geq l }\log(r_k / r_{k-1}) &= \sum_{k\geq l }  \kappa \frac{\log(k)}{k^2} + \bigO(\sum_{k\geq l }  k^{-2})\\
&= -\kappa \frac{\log(l)+1}{l} + \bigO(\log(l) l^{-2}) + \bigO(l^{-1})\\
&= -\kappa \frac{\log(l)}{l} + \bigO(l^{-1})
\end{aligned}
\end{equation*}
\end{proof}
where we have use the integral estimates of the remainders of series. Since the $\bigO$ bound in $\beta_l$ is uniform over $\beta_0$ by assumption, then the resulting $\bigO$ bound in $\log(r_l)$ is also uniform over $\beta_0$, which concludes the proof.
\begin{manualthm}{\ref{thm:backward_loss}}[Backward Information Loss]
Let $\epsilon \in (0,1)$, and consider the set $E_\epsilon$ as in \cref{eq:set_E}. Define $I^b_{l,\infty}(x,x') := \lim\limits_{N \to \infty} I^b_{l,N}(x,x')$ for all $x,x' \in \reals^d$. Then, 
\begin{itemize}
    \item If $l = \lfloor \alpha L \rfloor$ where $\alpha \in (0,1)$ is a constant, then there exists a constant $\mu$ such that in the limit $L \to \infty$,
    $$
    \sup_{(x,x')\in E_\epsilon} |I^b_{l,\infty}(x,x') - \mu| = \Theta(\log(L) L^{-1})
    $$
    \item In the limit $l, L \to \infty$ with $l/L \to 0$, 
    $$
    \sup_{(x,x')\in E_\epsilon} |I^b_{l,\infty}(x,x')| = \Theta(\left(L/l\right)^{-3}).
    $$
\end{itemize}

\end{manualthm}

\begin{proof}
Let $\epsilon \in (0,1)$. Note that $I^l_{b,\infty}(x,x') = \tilde{q}_l(x,x')$ where $\tilde{q}_l$ is defined in \cref{app:backprop}.

Using a Taylor expansion of $g$ near 1, Appendix Lemma 1 in \cite{hayou_ntk} shows that there exists a constant $\kappa >0$ such that 
$$
\sup_{(x,x') \in E} |g'(c_l(x,x')) - 1 + \frac{3}{l} - \kappa \frac{\log(l)}{l^{2}}| = \bigO(l^{-2})
$$

Let $\zeta_l(x,x')= \prod_{k=1}^{l-1} g'(c_k(x,x'))$ as in \cref{app:backprop}. It is clear that $(\zeta_l)$ satisfies the conditions in \cref{lemma:uniform_expansion}. Hence, letting $r_l = \zeta_l(x,x') \, l^3$ \footnote{We omit $(x,x')$ to alleviate the notations.}, we obtain
$$
\log(r_l) = \kappa \frac{\log(l)}{l} + \bigO(l^{-1})
$$
where the $\bigO$ bound is uniform over $(x,x') \in E$. 

The loss function is given by $\mathbb{L}(f(x), y)$, therefore $\tilde{q}_L(x,x') = \E[\fracpartial{\mathbb{L}(f(x), y)}{z_L^1(x)} \fracpartial{\mathbb{L}(f(x'), y')}{z_L^1(x')}]$. Using the result on the correlation propagation from Appendix Lemma 1 in \cite{hayou_ntk}, we obtain that $\tilde{q}_L(x,x') = \tilde{q} + \bigO(L^{-2})$ as $L$ goes to infinity, where $\tilde{q}$ is independent of $(x,x')$.

Now let us discuss the two cases:

\begin{itemize}
    \item Case 1 ($l=\lfloor \alpha L \rfloor$): in this case, we have that $\log(r_L) - \log(r_{\lfloor \alpha L \rfloor}) = \Theta(\log(L) L^{-1})$, which yields
    $$
    \frac{\zeta_L(x,x')}{\zeta_{\lfloor \alpha L \rfloor}(x,x')} = \alpha^3  + \Theta(\log(L) L^{-1})
    $$
    where $\Theta$ is uniform over $x,x'$. This proves that 
    $$
    \sup_{(x,x') \in E} \left| \tilde{q}_{\lfloor \alpha L \rfloor}(x,x') - \alpha^3 \tilde{q} \right| = \Theta(\log(L) L^{-1})
    $$
    
    \item Case 2 ($l/L \to 0$): in this case, we have that 
    $$
    \frac{\zeta_L(x,x')}{\zeta_{l}(x,x')} \sim (L/l)^{-3},
    $$
    uniformly over $x,x'$. We conclude that 
    
    $$
    \sup_{(x,x') \in E}| \tilde{q}_l(x,x')| = \Theta((L/l)^{-3}).
    $$
    
\end{itemize}

\end{proof}

\subsection{Proof of \cref{cor:equilibrium_ffnn}}

\begin{manualcorollary}{\ref{cor:equilibrium_ffnn}}[Equilibrium]
Under the conditions of Theorems \ref{thm:forward_loss} and \ref{thm:backward_loss}, the equilibrium for an FFNN is ahieved for layers with index
$$
l = \Theta(L^{3/5})
$$
where $L$ is the network depth.
\end{manualcorollary}
\begin{proof}
We have two cases:
\begin{itemize}
    \item If $l$ is of the same order as $L$, or simply $l = \lfloor\alpha L\rfloor$ where $\alpha \in (0,1)$ is a constant, then to have the equilibrium property, we need to have $l^{-2} = \Theta(\log(L) L^{-1})$ which is absurd. Hence, equilibrium cannot be achieved in this case.
    \item Therefore, the only possible scenario where equilibrium can be achieved is when $l/L \to 0$, in this case, the equilibrium property implies $l^{-2} = \Theta( (L/l)^{-3})$ which yields
    $$
    l  = \Theta(L^{3/5})
    $$
\end{itemize}
\end{proof}

\subsection{Connection between the equilibrium property and effective parameter updates}\label{app:EH_early_training}
% \section{A tentative justification of EH at early training stages}\label{sec:EH_early_training}
The EH conjectures architectural advantage for some FFNN layers to more effectively align with data. Using an approximate analysis of the second order geometry of DNNs, we depict the mechanisms by which the equilibrium property impacts how the alignment evolves during early training. We emphasize that the analysis in this section is not intended to be mathematically rigorous but rather insightful for future work on the EH. 
% The bias in principle eigenvectors of $\mathbf{H}_w$ explain hierarchical structure of CKA at the end of training, and is key to understanding the true relationship between EH and effective feature learning.
We restrict our analysis to the setting of FFNN but we now consider a classification task with $k$ classes (i.e. $o=k$). The loss function $\mathbb{L}$ is the cross-entropy loss. We denote by
% To gain insights on how the equilibrium property affects feature learning, we analyse the evolution of layer CKA during early stages of training of ReLU-activated Neural Network architectures on classification tasks. We highlight a previously overlooked component of the loss hessian matrix that intimately relates to hierarchical structure in CKA values at convergence. We further relate equilibrium property at initialization to effective feature learning at early training stages.
$F = (f_\theta(x_1)^T, f_\theta(x_2)^T, \dots, f_\theta(x_n)^T) \in \mathbb{R}^{nk}$ the concatenation of all outputs $f_{\theta}(x_i)$ evaluated on the training dataset, $w = \nabla_{F} \mathcal{L} \in \mathbb{R}^{nk}$ the gradient of the loss w.r.t to $F$, 
    % \fTBD{why transpose? since $F \in \mathbb{R}^{nC}$, its gradient should be in $\mathbb{R}^{nC}$ by definition. \textcolor{red}{the previous notation is taking derivative so the result will be of shape of $F^T$. I now change it to gradient $\nabla$ which is a lot better.}}.
$Y \in \mathbb{R}^{nk}$ the concatenation of all one-hot vectors given by labels $y_i$ in the dataset, and $\tilde{Y} = \mathbf{C}Y$ the centered version of $Y$.

% \paragraph{Gradient updates.} The gradient update at layer $l$ is given by 
% $$
% \delta \theta_l = -\eta \nabla_{\theta_l}f_t(X)^T \nabla_{f}\mathbb{L}(f_t(X), Y),
% $$
% where $\eta = \frac{\textup{LR}}{n}$. Hence, we obtain
% $$
% \|\delta \theta_l\|_2 = \eta \sqrt{ Z_t^T \hat{K}_l Z_t},
% $$
% where $Z_t := f_t(X) - Y$. At initialization, $Z_t$ is on average equal to $ - CY = \tilde{Y}$ (as we justified in \cref{sec:quantifying_role}). Now let us analyze the average magnitude of $\|\delta \theta_l\|_2$ in the two cases $l\ll L$ and $l \sim L$. 
% \begin{itemize}[leftmargin=*]
%     \item $l\ll L$: in this case, from \cref{thm:backward_loss}, we know that 
% \end{itemize}

\paragraph{Early training steps.} As shown in \cref{fig:align_example}, the alignments $(A_l)_{1 \leq l \leq L}$ sharply increase at early stages of training ($\approx 2$ epochs of batch training), and plateau soon afterwards. We hence focus on the evolution of the alignment at early training (let $T$ denote the total number of training epochs). 

For ReLU activation (and more generally piecewise linear activations), we can express the gradient updates in terms of the output hessian matrix.
\begin{manualthm}{A}\label{thm:ckhess}
Gradient updates are given by
\begin{equation} \label{quantities}
    \begin{aligned}
   \theta(t+1) = \left(\mathbf{I} - \frac{\eta}{L-1} \mathbf{H}_w(t) \right) \theta(t)
    \end{aligned}
\end{equation}
where, for $v \in \mathbb{R}^{nk}$, $\mathbf{H}_v = \sum_{x \in \mathcal{D}}\sum_{i =1}^{k} v_{x, i} \mathbf{H}^{i}(x)$, and $\mathbf{H}^{i}(x) = \frac{\partial^2 f^i_{\theta}(x)}{\partial \theta^2}$ is the output hessian evaluated at $x$. For $v_{x,i}$, x indexes the datapoint and i the component.
\end{manualthm}
\begin{proof}
Let $\mathbf{B}_{l_1,l_2}$ be the block in $\mathbf{H}^{i}(x)$ containing all entries of the form $\frac{\partial^2 f^i(x)}{\partial W_{l_1}^{jk} \partial W_{l_2}^{st}}$ where $ W_{l_1}^{jk}$ is one of layer $l_1$'s parameters, and $W_{l_2}^{st}$ one of layer $l_2$'s parameters. 

If $l_1 = l_2$, by piecewise linearity of $\phi$, $\mathbf{B}_{l_1,l_2}$ contains all zero entries. 

If $l_1 > l_2$, fixing $W_{l_1}^{jk}$,
\begin{align*}
    \frac{\partial^2 f^i(x)}{\partial W_{l_1}^{jk} \partial W_{l_2}^{st}} = \frac{\partial\phi(z_{l1-1}^{k}(x))}{\partial W_{l_2}^{st}} \frac{\partial f^i(x)}{\partial z_{l_1}}[j]
\end{align*}
where $\frac{\partial f^i(x)}{\partial z_{l_1}}[j]$ is the $j$-th entry of $\frac{\partial f^i(x)}{\partial z_{l_1}}$. Hence
\begin{align*}
    \sum_{s,t} \frac{\partial^2 f^i(x)}{\partial W_{l_1}^{jk} \partial W_{l_2}^{st}} W_{l_2}^{st} = \left( \sum_{s,t}\frac{\partial\phi(z_{l1-1}^{k}(x))}{\partial W_{l_2}^{st}}W_{l_2}^{st}\right) \frac{\partial f^i(x)}{\partial z_{l_1}}[j]
\end{align*}
By piecewise linearity of activation function, 
\begin{align*}
    \sum_{s,t} \frac{\partial^2 f^i(x)}{\partial W_{l_1}^{jk} \partial W_{l_2}^{st}} W_{l_2}^{st} = \phi(z_{l1-1}^{k}(x)) \frac{\partial f^i(x)}{\partial z_{l_1}}[j] =  \frac{\partial f^i(x)}{\partial W_{l_1}^{jk}}
\end{align*}
If $l_1 < l_2$, fixing $W_{l_1}^{jk}$,
\begin{align*}
    \frac{\partial^2 f^i(x)}{\partial W_{l_1}^{jk} \partial W_{l_2}^{st}} = \phi(z_{l1-1}^{k}(x)) \frac{\partial \frac{\partial f^i(x)}{\partial z_{l_1}}[j]}{\partial W_{l_2}^{st}}
\end{align*}
Using piecewise linearity of activation function again, we get:
\begin{align*}
    \sum_{s,t} \frac{\partial^2 f^i(x)}{\partial W_{l_1}^{jk} \partial W_{l_2}^{st}} W_{l_2}^{st} =\phi(z_{l1-1}^{k}(x)) \left( \sum_{s,t} \frac{\partial \frac{\partial f^i(x)}{\partial z_{l_1}}[j]}{\partial W_{l_2}^{st}}W_{l_2}^{st} \right) = \phi(z_{l1-1}^{k}(x)) \frac{\partial f^i(x)}{\partial z_{l_1}}[j] =  \frac{\partial f^i(x)}{\partial W_{l_1}^{jk}}
\end{align*}

Combining the above results we get: (fixing $l_1 \in \{1,..,L\}$, for any $W_{l_1}^{jk}$)
\begin{align*}
    \sum_{l_2 = 1}^{L} \sum_{s,t} \frac{\partial^2 f^i(x)}{\partial W_{l_1}^{jk} \partial W_{l_2}^{st}} W_{l_2}^{st} = (L-1) \frac{\partial f^i(x)}{\partial W_{l_1}^{jk}}
\end{align*}

The left hand side is the entry of $\mathbf{H}^{i}(x) \theta$ corresponding to parameter $W_{l_1}^{jk}$, the right hand side is the entry of $(L-1)\Psi^{i}(x) $ corresponding to parameter $W_{l_1}^{jk}$. We obtain
\begin{align*}
    \Psi^{i}(x) = \frac{1}{L-1} \mathbf{H}^{i}(x) \theta
\end{align*}
By this result, the gradient of loss w.r.t parameters can be written as:
% \begin{equation}
\begin{align*}
\nabla_{\theta} \mathcal{L}
&= \left(\frac{\partial \mathcal{L}}{\partial F} \frac{\partial F}{\partial \theta}\right)^T = (w^T \mathbf{\Psi}^T)^T = \mathbf{\Psi} w = \frac{\eta}{L-1} \mathbf{H}_w(t) \theta(t)\\
\end{align*}
% \end{equation}
Hence, gradient updates are given by
% \begin{equation} \label{quantities}
\begin{align*}
\theta(t+1) = \left(\mathbf{I} - \frac{\eta}{L-1} \mathbf{H}_w(t) \right) \theta(t)
\end{align*}
% \end{equation}
\end{proof}
\cref{thm:ckhess} provides a geometric interpretation of gradient updates. The directions in parameter space where the largest updates occur are controlled by $\mathbf{H}_w$. A similar result holds for tangent features evolution during early training stages. Let us first introduce a key approximation.
\begin{approximation}[Collinearity at early training stages]\label{approx1_feature_evolution}
At early stages of training, $\tilde{Y}$ and $w$ are almost negatively co-linear. Specifically, $w \approx  - \frac{\|w\|}{\|\tilde{Y}\|} \tilde{Y}$. As a result vectors $\mathbf{\Psi}(t+1) \tilde{Y} - \mathbf{\Psi}(t) \tilde{Y}$ and $- \eta \mathbf{H}_w (t) \mathbf{\Psi}(t) \tilde{Y}$ are highly correlated. 
\end{approximation}
See \cref{justification:approx1_feature_evolution} for a theoretical justification of \cref{approx1_feature_evolution} with empirical evidence of its validity.
% Furthermore, \cref{approx1} yields the following approximation for tangent features evolution during early training stages.
% \begin{approximation}[Feature evolution]\label{approx:feature_evolution}
% The vectors $\mathbf{\Psi}(t+1) \tilde{Y} - \mathbf{\Psi}(t) \tilde{Y}$ and $- \eta \mathbf{H}_w (t) \mathbf{\Psi}(t) \tilde{Y}$ are highly correlated. 
% \end{approximation}
% We refer the reader to \cref{justification:feature_evolution} for more theoretical insights on \cref{approx:feature_evolution}. \cref{fig:approx_feature_evol} shows the correlation between the vectors $U \left(\mathbf{\Psi}(t+1) \tilde{Y} - \mathbf{\Psi}(t) \tilde{Y} \right)$ and $ - U \mathbf{H}_w (t) \mathbf{\Psi}(t) \tilde{Y}$ during training for a depth 10 FFNN on CIFAR10, where $U$ contains top 50 eigenvectors of $U$ as its rows. As expected, they are highly correlated, especially during the early training stages.\\
% \begin{theorem}[Feature evolution]\label{thm:feature_evolution}
% Under \cref{approx1}, a first order approximation of the tangent feature yields
% \begin{equation}\label{eq:feature_evolution}
%     \mathbf{\Psi}(t+1) \tilde{Y}
%     \approx \left(\mathbf{I} - \eta \mathbf{H}_w (t)\right) \mathbf{\Psi}(t) \tilde{Y}
% \end{equation}
% \end{theorem}

\cref{thm:ckhess} and \cref{approx1_feature_evolution} 
% , the directions of the largest changes in parameters $\theta$ are highly correlated with the directions of the largest changes in the vector $\mathbf{\Psi} \tilde{Y}$ which is the key component in the CKA. This 
provides a link between the EH and effective feature learning.
In \cref{sec:EH}, we explained our intuition on how efficient parameter updates take place in layers at information equilibrium. On the other hand, \cref{thm:ckhess} and \cref{approx1_feature_evolution} suggest that the directions with the largest updates in parameters $\theta$ are highly correlated with the directions for the largest updates in $\mathbf{\Psi}\tilde{Y}$. Hence, $\mathbf{\Psi}_l \tilde{Y}$ are updated more effectively at layers with information equilibrium, leading to larger alignment value. 
%See \cref{app:effective_updates} for more detail.
 
% \begin{figure}[H]
% \centering
%   \includegraphics[width=\linewidth]{fig/correlation.png}
%   \caption{Placeholder caption.}
% \label{fig:corr}
% \end{figure}
% \begin{itemize}
%     \item Experiment 1: show that $\Psi \tilde{Y}$ evolves in that way: plot alignment between \fTBD{Is it possible to show that $\Psi \tilde{Y}$ and $\theta$ evolve approximately in the same directions? \textcolor{red}{yes it is, we could eigendecompose $H_w$ and calculate some alignments, I have already done that with MNIST, FMNIST and cifar10. Plotting it now}}
%     % \item Experiment 2: show that highest change in parameters correspond to highest increase in alignment
%     \item illustrate bias in eigenvector of $H_w$
% \end{itemize}

To confirm this intuition, we train an depth 10 FFNN on CIFAR10 and show in \cref{fig:eigenvectors_H} the norm of the top 100 eigenvectors of $\mathbf{H}_w$ (corresponding to top 100 eigenvalues in absolute value) projected to 3 layers, extreme layer 1 and 10 and middle layer 7 (by projection to layer $l$ we refer to the truncation of the vector to leave just the sub-vector that corresponds to layer $l$). During the sharp increase phase in alignments (middle subfigure in \cref{fig:eigenvectors_H}), the top eigenvectors of the Hessian are more concentrated on intermediate layers, suggesting that the sharpest increase in alignment occurs in those intermediate layers.

\emph{Remark.} The hessian of the loss is given by $\frac{\partial^2 \mathcal{L}}{\partial \theta ^2} = \mathbf{H}_w + \mathbf{\Psi} \frac{\partial ^2 \mathcal{L}}{\partial F^2} \mathbf{\Psi}^T:= S + I$. Previous works on second order geometry of DNNs (e.g. \cite{karakida2019universal, ghorbani2019investigation}) focused on large positive eigenvalues of the loss hessian arising mostly from $I$. \cref{approx1_feature_evolution} shows that at the other end of the spectrum, large negative eigenvalues arising from $S$ influences feature learning in \ref{feat_evo_1}. In particular, eigenvectors of $S$ associated with large negative eigenvalues are directions of significant increase in alignment between features and labels.

\begin{figure}[H]
\centering
\begin{subfigure}{.31\textwidth}
\includegraphics[width=\linewidth]{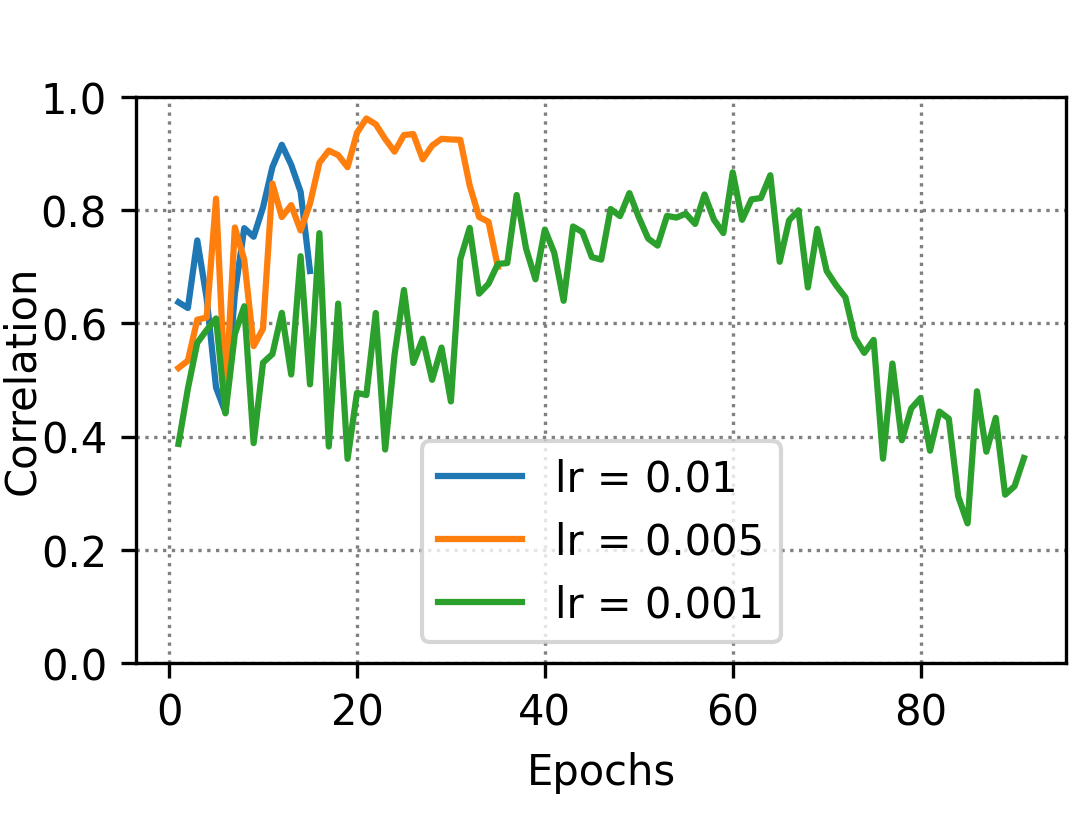}
  \caption{Empirical verification of \cref{approx1_feature_evolution}}
\label{fig:approx_feature_evol}
\end{subfigure}
% \begin{subfigure}{.32\textwidth}
%   \includegraphics[width=\linewidth]{fig/acc_vs_lr2.png}
%   \caption{Generalization VS Hierarchy \textcolor{red}{Seems that the constant in $\Theta(L^{3/5})$ is sensitive to the learning rate.}}
% \label{fig:e_vs_lr}
% \end{subfigure}
\begin{subfigure}{.32\textwidth}
   \includegraphics[width=\linewidth]{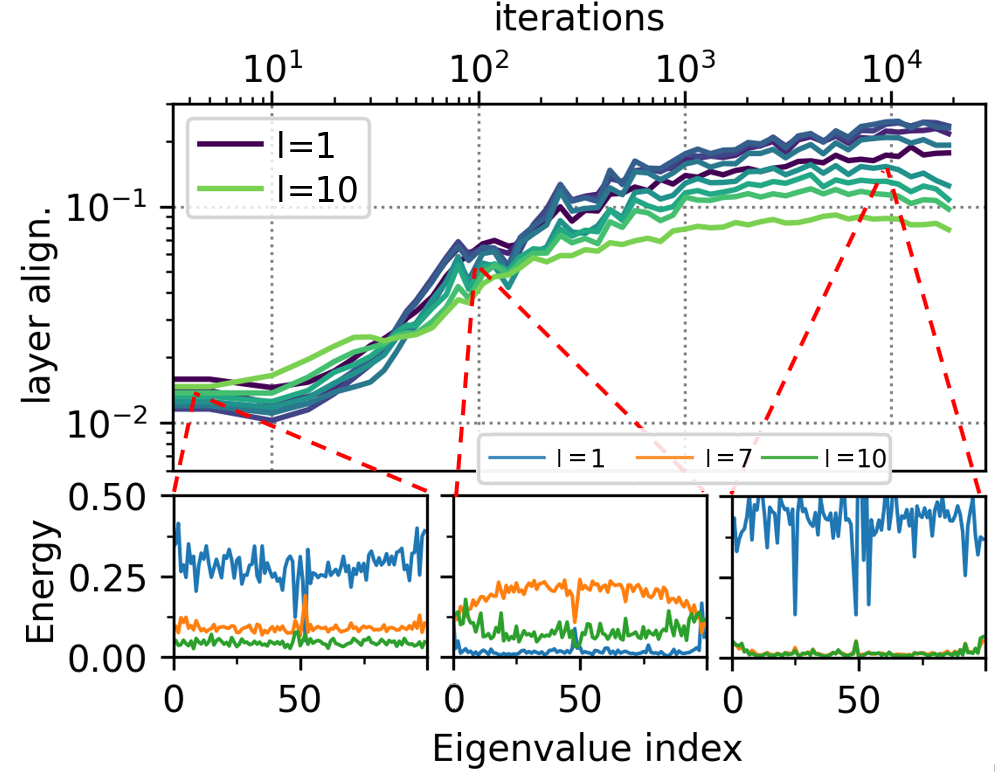}
  \caption{Alignment hierarchy during training}
\label{fig:eigenvectors_H}
\end{subfigure}
\caption{\small{\textbf{(a)} Correlation between vectors $U \left(\mathbf{\Psi}(t+1) \tilde{Y} - \mathbf{\Psi}(t) \tilde{Y} \right)$ and $ - U \mathbf{H}_w (t) \mathbf{\Psi}(t) \tilde{Y}$ for CIFAR10 on a 10-layer FFNN. 
%\textbf{(b)} Better layer alignment leads to better generalisation error, with a 10-layer FFNN on CIFAR10. Bottom panels show layer alignment for a selection of learning rates. 
\textbf{(c)} The norm of the eigenvectors of $\mathbf{H}_w$ (energy in the plot) related to top 100 eigenvalues in absolute value projected to 3 layers: 1, 7, 10 at three training times. The top eigenvalues in absolute value of $\mathbf{H}_w$ are 2, 15, 6 resp ($\mathbf{H}_w$ has a symmetric eigenspectrum shown in \cite{geiger2019jamming}) for three training times. The plot provides an illustration of why alignment hierarchy arises as more energy of top eigenvectors concentrate on intermediate layers during critical increase in alignments.}}
\end{figure}

\subsection{Justification of \cref{approx1_feature_evolution}}\label{justification:approx1_feature_evolution}

\textbf{\cref{approx1_feature_evolution}} (Collinearity at early training stages)
\textit{(1) At early stages of training, $\tilde{Y}$ and $w$ are almost negatively co-linear. Specifically, $w \approx  - \frac{\|w\|}{\|\tilde{Y}\|} \tilde{Y}$. (2) As a result vectors $\mathbf{\Psi}(t+1) \tilde{Y} - \mathbf{\Psi}(t) \tilde{Y}$ and $- \eta \mathbf{H}_w (t) \mathbf{\Psi}(t) \tilde{Y}$ are highly correlated.}

\subsubsection{Justification of (1)}

% Intuitively speaking, this approximation holds as most training datapoints\fTBD{we have not specified anything about the dataset. Are you assuming datapoints are of similar difficulty? if yes you should make it clear and define what you mean by difficulty.} are of similar difficulty to learn for a well-designed neural network, and at small times, they are optimized uniformly\fTBD{I couldnt understand what you mean. Does "they are optimized uniformly" refer to datapoints?}. 
The intuition behind \cref{approx1_feature_evolution} has roots in the assumption that the dataset is balanced. To see this, let $(x,y) \in \dataset$ and $\mathbb{L}_x = \mathbb{L}(f_\theta(x), y)$ the cross-entropy loss for the datapoint $(x,y)$. Let $c$ be the true label of $x$, and $i \in \{1,..,k\}$ such that $i\neq c$. We have that
\begin{equation}\label{eq:gradient_loss_output_}
\begin{aligned}
    \nabla_{f^i_{\theta}(x)} \mathcal{L} &= \frac{\exp \left(f^{i}_{\theta}(x)\right)}{\sum_{j = 1}^{k} \operatorname{exp}\left(f^{j}_{\theta}(x)\right)}, \;\; i\neq c\\
    % \nabla_{f^c_{\theta}(x)} \mathcal{L} &= \frac{\exp \left(f^{c}_{\theta}(x)\right)}{\sum_{j = 1}^{k} \operatorname{exp}\left(f^{j}_{\theta}(x)\right)} - 1 \\
    \nabla_{f^c_{\theta}(x)} \mathcal{L} &= - \sum_{i \neq c} \frac{\exp \left(f^{i}_{\theta}(x)\right)}{\sum_{j = 1}^{k} \operatorname{exp}\left(f^{j}_{\theta}(x)\right)}
\end{aligned}
\end{equation}
% \begin{align*}
%     \nabla_{f^c_{\theta}(x)} \mathcal{L} = \frac{\exp \left(f^{c}_{\theta}(x)\right)}{\sum_{j = 1}^{C} \operatorname{exp}\left(f^{j}_{\theta}(x)\right)} - 1
% \end{align*}
where $f^i_{\theta}(x)$ is the i-th entry of  $f_{\theta}(x)$. When the dataset is balanced, i.e. the number of datapoint per class is approximately the same for all classes, the corresponding entries of $\tilde{Y}$ satisfy $\tilde{Y}_{x,c} \approx \frac{k-1}{k}$ and $\tilde{Y}_{x,i} \approx \frac{-1}{k}$. At initialization, with random weights, ${f^i_{\theta}(x)}$ are on average similar across choices of $x$ and $i$. Hence, using \cref{eq:gradient_loss_output_}, on average we have $w_{x,c} \approx -\frac{k-1}{k}$ and $w_{x,i} \approx \frac{1}{k}$, thus, $\tilde{Y}$ are almost negatively co-linear. \cref{fig:justify_approx1} illustrates this result on various architectures and datasets (see also \cref{app:further_experiments}). We observe that \cref{approx1_feature_evolution} also holds during early training steps ($\textup{corr}(w,\tilde{Y})\approx -1$ during the first training epoch).

\begin{figure}
\centering
\begin{subfigure}{.23\textwidth}
  \centering
  \includegraphics[width=\textwidth]{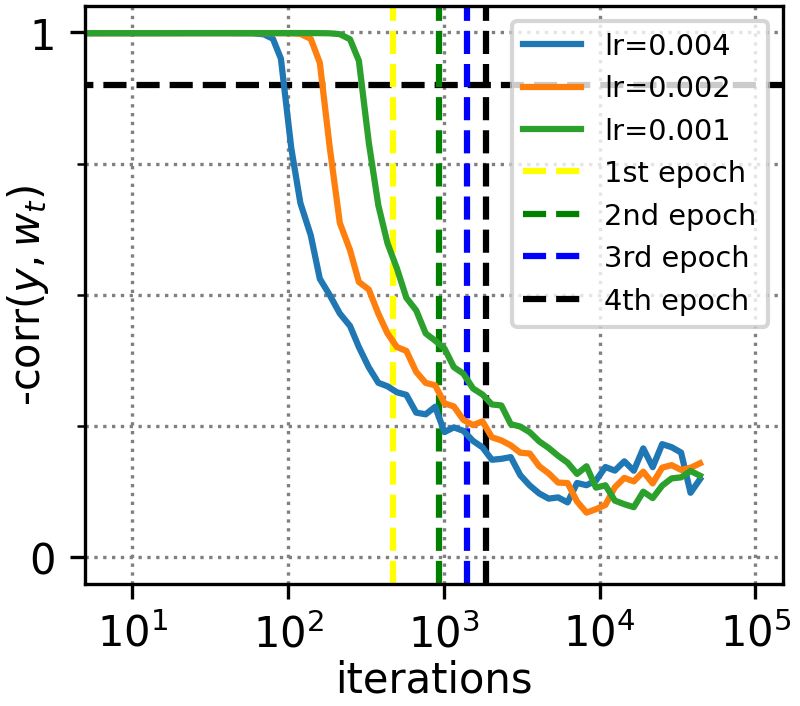}
  \caption{MNIST FFNN}
  \label{fig:sapprox1_1}
\end{subfigure}
\begin{subfigure}{.23\textwidth}
  \centering
  \includegraphics[width=\textwidth]{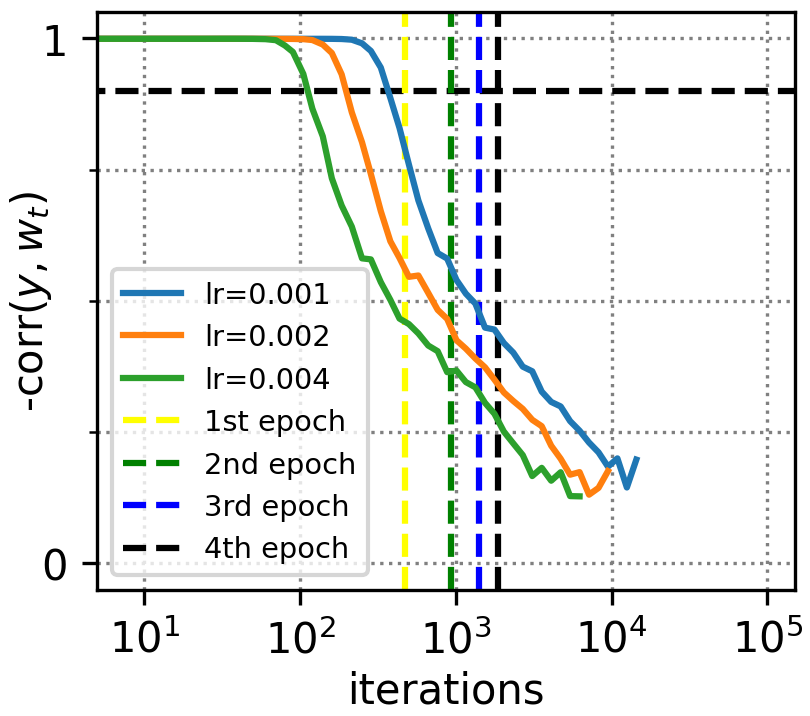}
  \caption{KMNIST FFNN}
  \label{fig:sapprox1_2}
\end{subfigure}
\begin{subfigure}{.23\textwidth}
  \centering
  \includegraphics[width=\textwidth]{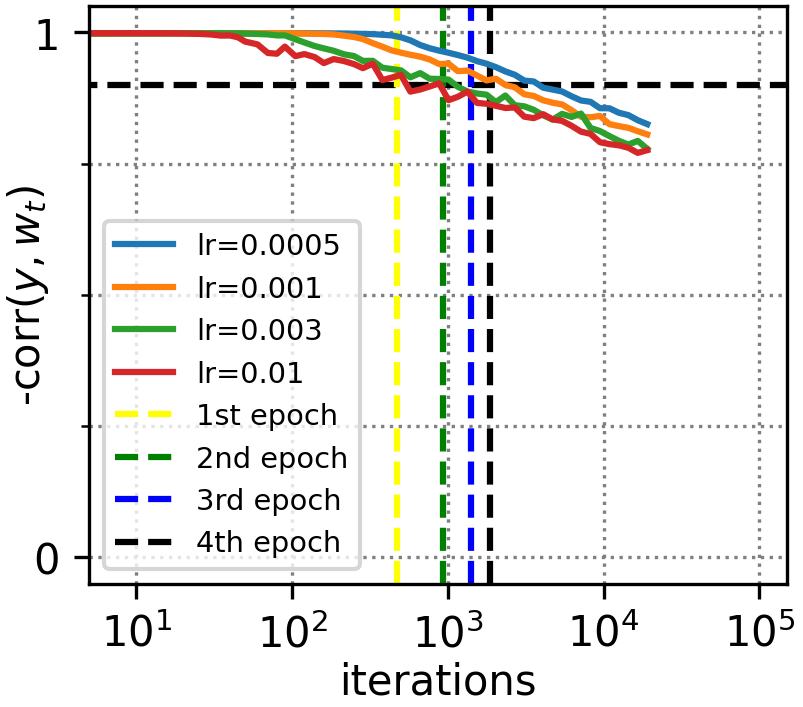}
  \caption{CIFAR10 FFNN}
  \label{fig:sapprox1_3}
\end{subfigure}
\begin{subfigure}{.23\textwidth}
  \centering
  \includegraphics[width=\textwidth]{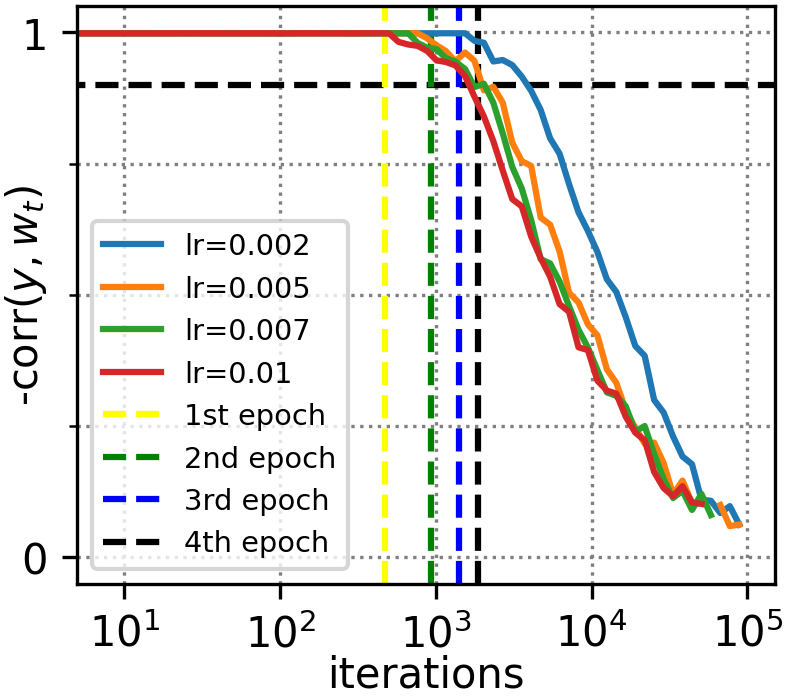}
  \caption{CIFAR10 VGG19}
  \label{fig:sapprox1_4}
\end{subfigure} 

\begin{subfigure}{.23\textwidth}
  \centering
  \includegraphics[width=\textwidth]{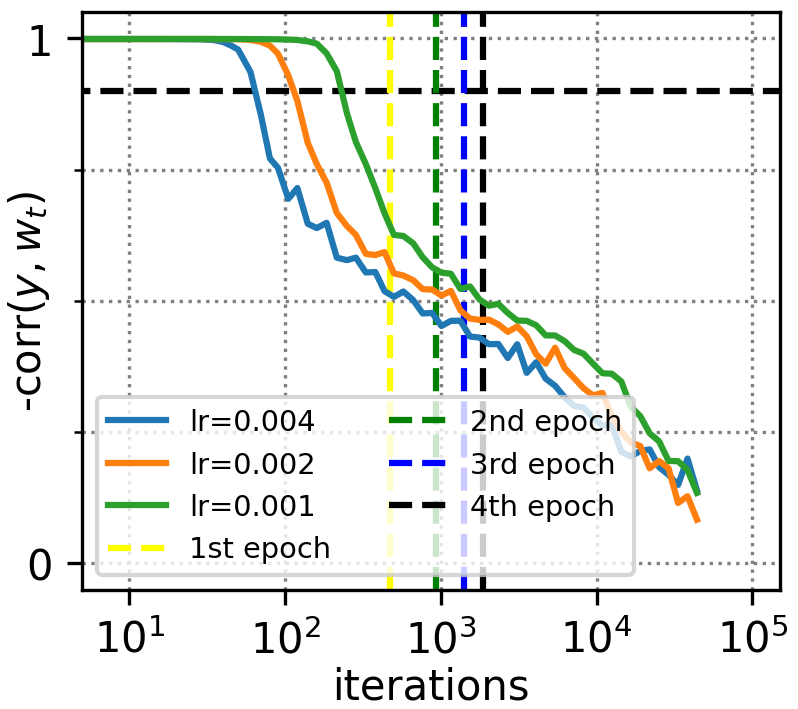}
  \caption{FashionMNIST FFNN}
  \label{fig:sfig1}
\end{subfigure}
\begin{subfigure}{.23\textwidth}
  \centering
  \includegraphics[width=\textwidth]{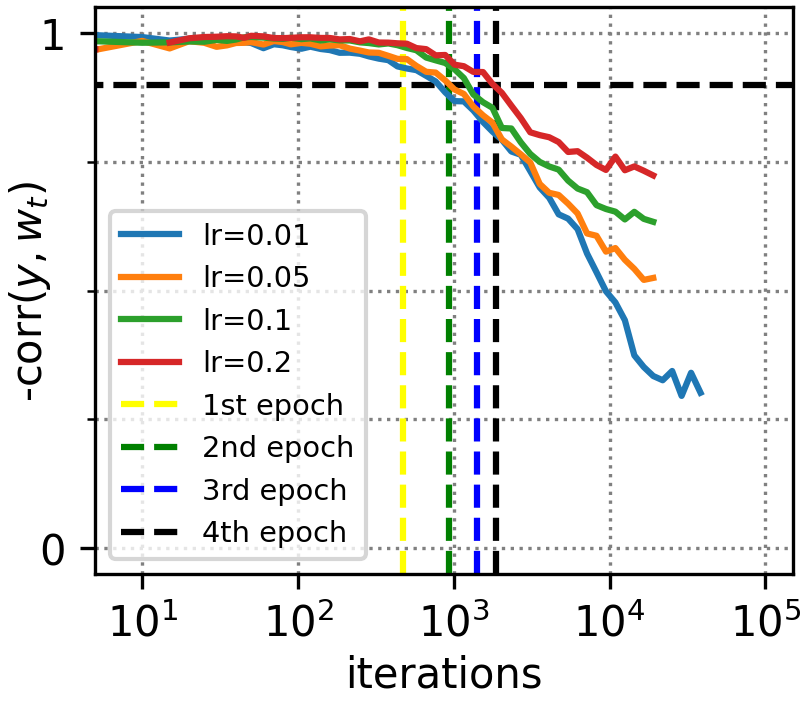}
  \caption{CIFAR100 ResNet18}
  \label{fig:sfig3}
\end{subfigure}
\begin{subfigure}{.23\textwidth}
  \centering
  \includegraphics[width=\textwidth]{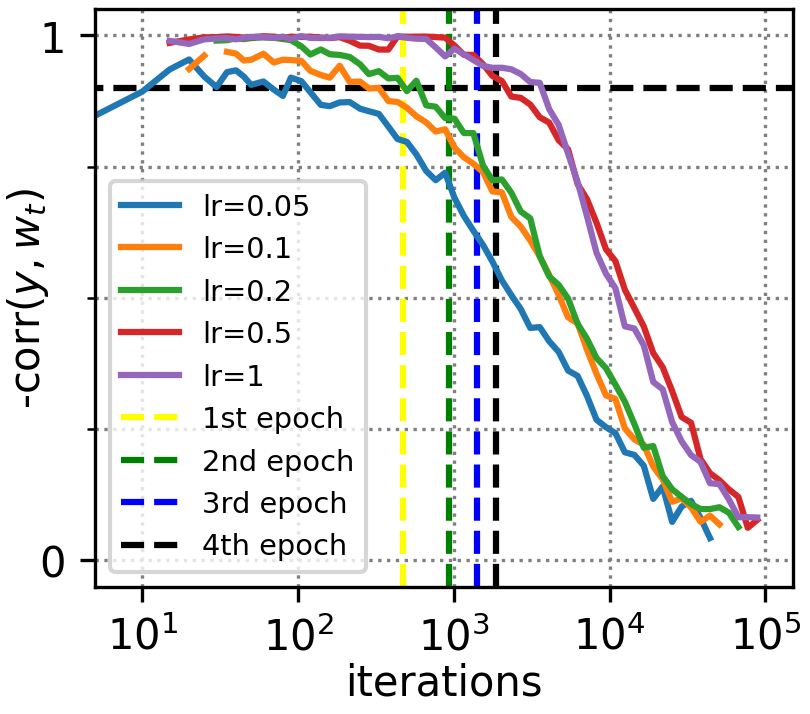}
  \caption{CIFAR10 ResNet18}
  \label{fig:sfig4}
\end{subfigure}
\begin{subfigure}{.23\textwidth}
  \centering
  \includegraphics[width=\textwidth]{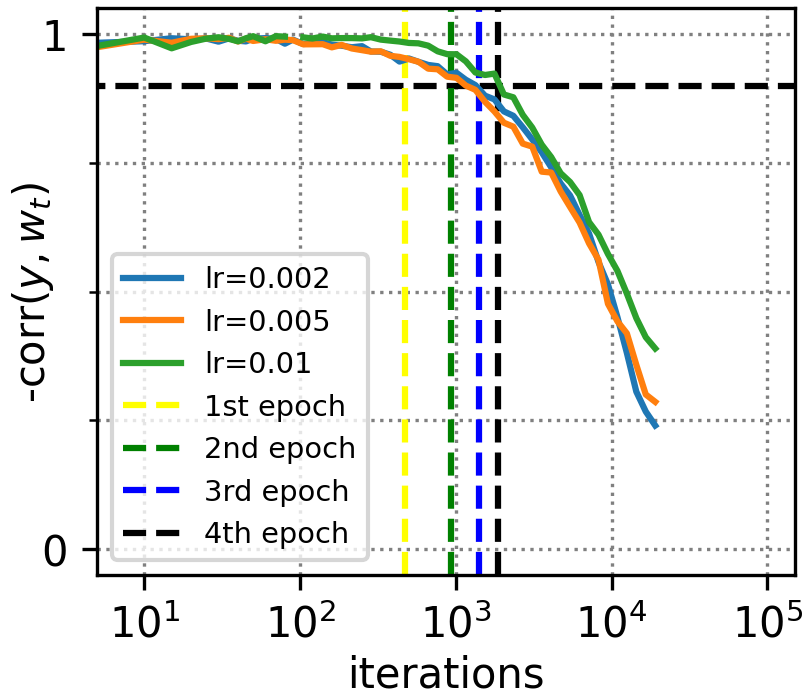}
  \caption{CIFAR100 VGG19}
  \label{fig:sfig2}
\end{subfigure} 

\caption{\small{Empirical validation of \cref{approx1_feature_evolution}: a demonstration that $-\textrm{corr}(\tilde{Y},w_t) \approx 1$ at early $t$.} %\textcolor{red}{Are we gonna include AH results on CIFAR100 with VGG19/RESNET18 in the appendix? otherwise it doesnt make sense to include empirical validation on CIFAR100 with VGG19/RESNET18}
}\label{fig:justify_approx1}
\end{figure}

\subsubsection{Justification of (2)}\label{justification:feature_evolution}

(2) $\mathbf{\Psi}(t+1) \tilde{Y} - \mathbf{\Psi}(t) \tilde{Y}$ and $- \eta \mathbf{H}_w (t) \mathbf{\Psi}(t) \tilde{Y}$ are highly correlated.

We will first need \cref{approx_time_lr}.

\begin{approximation}[$1^{st}$ order approximation]\label{approx_time_lr}
\begin{align*}
    \Psi^{i}(x)(t+1) - \Psi^{i}(x)(t) = \mathbf{H}^{i}(x)(t) 
    \left(\theta(t+1) - \theta(t)\right) + \mathcal{E}_{x, i}(t)
\end{align*}
\end{approximation}
where $ \mathcal{E}_{x, i}(t)$ includes higher order terms of $\theta(t+1) - \theta(t)$.
We will first justify \cref{approx_time_lr}:
\begin{align*}
    \Psi^{i}(x)(t+1) - \Psi^{i}(x)(t)
    & = \left.\frac{\partial f^{i}(x) }{\partial \theta}\right|_{\theta(t+1)} - \left.\frac{\partial f^{i}(x) }{\partial \theta}\right|_{\theta(t)}\\
   & = \left.\frac{\partial^2 f^{i}(x) }{\partial \theta^2}\right|_{\theta(t)} \left(\theta(t+1) - \theta(t)\right) +\mathcal{E}_{x, i}(t)\\
   &=
   \mathbf{H}^{i}(x)(t) 
    \left(\theta(t+1) - \theta(t)\right)+\mathcal{E}_{x, i}(t)
\end{align*}
The second step is by Taylor expanding $\left.\frac{\partial f^{i}(x) }{\partial \theta}\right|_{\theta(t+1)}$ around $\theta(t)$. Given \cref{approx_time_lr}, for an arbitrary fixed vector $v \in \mathbb{R}^{kn}$, the evolution of $\mathbf{\Psi } v$ can be approximated by:
\begin{equation}
    \begin{aligned}
    % \frac{\partial \theta }{\partial t} &= - \eta \mathbf{\Psi} w\\
    % \frac{\partial \mathbf{\Psi} v}{\partial t} & =\left( \sum_{x, i} v_{x,i}\frac{\partial \mathbf{\Psi}^i(x) }{\partial \theta} \right) \frac{\partial \theta}{\partial t}\\
    % &= \mathbf{H}_v \frac{\partial \theta }{\partial t}\\
    % & = - \eta \mathbf{H}_v \mathbf{\Psi} w\\
    \mathbf{\Psi} (t+1) v - \mathbf{\Psi} (t) v 
    % & = \sum_{x \in \mathcal{D}} \sum_{j = 1}^{C} v_{x,i} \frac{\partial  \mathbf{\Psi}^i(x)}{\partial t}(t+\epsilon_{x,i})\\
    & = \sum_{x \in \mathcal{D}} \sum_{j = 1}^{k} v_{x,i} \left(\mathbf{H}^{i}(x)(t)\left(\theta(t+1) - \theta(t)\right)+\mathcal{E}_{x, i}(t)\right)\\
    % & =\sum_{x \in \mathcal{D}} \sum_{j = 1}^{C} v_{x,i} \frac{\partial  \mathbf{\Psi}^i(x)}{\partial \theta} \frac{\partial \theta}{\partial t}(t)\\
    & = -\eta \sum_{x \in \mathcal{D}} \sum_{j = 1}^{k} v_{x,i} \mathbf{H}^{i}(x)(t) \mathbf{\Psi}(t) w + \sum_{x \in \mathcal{D}} \sum_{j = 1}^{k} v_{x,i} \mathcal{E}_{x, i}(t)\\
    & = -\eta \mathbf{H}_v(t) \mathbf{\Psi}(t) w+ \sum_{x \in \mathcal{D}} \sum_{j = 1}^{k} v_{x,i} \mathcal{E}_{x, i}(t)
    \end{aligned}
\end{equation}
Setting $v = \tilde{Y}$ we get:
% \begin{equation} \label{feat_evo}
%     \begin{aligned}
%     \frac{\partial \mathbf{\Psi} y}{\partial t}
%     & = - \eta \mathbf{H}_y \mathbf{\Psi} w\\
%     & = - \eta \mathbf{H}_w \mathbf{\Psi} y\\
%     \end{aligned}
% \end{equation}
% In the last step we use Assumption \ref{approx1}. Turning this equation again into a discrete form using Assumption \ref{approx_time_lr}, we yield:
\begin{equation} \label{feat_evo_1}
    \begin{aligned}
    \mathbf{\Psi}(t+1) \tilde{Y}
     & = \mathbf{\Psi}(t) \tilde{Y} - \eta \mathbf{H}_{\tilde{Y}} \mathbf{\Psi}(t) w + \sum_{x \in \mathcal{D}} \sum_{j = 1}^{k} \tilde{Y}_{x,i} \mathcal{E}_{x, i}(t)\\
    & \approx \mathbf{\Psi}(t) \tilde{Y} - \eta \mathbf{H}_w \mathbf{\Psi}(t) \tilde{Y}+ \sum_{x \in \mathcal{D}} \sum_{j = 1}^{k} \tilde{Y}_{x,i} \mathcal{E}_{x, i}(t)\\
    &= \left(\mathbf{I} - \eta \mathbf{H}_w \right) \mathbf{\Psi}(t) \tilde{Y} + \sum_{x \in \mathcal{D}} \sum_{j = 1}^{k} \tilde{Y}_{x,i} \mathcal{E}_{x, i}(t)
    \end{aligned}
\end{equation}
In the second step we use \cref{approx1_feature_evolution} (1). As $\sum_{x \in \mathcal{D}} \sum_{j = 1}^{k} \tilde{Y}_{x,i} \mathcal{E}_{x, i}(t)$ contains high order terms in $\eta$, this is small compared to $\left(\mathbf{I} - \eta \mathbf{H}_w \right) \mathbf{\Psi}(t) \tilde{Y}$. Hence, the update in feature vector $\mathbf{\Psi}(t+1) \tilde{Y} - \mathbf{\Psi}(t) \tilde{Y}$ is highly correlated with $- \eta \mathbf{H}_w \mathbf{\Psi}(t) \tilde{Y}$.

\section{A Justification of the choice of Tangent Features and the CKA alignment}\label{app:justif_tf}
Let $(x,y) \in \dataset$ and $\mathbb{L}_x = \mathbb{L}(f_\theta(x), y)$ the cross-entropy loss for the datapoint $(x,y)$. Let $c$ be the true label of $x$, and $i \in \{1,..,k\}$ such that $i\neq c$. We have that
\begin{equation}\label{eq:gradient_loss_output}
\begin{aligned}
    \nabla_{f^i_{\theta}(x)} \mathbb{L} &= \frac{\exp \left(f^{i}_{\theta}(x)\right)}{\sum_{j = 1}^{k} \operatorname{exp}\left(f^{j}_{\theta}(x)\right)}, \;\; i\neq c\\
    % \nabla_{f^c_{\theta}(x)} \mathcal{L} &= \frac{\exp \left(f^{c}_{\theta}(x)\right)}{\sum_{j = 1}^{k} \operatorname{exp}\left(f^{j}_{\theta}(x)\right)} - 1 \\
    \nabla_{f^c_{\theta}(x)} \mathbb{L} &= - \sum_{i \neq c} \frac{\exp \left(f^{i}_{\theta}(x)\right)}{\sum_{j = 1}^{k} \operatorname{exp}\left(f^{j}_{\theta}(x)\right)}
\end{aligned}
\end{equation}

Hence, for any datapoint $(x,y)$, we have that

$$\nabla_{f_{\theta}(x)} \mathbb{L} = \textup{softmax}(f(x))  - y.$$

It is straightforward that update of the network output (evaluated on the whole dataset $(X,Y)$) with one gradient step is given by (see e.g. \cite{jacot2018NTK})

$$df_t(X) = - \eta \hat{\mathbf{K}}  \nabla_y\mathbb{L}(f_t(X), Y),$$
where $\hat{\mathbf{K}}^L$ is the tangent kernel matrix, $f_t(X), Y \in \reals^{on}$ are the concatenated vectors of $(f(x_i))_{1 \leq i \leq n}$ and $(y)_{1 \leq i \leq n}$, and $\eta = \textup{LR}/n$ is the normalized learning rate. Consider the case of large width $N \gg 1$. At initialization, on average, the network output is a random classifier. 
$$
\mathbb{E}_W\left[\frac{d}{dt} f_t(X) \right] \approx \hat{\mathbf{K}} C Y,
$$
where we have used the gradient independence result from \cref{lemma:gradient_independence} and the approximation that the NTK is almost deterministic in the large width limit (see \cite{jacot2018NTK}). This yields
$$
\mathbb{E}_W\left[\frac{d}{dt} y^T f_t(X) \right] \approx  y^T \hat{\mathbf{K}} y.
$$
The alignment between the tangent kernel and data labels has a direct impact on how fast the network output aligns with the true labels.\\

To measure the role played by each layer, let us consider the scenario when we freeze all but the $l^{th}$ layer parameters, the previous dynamics become 
$$
\mathbb{E}_W\left[\frac{d}{dt} y^T \Tilde{f}_t(X) \right] \approx  y^T \hat{\mathbf{K}}_l y ,
$$
where $\hat{\mathbf{K}}_l$ is the tangent kernel matrix for layer $l$. Observe that 
$$y^T \hat{\mathbf{K}}_l y \leq \rho(\hat{\mathbf{K}}_l) \|y\|^2 \leq \textup{Tr}(\hat{\mathbf{K}}_l)\|y\|^2,$$
with equality if and only if the kernel matrix $\hat{\mathbf{K}}_l$ is perfectly aligned with the data labels matrix $y y^T$. Therefore, the alignment $A_l$ has a direct impact on the alignment between the data labels and output function (note that a perfect alignment between $y$ and $f_t(X)$ indicates $100\%$ classification accuracy).

% \section{Justification of the Equilibrium Hypothesis}
% Let us consider the simplest case with $o=1$ and the squared error loss function. The gradient flow dynamics yield

% $$df_t(X) = - \hat{\mathbf{K}}  (f_t(X), Y) dt,$$
% where $\hat{\mathbf{K}}$ is the tangent kernel matrix.

\section{Optimal Feature Evolution Scheme}
\citet{shan2021rapid} proposed Optimal Feature Evolution (OFE) scheme to model the evolution of tangent features during gradient descent (GD) training. Under OFE, the tangent features $\mathbf{\Psi}$ evolve greedily so that the change in empirical loss $\mathcal{L}$ is maximised at each time step.  However, it is not verified empirically if OFE matches any variants of GD methods. In the following, we propose Generalised Optimal Feature Evolution (GOFE) scheme to capture GD methods more closely and gain insights into the evolution of layerwise CKA.

\subsection{Optimal Feature Evolution with fixed learning rates}
In \cite{shan2021rapid}, the optimal feature evolution paradigm is only given for MSE loss. The following is a summary of OFE evolution scheme for any twice differentiable loss $\mathcal{L}$. We inherit notation from \cref{app:proofs}.

By GD training dynamics we have:

\begin{equation}
    \begin{aligned}
    \frac{\partial \theta^T}{\partial t} &= - \eta \frac{\partial \mathcal{L}}{\partial \theta} = - \eta \frac{\partial \mathcal{L}}{\partial F} \frac{\partial F}{\partial \theta}\\
    & = - \eta w^T \mathbf{\Psi}^T\\
    \Rightarrow \frac{\partial \theta}{\partial t} &= - \eta \mathbf{\Psi} w
    \end{aligned}
\end{equation}

The evolution of $\mathcal{L}$ is:

\begin{equation}
    \begin{aligned}
    \frac{\partial \mathcal{L}}{\partial t} &= \frac{\partial \mathcal{L}}{\partial F} \frac{\partial F}{\partial \theta} \frac{\partial \theta}{\partial t} = - \eta w^T \mathbf{\Psi}^T \mathbf{\Psi} w\\
    \end{aligned}
\end{equation}

Using the same argument as in \cite{shan2021rapid}, we optimise the term $w^T \mathbf{\Psi}^T \mathbf{\Psi} w$ w.r.t $\mathbf{\Psi}$ by evolving $\mathbf{\Psi}$ in the direction of largest decrease in $-w^T \mathbf{\Psi}^T \mathbf{\Psi} w$ with a learning rate of $\lambda$. This yield:

\begin{equation}
    \begin{aligned}
    \frac{\partial \mathbf{\Psi}^{T}}{\partial t}&= - \lambda \frac{\partial \left(- \eta w^T \mathbf{\Psi}^T \mathbf{\Psi} w\right)}{\partial \mathbf{\Psi}} \\
    \Rightarrow \frac{\partial \mathbf{\Psi}}{\partial t}&= 2 \lambda \eta \mathbf{\Psi} w w^{\top} \\
    \end{aligned}
\end{equation}

We could absorb the $2$ factor in this equation into $\lambda$ to produce the dynamics:

\begin{equation} \label{OFE}
    \begin{aligned}
    \frac{\partial \mathbf{\Psi}}{\partial t}&=\lambda \eta \mathbf{\Psi} w w^{\top}\\
    \frac{\partial w}{\partial t}&=- \eta \frac{\partial w}{\partial \theta} \mathbf{\Psi}^T \mathbf{\Psi} w 
    \end{aligned}
\end{equation}

At a first look, this model adds an interesting layer of complexity over fixed tangent kernel learning: the kernel evolves in predictable and alignment-boosting ways. However, there's no strong empirical evidence that OFE captures any variant of GD training.

% A direct generalisation of this dynamic is where layerwise tangent vectors evolve at a layer-specific learning rate:

% \begin{equation} \label{OFE2}
%     \begin{aligned}
%     \frac{\partial \mathbf{\Psi}_l}{\partial t}&=\lambda_l \eta \mathbf{\Psi} w w^{\top}\\
%     \frac{\partial w}{\partial t}&=- \eta \frac{\partial w}{\partial \theta} \mathbf{\Psi}^T \mathbf{\Psi} w = - \eta \frac{\partial w}{\partial \theta} \sum_{l }  \mathbf{\Psi}_l^T \mathbf{\Psi}_l w
%     \end{aligned}
% \end{equation}

\subsection{Generalised Optimal Feature Evolution and GD training}
To better capture GD/SGD/NGD dynamics, we first introduce the paradigm of Generalised Optimal Feature Evolution (GOFE):

% The generalised feature evolution scheme:

\begin{equation}
    \begin{aligned}
    \frac{\partial \mathbf{\Psi}}{\partial t}
    &= \eta \mathbf{V} \mathbf{\Psi} w w^T \\
    % \frac{\partial \theta}{\partial t}&= -\eta \Psi \Delta\\
    \frac{\partial w}{\partial t}&= -\eta \frac{\partial w}{\partial F} \mathbf{\Psi}^T \mathbf{A} \mathbf{\Psi} w = - \eta\frac{\partial ^2 \mathcal{L}}{\partial F^2} \mathbf{\Psi}^T \mathbf{A} \mathbf{\Psi} w
    \end{aligned}
\end{equation}

where $\mathbf{V}$ is a velocity vector that may or may not depend on time step $t$, $\mathbf{A}$ is a time-dependent matrix describing the training procedure used (i.e. for full batch GD, $\mathbf{A}$ is simply the identity matrix, and for natural gradient descent, $\mathbf{A}$ is the inverse of Fisher Information Matrix). $\frac{\partial w}{\partial F}$ gradient of $w$ w.r.t to the output F. The last equality is due to definition of $w^T = \frac{\partial \mathcal{L}}{\partial F}$. Note that in OFE, $\mathbf{V}$ is simply a diagonal matrix with diagonal entries $\lambda$.
Note that in practice the feature evolution is realised by the set of difference equations:

\begin{equation}
    \begin{aligned}
    \Delta_t( \mathbf{\Psi}) &= \eta \mathbf{V}_t \mathbf{\Psi}_t w_t w_t^T \\
    \Delta_t(\theta) & = - \eta \mathbf{A}_t \mathbf{\Psi}_t w_t\\
    % \frac{\partial \theta}{\partial t}&= -\eta \Psi \Delta\\
    % \frac{\partial w}{\partial t}&= -\eta \frac{\partial w}{\partial F} \mathbf{\Psi}^T \mathbf{A} \mathbf{\Psi} w = - \eta\frac{\partial ^2 \mathcal{L}}{\partial F^2} \mathbf{\Psi}^T \mathbf{A} \mathbf{\Psi} w\\
    \end{aligned}
\end{equation}

This would allow us to conduct several calculations exactly in the following. The $t$ index each time step.

It turns out that for each gradient descent dynamics, be it full batch or stochastic GD or natural gradient descent, there is an equivalent formulation of its training dynamics in terms of GOFE. By 'equivalent' we mean that at each time step, the gradient propagated to the weights of the network is the same. To achieve this equivalence we need the following two dynamics of gradient changes to be the same:

Under GOFE:
\begin{equation}
    \begin{aligned}
    \frac{\partial (\mathbf{\Psi} w)}{\partial t}
    &= \frac{\partial (\mathbf{\Psi})}{\partial t}  w + \mathbf{\Psi}   \frac{\partial (w)}{\partial t}\\
    & = \eta \mathbf{V} \mathbf{\Psi} w w^T w - \eta \mathbf{\Psi} \frac{\partial ^2 \mathcal{L}}{\partial F^2} \mathbf{\Psi}^T \mathbf{A} \mathbf{\Psi} w \\
    & = \eta \|w\|^2 \mathbf{V} \mathbf{\Psi} w - \eta \mathbf{\Psi} \frac{\partial ^2 \mathcal{L}}{\partial F^2} \mathbf{\Psi}^T \mathbf{A}\mathbf{\Psi} w  \\
    \end{aligned}
\end{equation}

Under gradient descent with gradient adjustment matrix $\mathbf{A}$:
\begin{equation}
    \begin{aligned}
    \frac{\partial (\mathbf{\Psi} w)}{\partial t}
    &= \frac{\partial (\mathbf{\Psi} w)}{\partial \theta} \frac{\partial \theta}{\partial t} =\frac{\partial^2 \mathcal{L}}{\partial \theta ^2} (- \eta \mathbf{A} \Psi w)
    = - \eta \frac{\partial^2 \mathcal{L}}{\partial \theta ^2} \mathbf{A} \mathbf{\Psi} w\\
    & = -\eta \mathbf{H}_{w} \mathbf{A} \mathbf{\Psi} w - \eta \mathbf{\Psi} \frac{\partial ^2 \mathcal{L}}{\partial F^2} \mathbf{\Psi}^T \mathbf{A} \mathbf{\Psi} w
    % & = \eta V \Psi \Delta \Delta^T \Delta - \eta \frac{1}{N} \Psi \Psi^T \Psi \Delta \\
    % & = \|\Delta\|^2 \text{diag}(v_1 \eta, v_1\eta,...,v_l\eta, v_l\eta,...v_L \eta) \Psi \Delta - \eta I \Psi \Delta \\
    \end{aligned}
\end{equation}

During the derivation we have used a well known decomposition of the loss hessian:

\begin{equation}
    \begin{aligned}
    \frac{\partial^2 \mathcal{L}}{\partial \theta ^2} = \mathbf{H}_w + \mathbf{\Psi} \frac{\partial ^2 \mathcal{L}}{\partial F^2} \mathbf{\Psi}^T
    \end{aligned}
\end{equation}

where $\mathbf{H}_w$ is the same as $\mathbf{H}_w$ in \cref{app:EH_early_training}. Equating the two dynamics we need:

\begin{equation}
    \begin{aligned}
    &\eta \|w\|^2 \mathbf{V} \mathbf{\Psi} w - \eta \mathbf{\Psi} \frac{\partial ^2 \mathcal{L}}{\partial F^2} \mathbf{\Psi}^T \mathbf{A} \mathbf{\Psi} w = -\eta \mathbf{H}_{w} \mathbf{A} \mathbf{\Psi} w - \eta \mathbf{\Psi} \frac{\partial ^2 \mathcal{L}}{\partial F^2} \mathbf{\Psi}^T \mathbf{A} \mathbf{\Psi} w\\
    &\Longleftrightarrow (\|w\|^2 \mathbf{V} + \mathbf{H}_{w} \mathbf{A}) \mathbf{\Psi} w = 0\\
    % &\Longleftrightarrow (\|w\|^2 \mathbf{V} + \mathbf{H}_{w} \mathbf{A}) \mathbf{H}_{w} \theta  = 0
    \end{aligned}
\end{equation}

We hence set $\mathbf{V} = - \frac{\mathbf{H}_{w} \mathbf{A}}{\|w\|^2}$. Under the assumption that $w$ and $y$ are highly correlated at early training, we could directly derive the evolution of $\mathbf{\Psi}$. In fact, under GD ($\mathbf{A}$ being identity matrix):

\begin{align}
    \frac{\partial \mathbf{\Psi}}{\partial t}
    = - \eta \frac{\mathbf{H}_{w} \mathbf{A}}{\|w\|^2} \mathbf{\Psi} w w^T
    = - \eta \mathbf{H}_{w} \mathbf{\Psi} \frac{w w^T}{\|w\|^2}
    = - \eta \mathbf{H}_{w} \mathbf{\Psi} \frac{y y^T}{\|y\|^2}
    % \frac{\partial \theta}{\partial t}&= -\eta \Psi \Delta\\
\end{align}

We could test GOFE against results derived without it, in fact multiplying $y$ to the above equation yields:
\begin{equation}\label{eq:gofe_feature_evolution}
\begin{aligned} 
    \frac{\partial \mathbf{\Psi}y}{\partial t} &= - \eta \mathbf{H}_{w} \mathbf{\Psi} \frac{y y^T}{\|y\|^2} y
    = - \eta \mathbf{H}_{w} \mathbf{\Psi} y
\end{aligned}
\end{equation}

which is the continuous version of \cref{feat_evo_1}. Also for any fixed vector $u$ orthogonal to $y$, we have: 
\begin{align}
    \frac{\partial \mathbf{\Psi}u}{\partial t} &= - \eta \mathbf{H}_{w} \mathbf{\Psi} \frac{y y^T}{\|y\|^2} u
    = 0
\end{align}
This relates to the increase in tangent kernel anistropy in \cite{baratin2021implicit}, as $u^T\mathbf{\Psi}^T\mathbf{\Psi}u$ stays constant over training while $y^T\mathbf{\Psi}^T\mathbf{\Psi}y$ increases sharply due to large negative eigenvalues in $\mathbf{H}_w$.

\subsection{Explaining the Hierarchy using feature evolution scheme}
\cref{eq:gofe_feature_evolution} gives us a way to describe $\mathbf{\Psi}(t+1) y $ as $\mathbf{H}(t) \mathbf{\Psi}(t) y $ for some matrix $\mathbf{H}(t) = \left(\mathbf{I}- \eta \mathbf{H}_{w}(t)\right) $ which also describes evolution of parameters. Take an orthogonal basis consisting of $u_0 = \frac{y}{\|y\|}, u_1 = \frac{1}{\sqrt{kN} }(1,1...,1)^T, u_2, ..., u_N \in \mathbb{R}^{kN}$, $\mathbf{U}$ be the $kN \times kN$ matrix with $u_i$ as columns and we would have:
\begin{equation} \label{eq:hierarchy}
    \begin{aligned}
    A_l(t+1) &=  \frac{y^T \mathbf{\Psi}(t+1)^T \mathbf{M}_l \mathbf{\Psi}(t+1) y}{\|\mathbf{\Psi}(t+1)^T \mathbf{M}_l \mathbf{\Psi}(t+1) \mathbf{C}\|_F \|y\|^2}\\
    & =  A_l(t) \cdot \frac{u_0^T \mathbf{\Psi}(t+1)^T \mathbf{M}_l \mathbf{\Psi}(t+1) u_0}{u_0^T \mathbf{\Psi}(t)^T \mathbf{M}_l \mathbf{\Psi}(t) u_0} \cdot \frac{\|U^T \mathbf{\Psi}(t)^T \mathbf{M}_l \mathbf{\Psi}(t) \mathbf{C} U\|_F }{\|U^T \mathbf{\Psi}(t+1)^T \mathbf{M}_l \mathbf{\Psi}(t+1) \mathbf{C} U\|_F}\\
   & \approx A_l(t) \cdot \frac{u_0^T \mathbf{\Psi}(t)^T \mathbf{H}(t)^T \mathbf{M}_l \mathbf{H}(t) \mathbf{\Psi}(t) u_0}{u_0^T \mathbf{\Psi}(t)^T \mathbf{M}_l \mathbf{\Psi}(t) u_0}\\
   & = A_l(t) \cdot \frac{\theta(t)^T \mathbf{H}_y(t)^T \mathbf{H}^T \mathbf{M}_l \mathbf{H} \mathbf{H}_y(t)  \theta(t)}{\theta(t)^T \mathbf{H}_y(t)^T \mathbf{M}_l \mathbf{H}_y(t)  \theta(t)}\\
  & \approx A_l(t) \cdot \frac{\operatorname{tr}\left(\mathbf{H}_y(t)^T \mathbf{H}(t)^T \mathbf{M}_l \mathbf{H}(t) \mathbf{H}_y(t) \right)}{\operatorname{tr}\left(\mathbf{H}_y(t)^T \mathbf{M}_l \mathbf{H}_y(t) \right)}\\
  & \approx A_l(t) \cdot \frac{\operatorname{tr}\left(\mathbf{M}_l \left(\mathbf{I}-\eta  \mathbf{H}_{w}(t) \right) \mathbf{H}_w^2(t) \left(\mathbf{I}-\eta  \mathbf{H}_{w}(t) \right) \right)}{\operatorname{tr}\left(\mathbf{M}_l \mathbf{H}_w^2(t) \right)}\\
  & = A_l(t) - 2 \eta  \frac{\operatorname{tr}\left(\mathbf{M}_l  \mathbf{H}_w^3(t) \right) }{\operatorname{tr}\left(\mathbf{M}_l \mathbf{H}_w^2(t) \right)} + \eta^2  \frac{\operatorname{tr}\left(\mathbf{M}_l  \mathbf{H}_w^4(t) \right)}{\operatorname{tr}\left(\mathbf{M}_l \mathbf{H}_w^2(t) \right)}
    % & = A_l(0) \frac{\theta(0)^T \mathbf{H}_y^T \mathbf{M}_l \mathbf{H}_y \theta(0)}{\operatorname{tr}\left( \mathbf{M}_l \mathbf{\Psi}(0) \mathbf{C} \mathbf{\Psi}(0)^T \mathbf{M}_l \mathbf{\Psi}(0) \mathbf{C} \mathbf{\Psi}(0)^T \right)^{1/2} \|y\|^2}\\
    \end{aligned}
\end{equation}
The first approximation holds in the case of large $N$ and the second approximation is based on the assumption that $\theta(t)$ is independent from $\mathbf{H}_y(t)$ and $\mathbf{H}(t)$, and each entry is drawn from i.i.d normal distribution.The second approximation has its roots in Gradient Independence \cref{section:gradient_independence}. The third approximation uses \cref{approx1_feature_evolution}. This derivation illustrates that hierarchical structure of CKA likely arise out of bias in $\mathbf{H}_w$'s third and fourth moment. Actually, for common learning rates of $\approx 0.005$ used for deep networks, $\mathbf{H}_w$'s largest positive eigenvalue is usually around $5-15$, hence the third part of \cref{eq:hierarchy} is dominated by the first and second part which is around $0.01 - 0.07$. We later empirically illustrate interesting structural bias in $\mathbf{H}_w(t)$. Let $\mathbf{V}(t)$ diagonalises $\mathbf{H}_w(t)$:
\begin{align*}
    \frac{\operatorname{tr}\left(\mathbf{M}_l  \mathbf{H}_w^3(t) \right) }{\operatorname{tr}\left(\mathbf{M}_l \mathbf{H}_w^2(t) \right)} & = \frac{\operatorname{tr}\left(\mathbf{V}(t)^T \mathbf{M}_l  \mathbf{V}(t) \mathbf{V}(t)^T\mathbf{H}_w^3(t) \mathbf{V}(t) \right) }{\operatorname{tr}\left(\mathbf{V}(t)^T \mathbf{M}_l  \mathbf{V}(t) \mathbf{V}(t)^T\mathbf{H}_w^2(t) \mathbf{V}(t) \right)}\\
    & = \frac{\sum_{i} c_i(t) \lambda_i^3}{\sum_{i} c_i(t) \lambda_i^2}
\end{align*}

% The approximation comes from the empirical observation that for any time $t$ distinct eigenvectors $v_i(t), v_j(t)$ of $\mathbf{H}_w(t)$ satisfies $v_i(t)^T \mathbf{M}_l v_j(t) \approx 0$, and hence $\mathbf{V}(t)^T \mathbf{M}_l  \mathbf{V}(t)$ is approximately diagonal as in \textcolor{red}{ADDPLOTT}. 
where $c_i(t) := v_i(t)^T \mathbf{M}_l v_i(t)$ and $\lambda_i(t)$ is the eigenvalue corresponding to $v_i(t)$. The quantity is a weighted average of all eigenvalues of $\mathbf{H}_w(t)$.

\section{Further experimental results}\label{app:further_experiments}

\begin{figure}[H]
\centering
\begin{subfigure}{.48\textwidth}
  \centering
  \includegraphics[width=\textwidth]{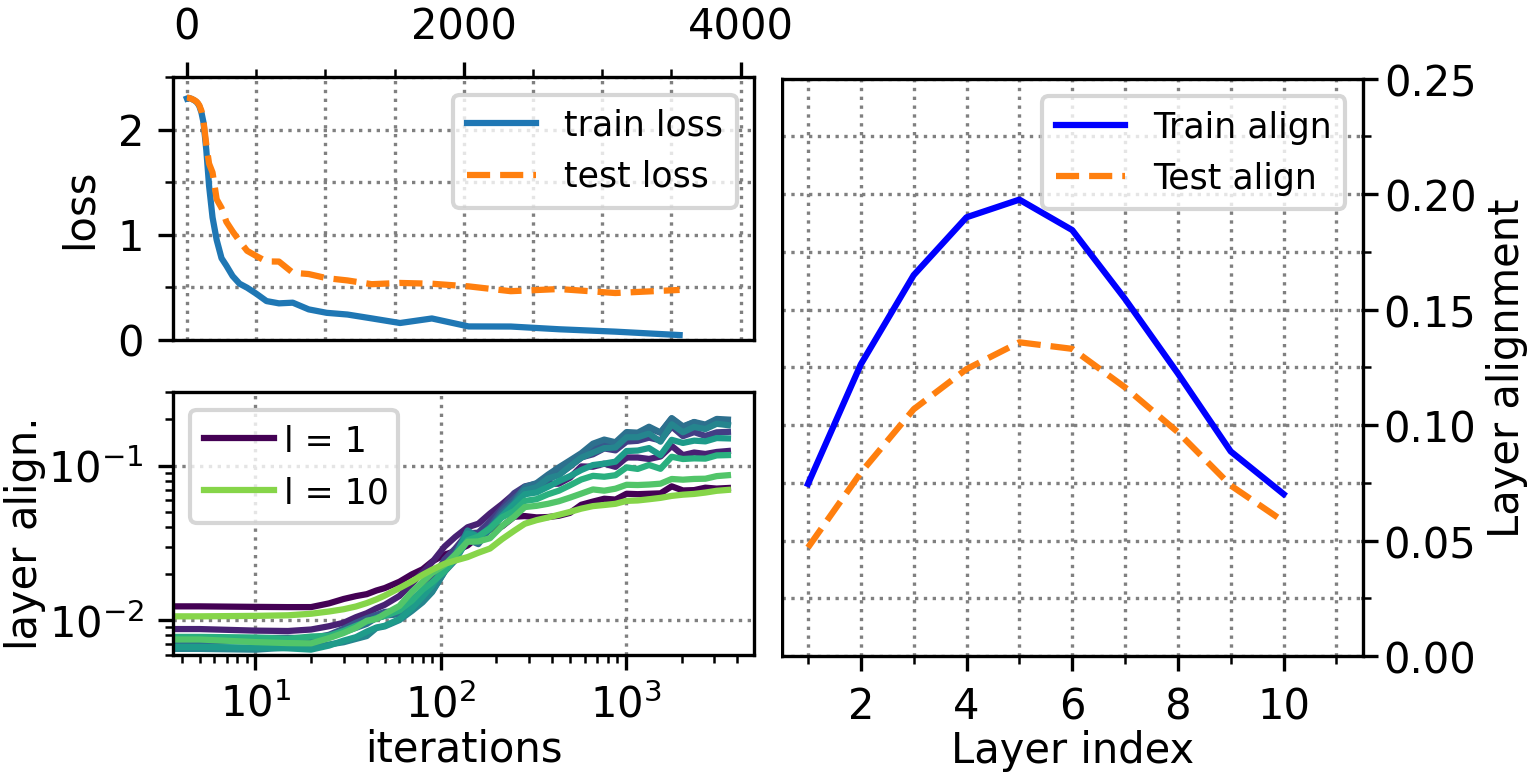}
  \caption{KMNIST}
  \label{fig:sfigkmnist}
\end{subfigure}
\begin{subfigure}{.48\textwidth}
  \centering
  \includegraphics[width=\textwidth]{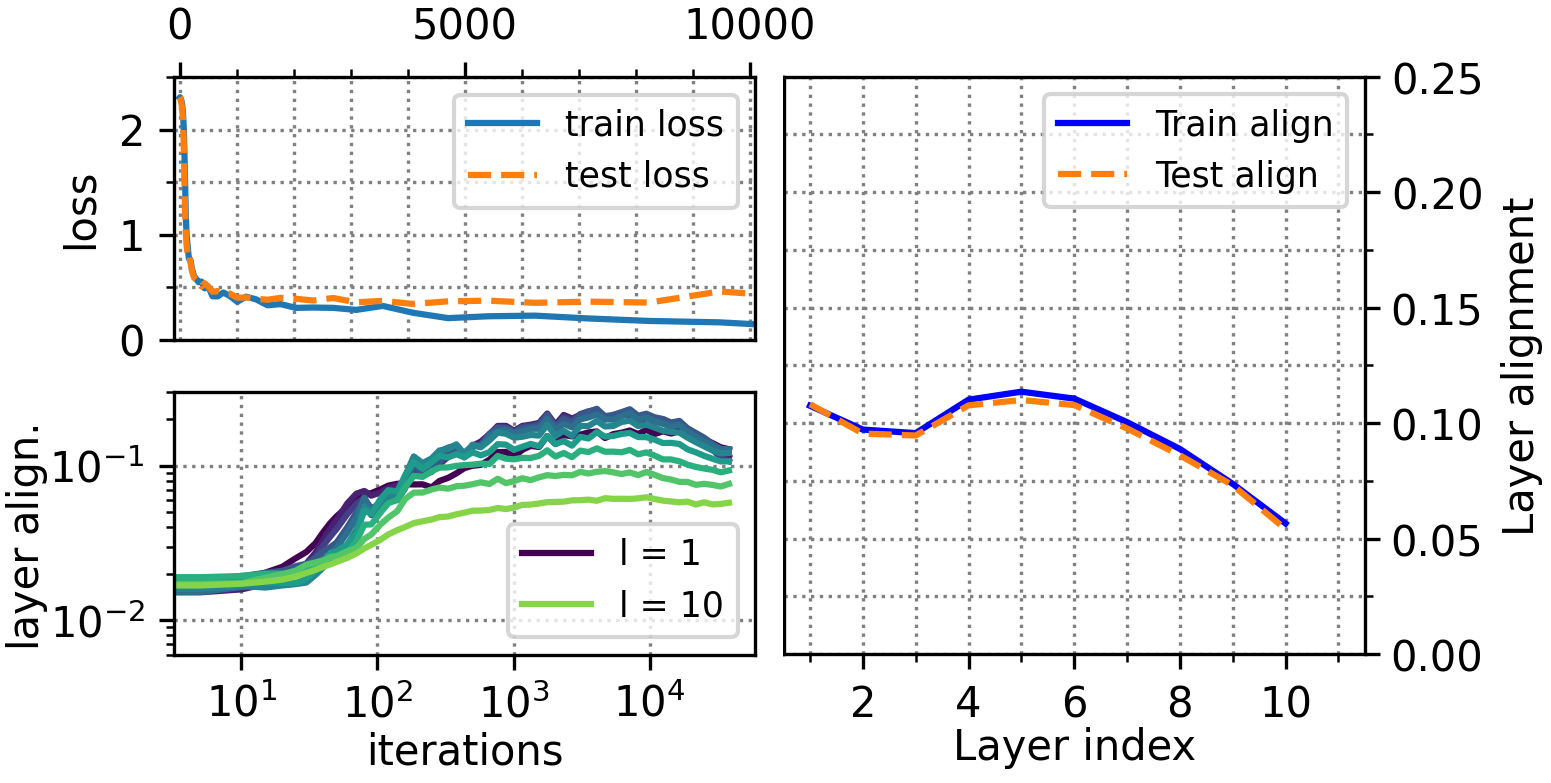}
  \caption{Fashion MNIST}
  \label{fig:sfigfmnist}
\end{subfigure}
\caption{Supplementary experiments for \cref{fig:align_example}.
Layerwise alignment hierarchy for the KMNIST and Fashion MNIST datasets when trained on a FFNN with depth 10 and width 256. Left hand panels show progression of loss and layer alignment with iterations of SGD. Right hand panel shows layer alignment at the end of training. Experiments above and in \cref{fig:align_example} used 10 layer FFNNs with 256 neurons in each layer, and were optimised with SGD with weight decay, momentum, and learning rates of 0.003.
}
\label{fig:align_example_supp}
\end{figure}

\begin{figure}[H]
\centering
  \begin{subfigure}{0.31\linewidth}
  \includegraphics[width=\linewidth]{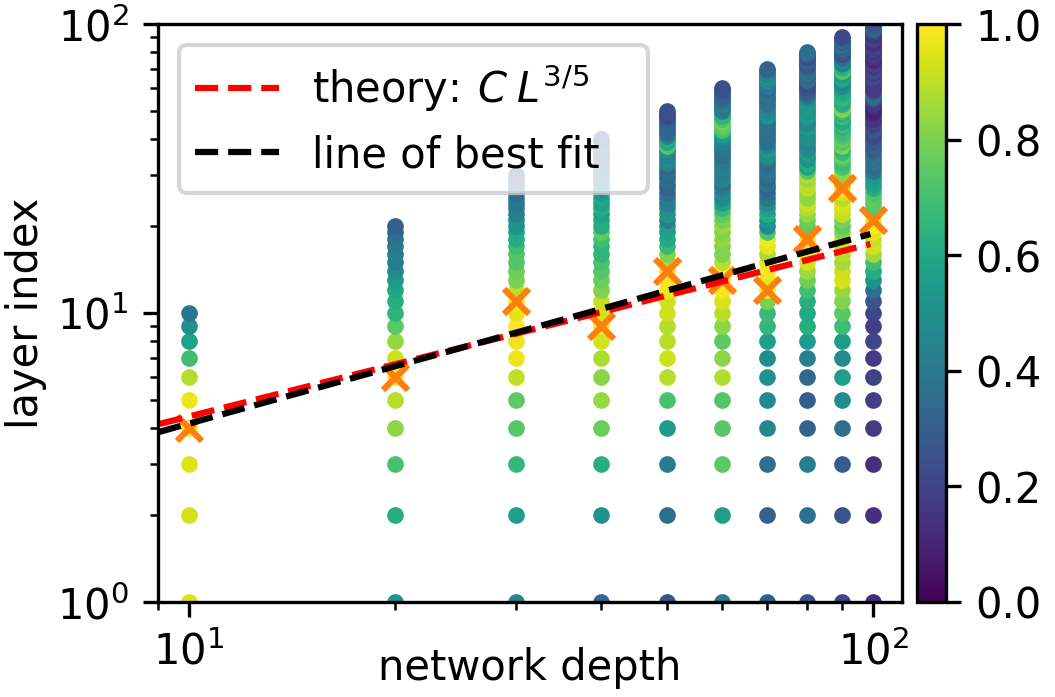}
  \caption{FFNN on Fashion MNIST}
\end{subfigure}
  \begin{subfigure}{0.31\linewidth}
  \includegraphics[width=\linewidth]{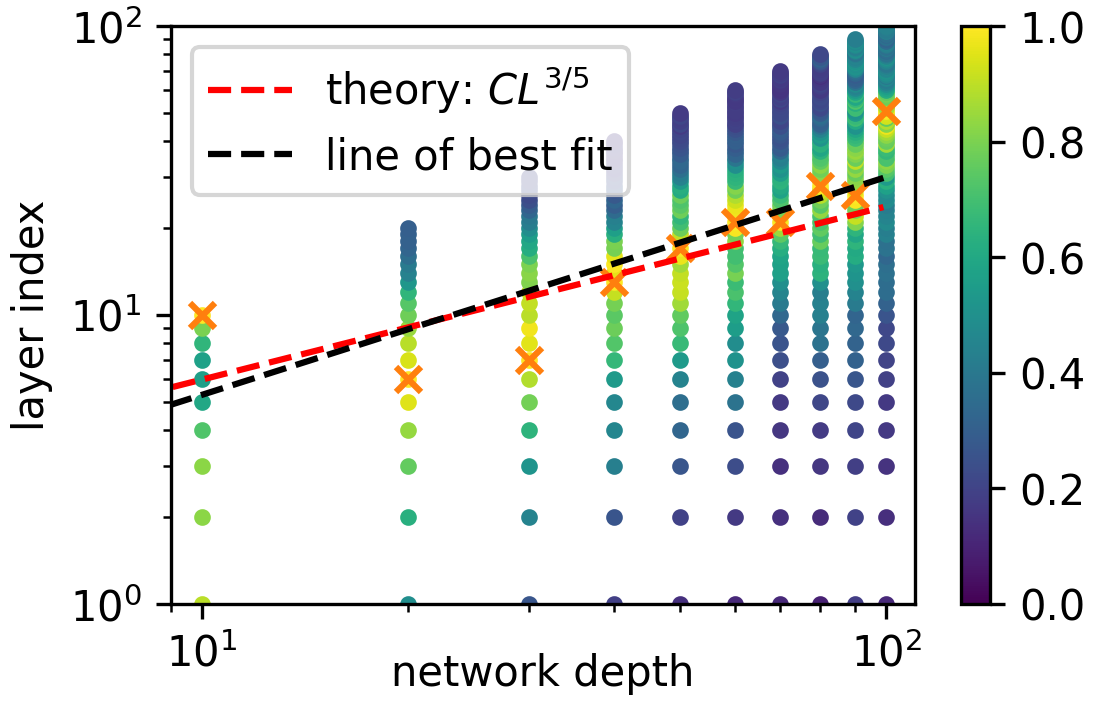}
  \caption{FFNN on Fashion MNIST}
\end{subfigure}
\begin{subfigure}{0.31\linewidth}
  \includegraphics[width=\linewidth]{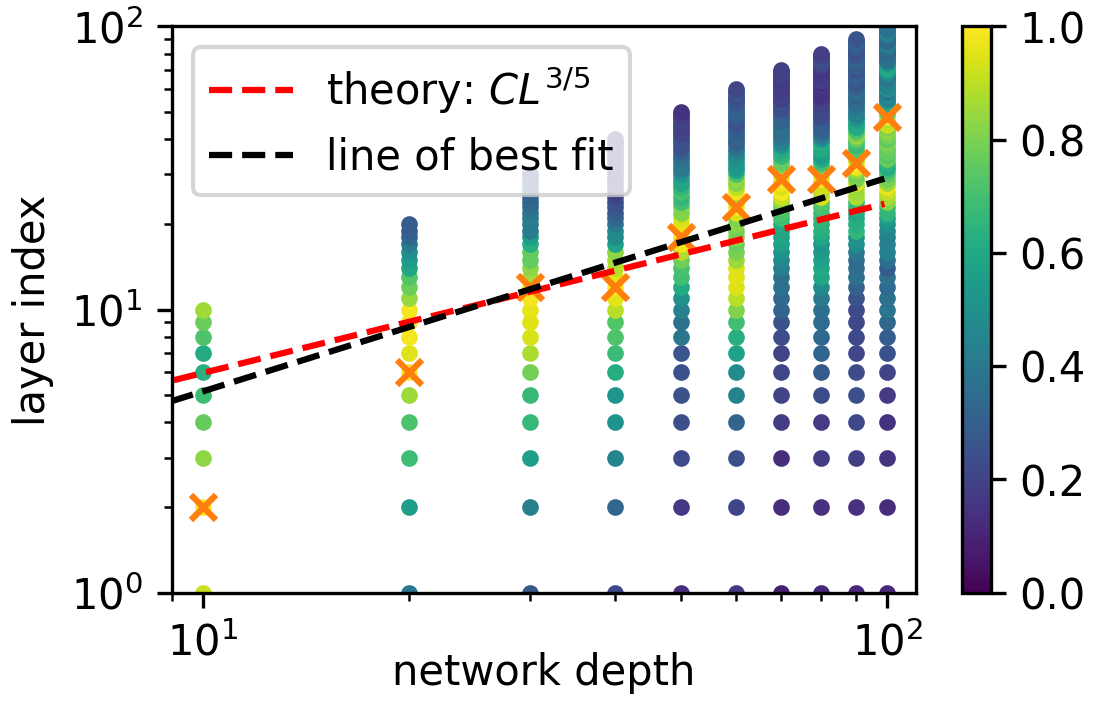}
  \caption{FFNN on Fashion MNIST}
\end{subfigure}
\caption{Here we compare (a) Fashion MNIST on a FFNN with batch size 128 and learning rate for the $10j$'th layer given by $[0.003, 0.004, 0.004, 0.002, 0.001, 0.0007, 0.0003, 0.0002, 0.0001, 0.00007]$. This variation of layer-wise learning rate produces the best generalisation at each depth. (b) Fashion MNIST on a FFNN with batch size 128 and learning rate 4x smaller per layer than in (a). (c) uses the same learning rates as (a) but a batch size of 512.
Note that (a) is the same experiment as \cref{fig:FMNIST_35} in the main text.
There is a clear upward shift in the y-intercept with decreasing learning rate (b) and increased batch size (c), as discussed in \cref{sec:generalization}, meaning that the peak of the AH shifts towards the last layer, and this correlates with poorer performance.
}\label{fig:C_change_AH}
\end{figure}

\begin{figure}[H]
\centering

\begin{subfigure}{\linewidth}
  \includegraphics[width=\linewidth]{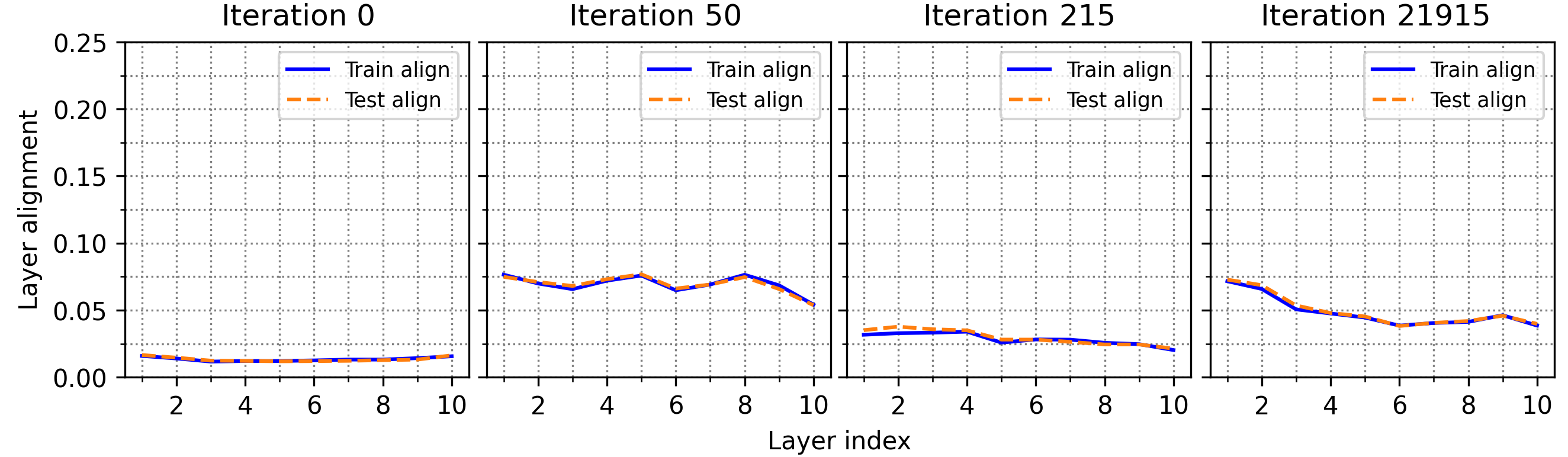}
  \caption{Learning Rate 0.0005, test accuracy 88.4\%, test loss 0.53\\$\quad$}
\end{subfigure}

\begin{subfigure}{\linewidth}
  \includegraphics[width=\linewidth]{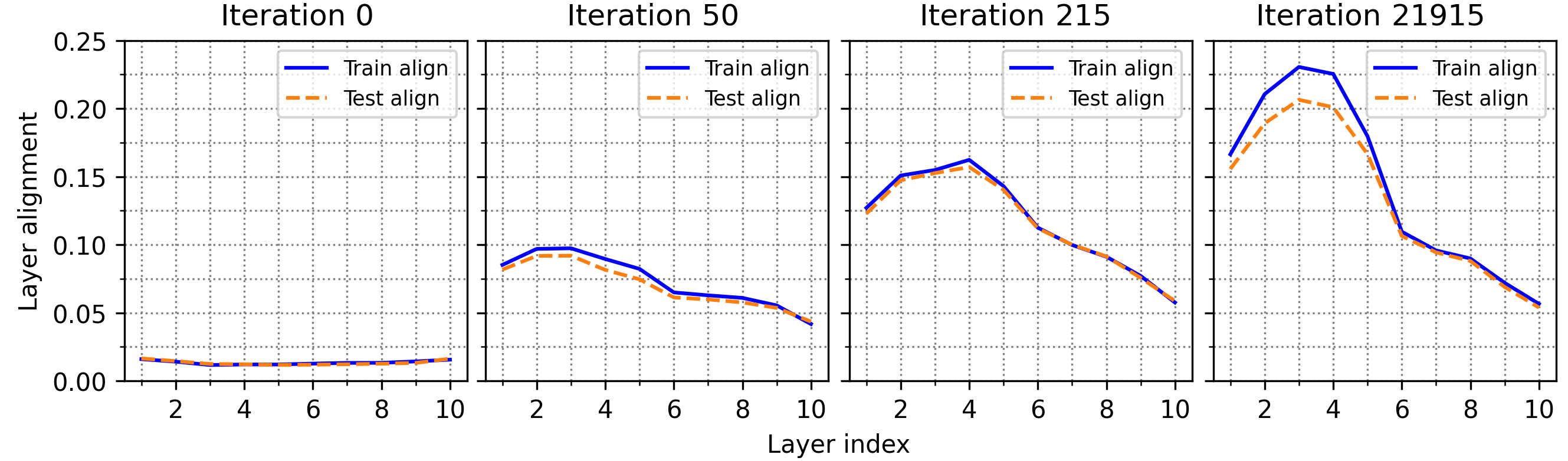}
  \caption{Learning Rate 0.003, test accuracy 88.3\%, test loss 0.53\\$\quad$}
\end{subfigure}

\begin{subfigure}{\linewidth}
  \includegraphics[width=\linewidth]{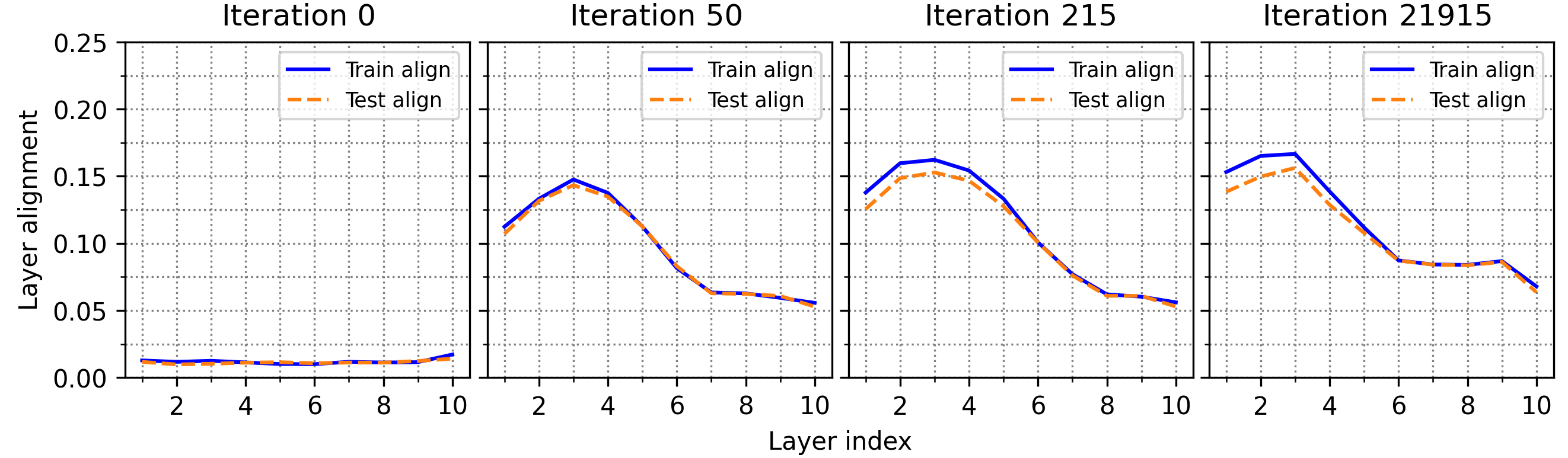}
  \caption{Learning Rate 0.05, test accuracy 87.1\%, test loss 0.41}
\end{subfigure}

\caption{Alignment progress during training. Fashion MNIST. Further detail for \cref{fig:sfigfmnist}.}
\label{fig:align_train_fmnist}
\end{figure}

\begin{figure}[H]
\centering

\begin{subfigure}{\linewidth}
  \includegraphics[width=\linewidth]{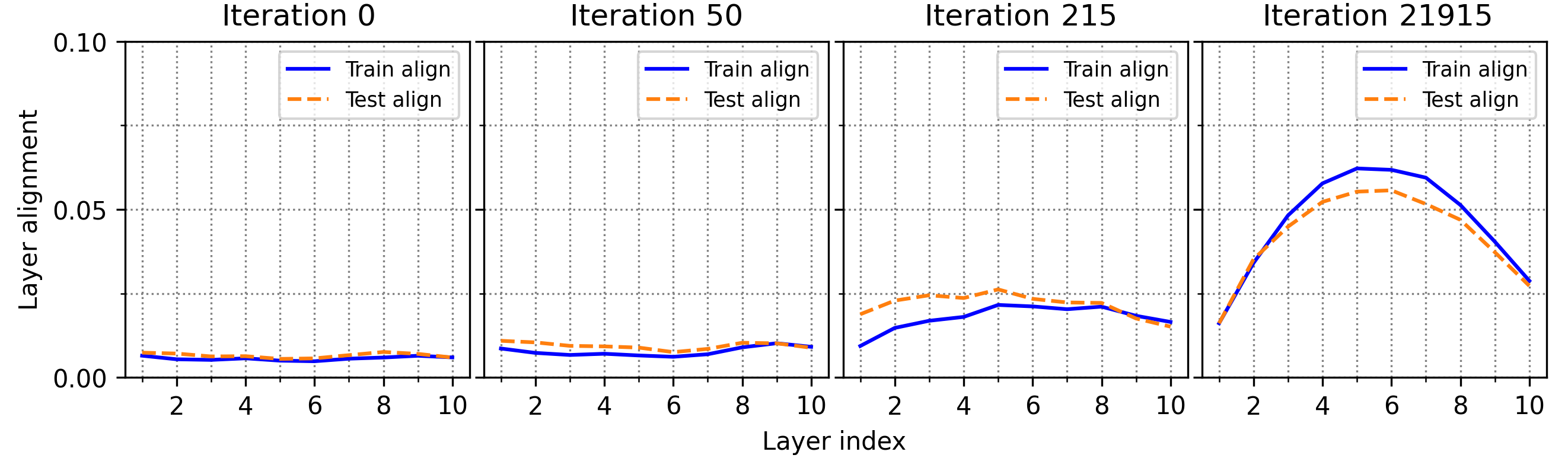}
  \caption{Learning Rate 0.0005, test accuracy 53.9\%, test loss 1.30\\$\quad$}
\end{subfigure}

\begin{subfigure}{\linewidth}
  \includegraphics[width=\linewidth]{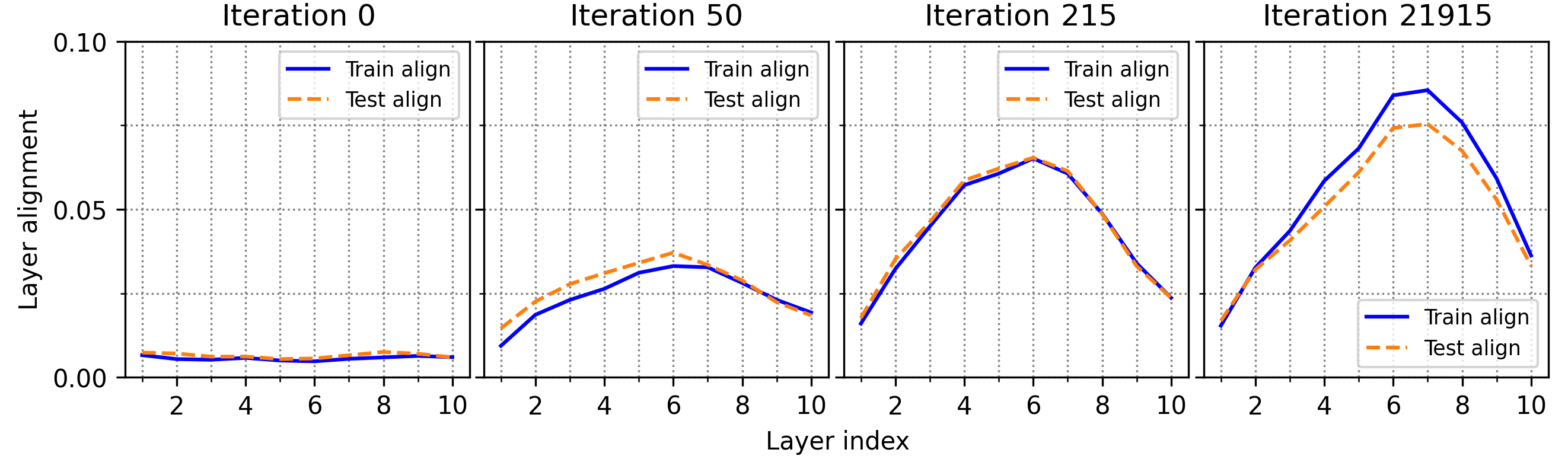}
  \caption{Learning Rate 0.003, test accuracy 56.7\%, test loss 1.21\\$\quad$}
\end{subfigure}

\begin{subfigure}{\linewidth}
  \includegraphics[width=\linewidth]{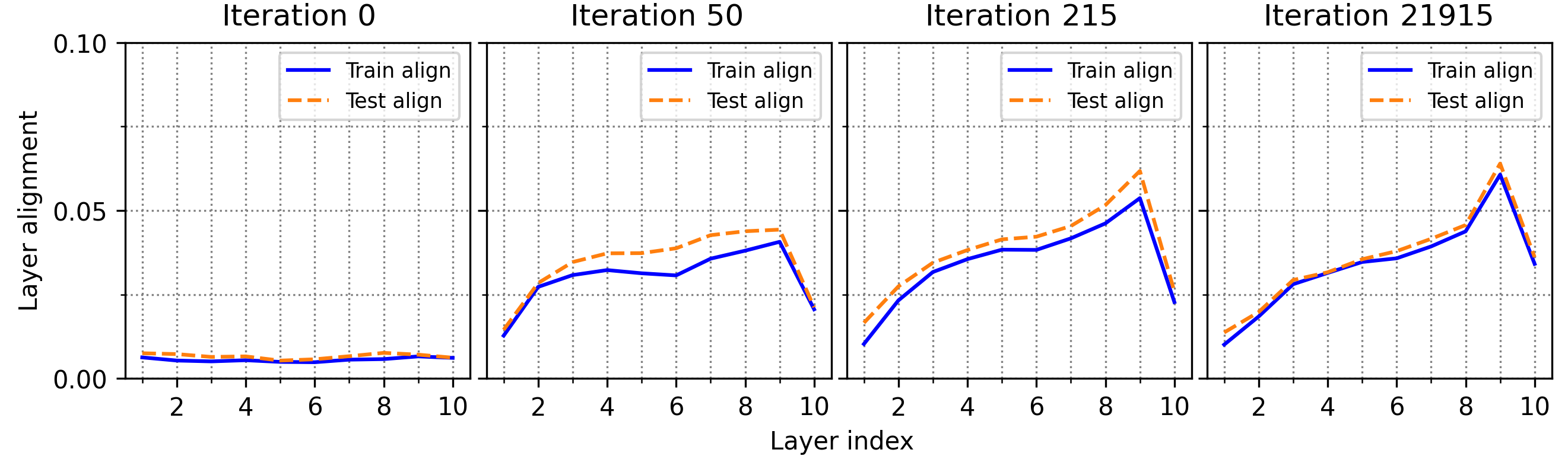}
  \caption{Learning Rate 0.05, test accuracy 51.4\%, test loss 1.41}
\end{subfigure}

\caption{Alignment progress during training. CIFAR10. Further experiments from \cref{fig:sfig1_2}.}\label{fig:align_train_cifar}
\end{figure}

\begin{figure}[H]
\centering
\begin{subfigure}{0.3\linewidth}
  \includegraphics[width=\linewidth]{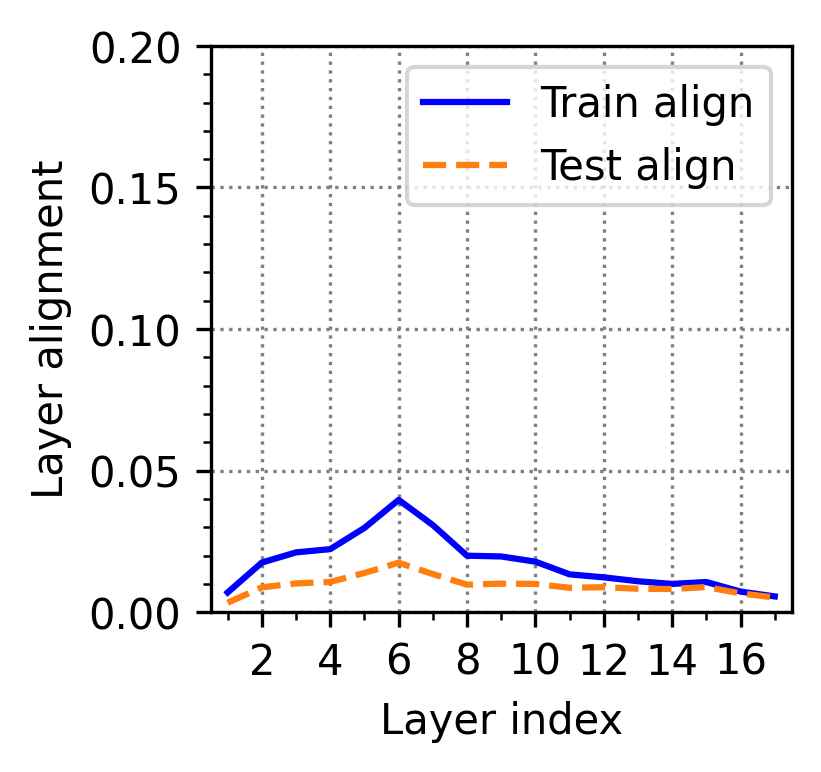}
  \caption{lr 0.001, Acc 46.8\%}
\end{subfigure}
  \begin{subfigure}{0.3\linewidth}
  \includegraphics[width=\linewidth]{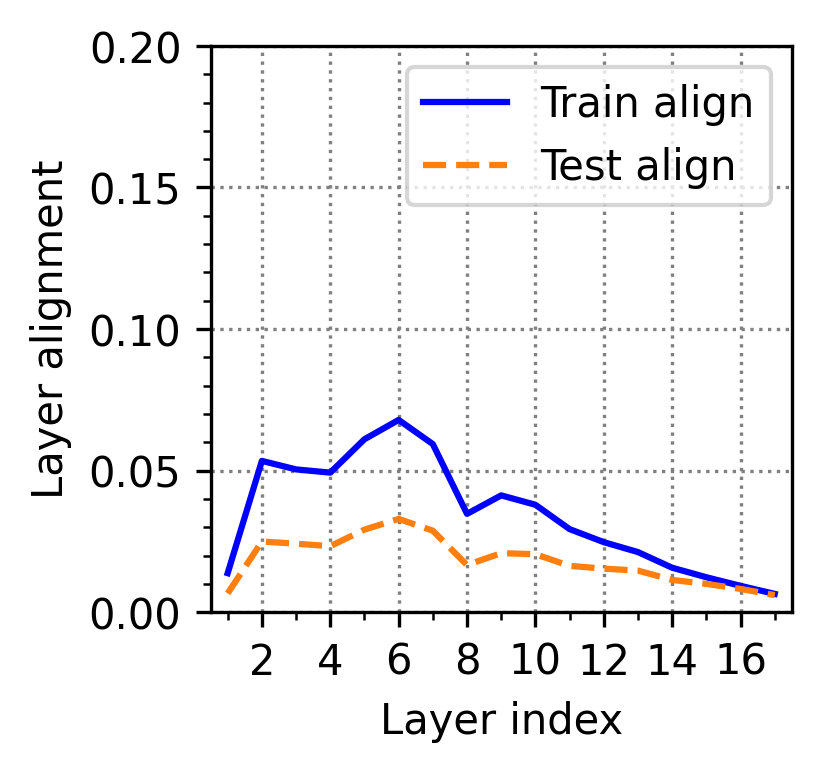}
  \caption{lr 0.005, Acc 53.9\%}
\end{subfigure}
  \begin{subfigure}{0.3\linewidth}
  \includegraphics[width=\linewidth]{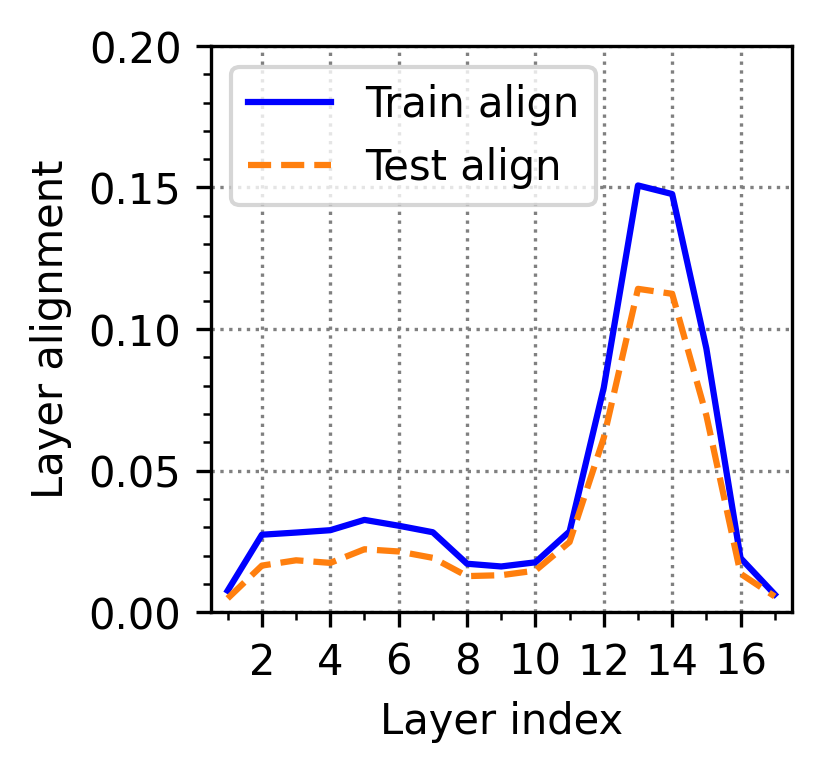}
  \caption{lr 0.01, Acc 62.5\%}
\end{subfigure}
\caption{(a), (b) and (c) show VGG19 trained with SGD (with momentum and weight decay) for 100 epochs on CIFAR100 dataset with three different learning rates (lr). This is an addition to \cref{fig:egs_main} with more complex architectures. To properly compare the three learning rates, training should be stopped at fixed training loss, as 100 epochs may not allow for convergence with the smaller learning rates (in (a) and (b)).}\label{fig:cifar100_hierarchy_generalisation}
\end{figure}

% \begin{figure}[H]
% \centering

% \begin{subfigure}{\linewidth}
%   \includegraphics[width=\linewidth]{fig/align_vs_epochs/cifar00005.png}
%   \caption{Learning Rate 0.0005, test accuracy 53.9\%, test loss 1.30\\$\quad$}
% \end{subfigure}

% \begin{subfigure}{\linewidth}
%   \includegraphics[width=\linewidth]{fig/align_vs_epochs/cifar0005.png}
%   \caption{Learning Rate 0.003, test accuracy 56.7\%, test loss 1.21\\$\quad$}
% \end{subfigure}

% \begin{subfigure}{\linewidth}
%   \includegraphics[width=\linewidth]{fig/align_vs_epochs/cifar005.png}
%   \caption{Learning Rate 0.05, test accuracy 51.4\%, test loss 1.41}
% \end{subfigure}

% \caption{Alignment progress during training. CIFAR10. Further experiments from \cref{fig:sfig1_2}.}\label{fig:align_train_cifar}
% \end{figure}

\begin{figure}[H]
\centering
\begin{subfigure}{0.23\linewidth}
  \includegraphics[width=\linewidth]{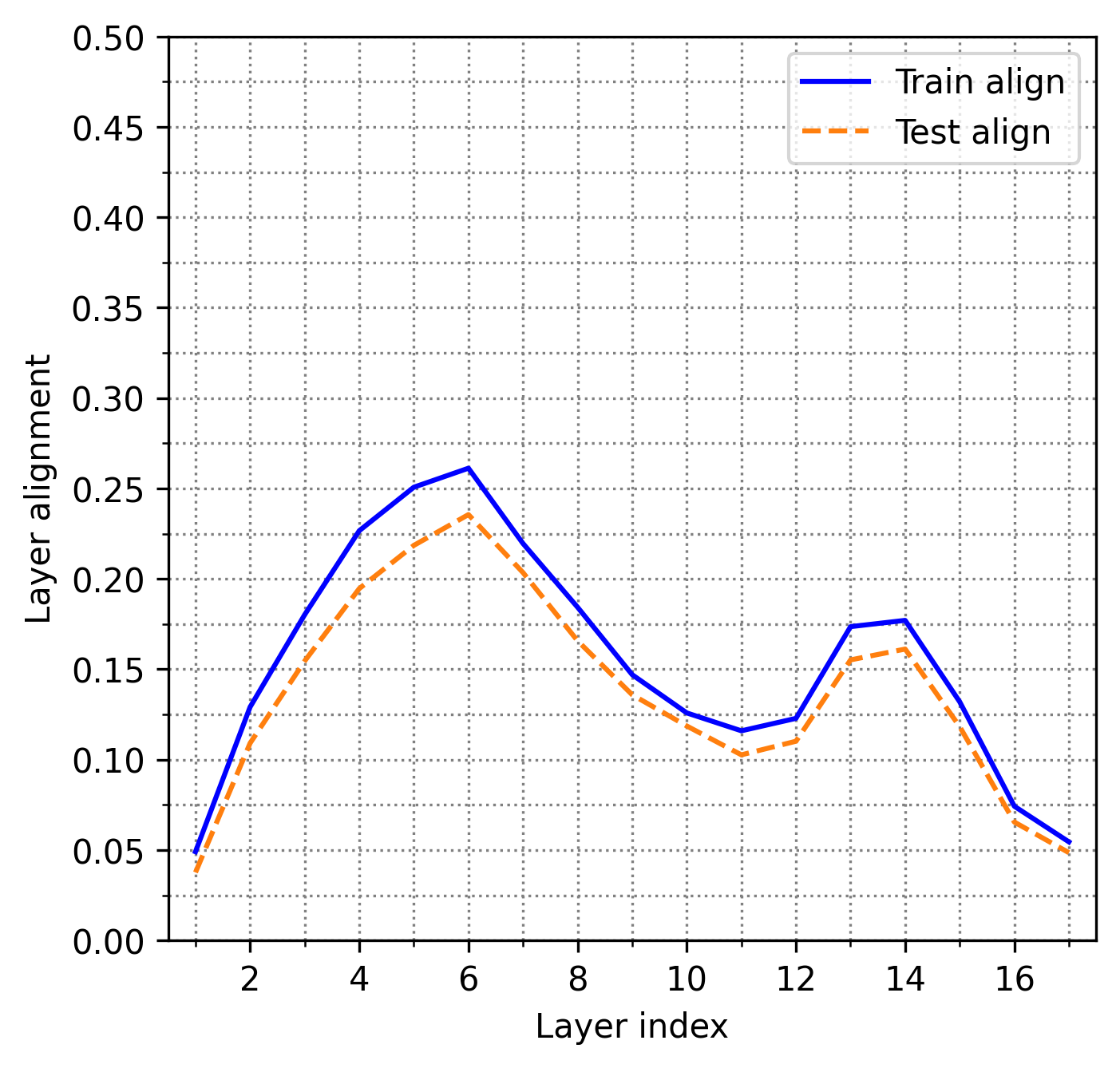}
  \caption{lr 0.0005, Acc Gap  13.1\%, Loss Gap 0.560}
\end{subfigure}
\begin{subfigure}{0.23\linewidth}
  \includegraphics[width=\linewidth]{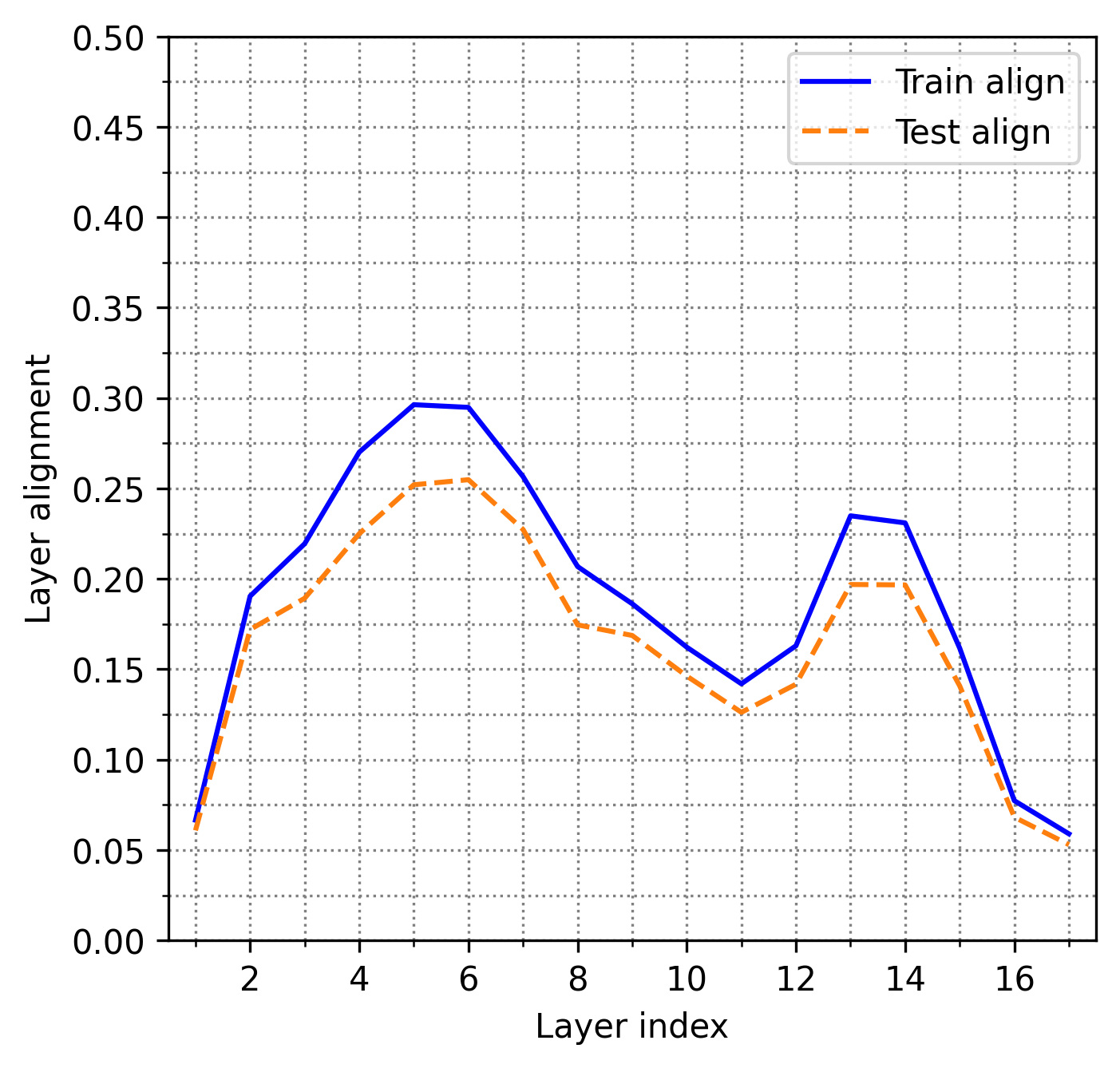}
  \caption{lr 0.001, Acc Gap  12.3\%, Loss Gap 0.484}
\end{subfigure}
  \begin{subfigure}{0.23\linewidth}
  \includegraphics[width=\linewidth]{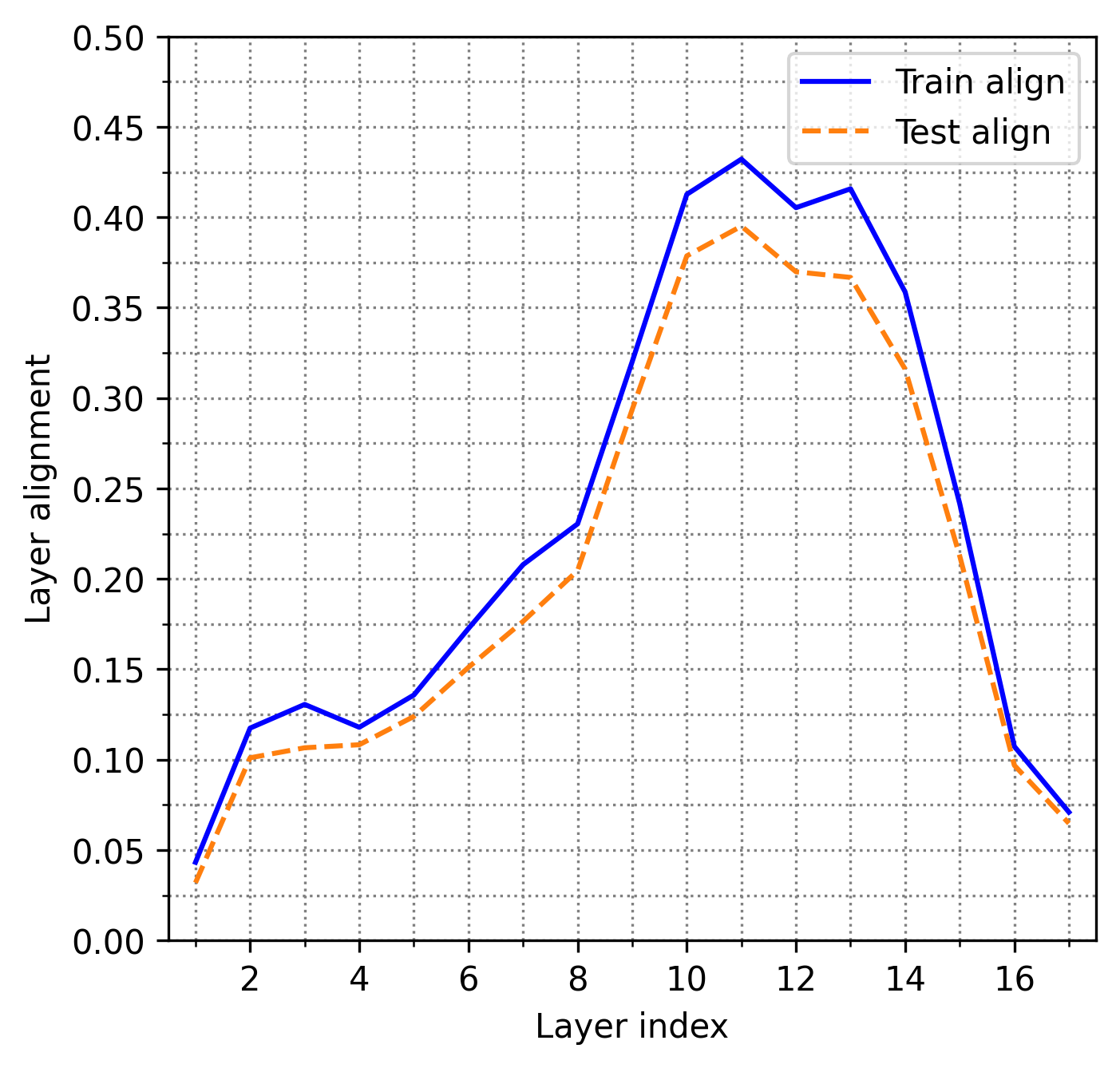}
  \caption{lr 0.0032, Acc Gap 8.89\%, Loss Gap 0.357}
\end{subfigure}
  \begin{subfigure}{0.23\linewidth}
  \includegraphics[width=\linewidth]{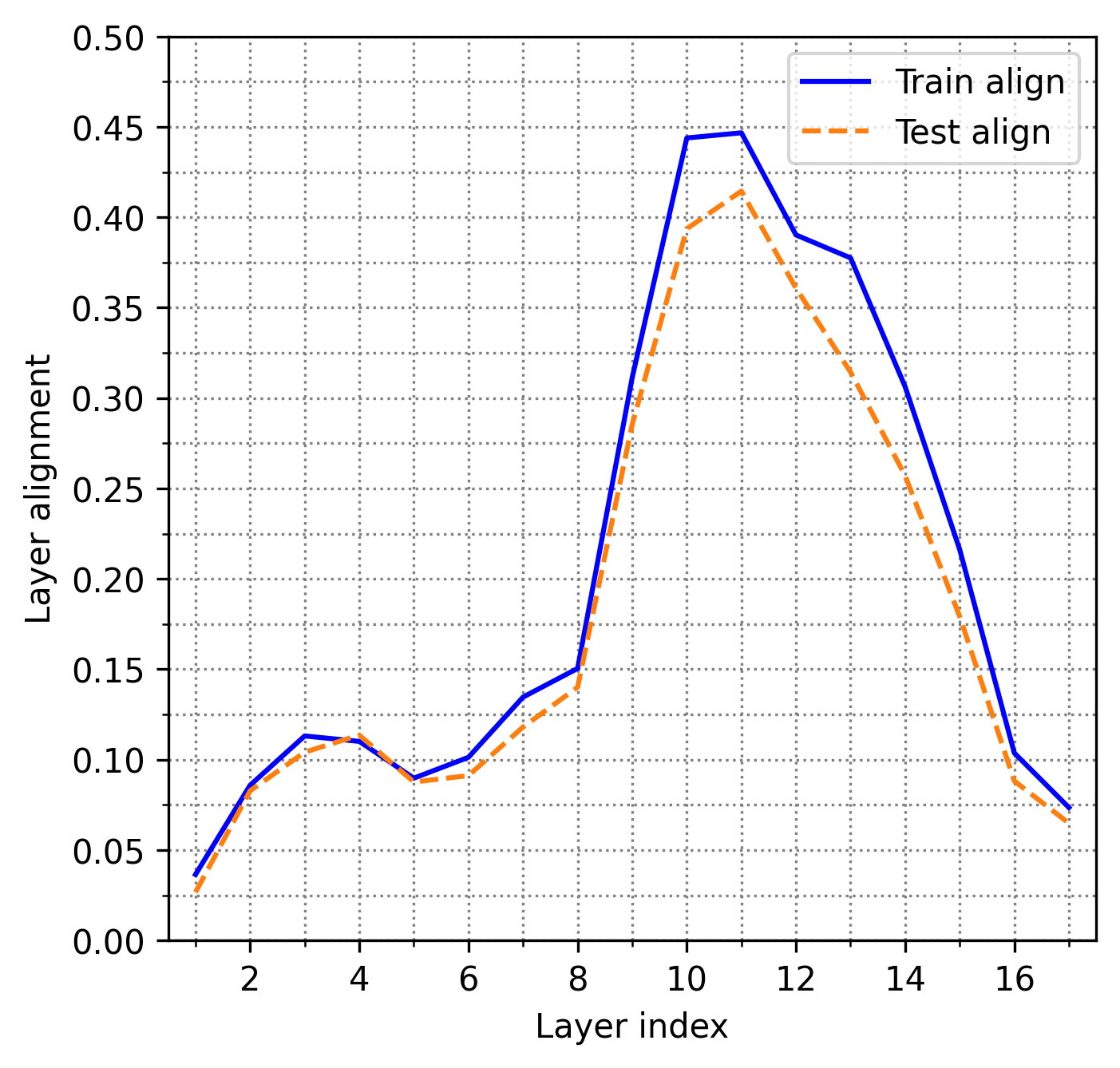}
  \caption{lr 0.01, Acc Gap 5.9\%, Loss Gap 0.224}
\end{subfigure}
\caption{(a), (b) and (c) show VGG19 trained with SGD (with momentum and weight decay) until convergence (train loss reaches 0.1) epochs on CIFAR10 dataset with three different learning rates (lr). This is an addition to \cref{fig:egs_main}.}\label{fig:cifar10_vgg19_hierarchy_generalisation}
\end{figure}

\begin{figure}[H]
\centering
\begin{subfigure}{0.3\linewidth}
  \includegraphics[width=\linewidth]{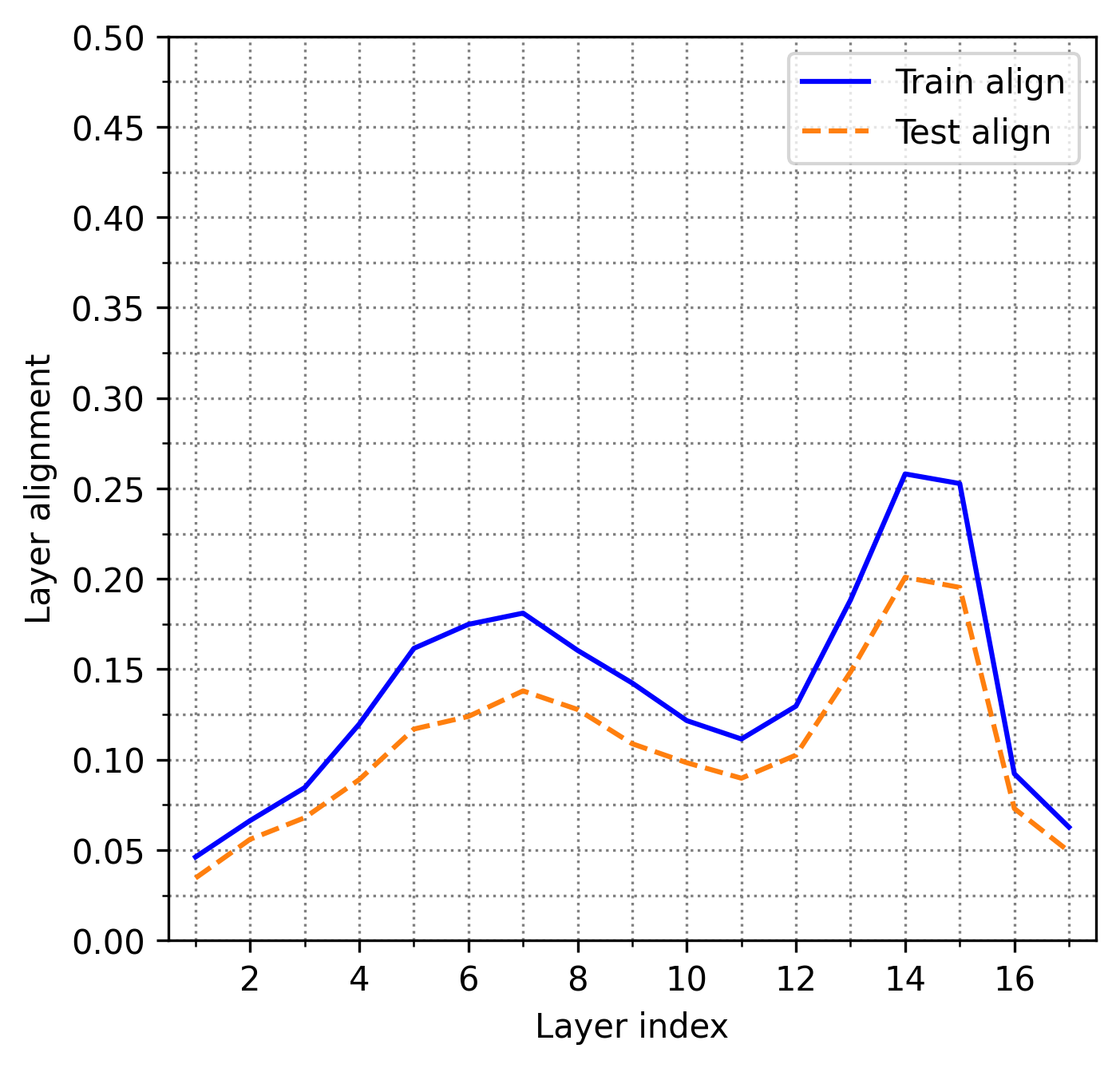}
  \caption{lr 0.000032, Acc Gap  17.4\%, Loss Gap 0.769}
\end{subfigure}
\begin{subfigure}{0.3\linewidth}
  \includegraphics[width=\linewidth]{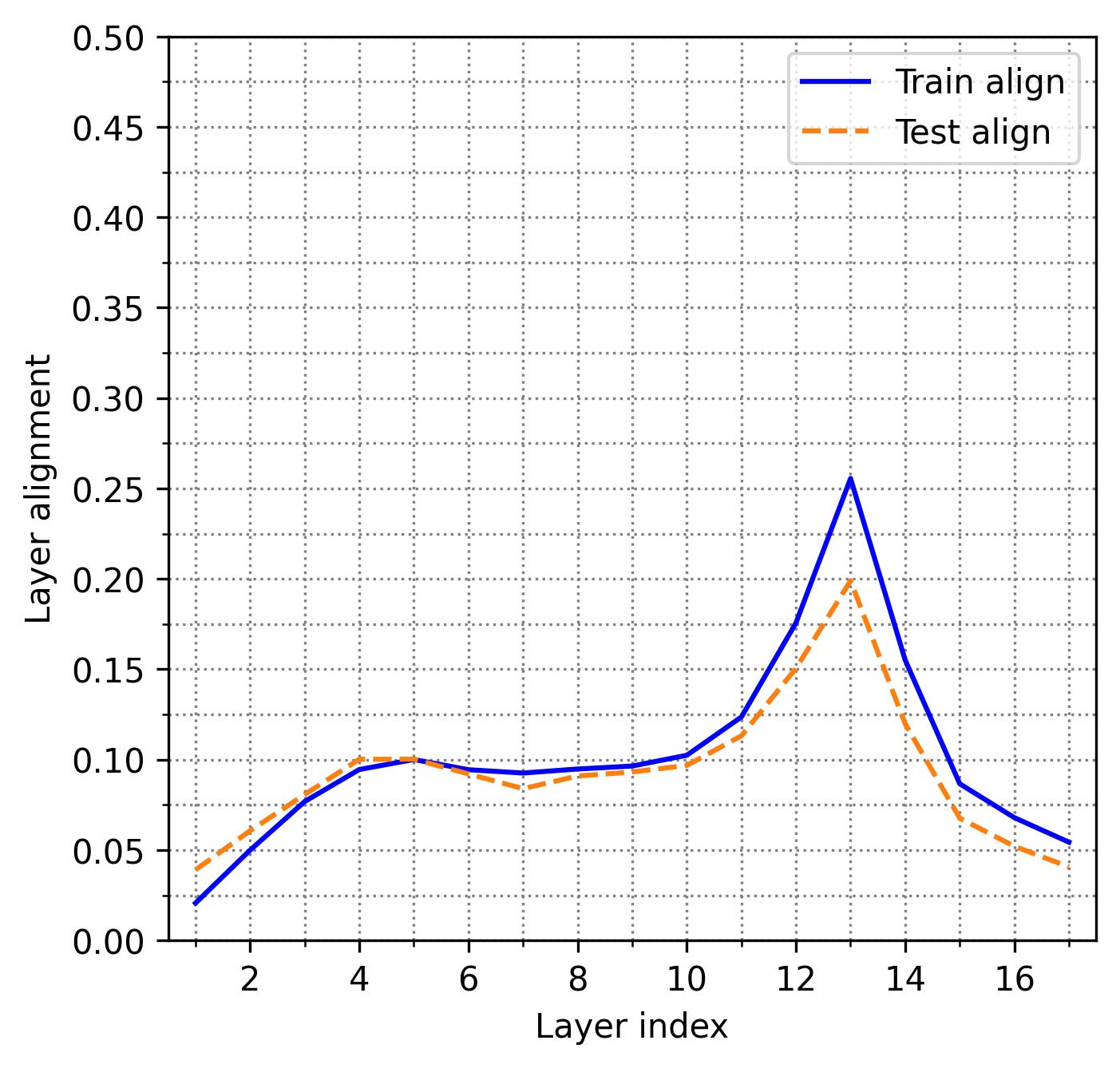}
  \caption{lr 0.0001, Acc Gap  12.4\%, Loss Gap 0.557}
\end{subfigure}
  \begin{subfigure}{0.3\linewidth}
  \includegraphics[width=\linewidth]{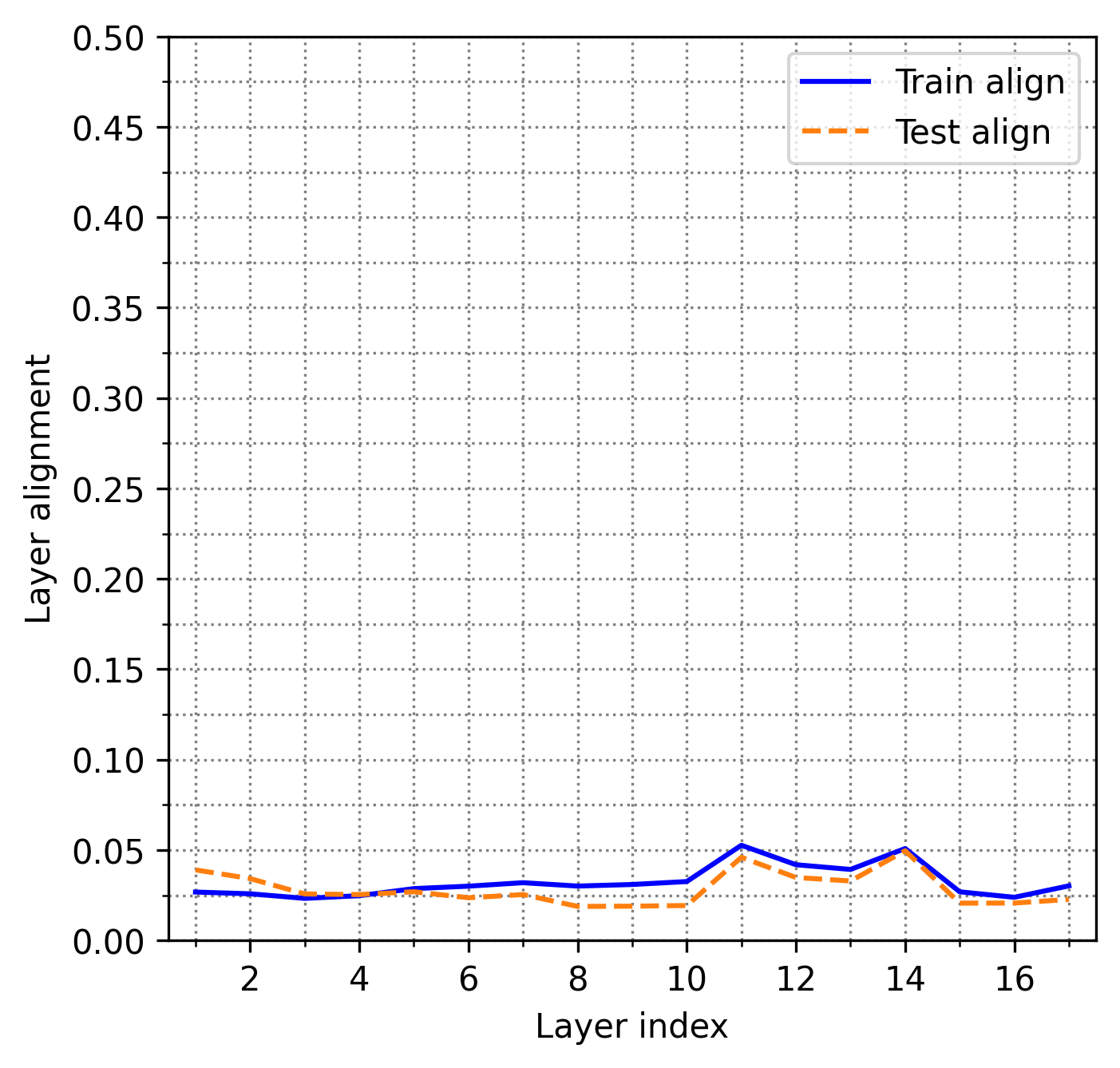}
  \caption{lr 0.00032, Acc Gap 11.9\%, Loss Gap 0.507}
\end{subfigure}
\caption{(a), (b) and (c) show VGG19 trained with ADAM until convergence (train loss reaches 0.1) epochs on CIFAR10 dataset with three different learning rates (lr). Note that ADAM exhibits worse gen. error and less salient alignment patterns.}\label{fig:cifar10_vgg19_hierarchy_generalisation_adam}
\end{figure}

\begin{figure}[H]
\centering
\begin{subfigure}{0.23\linewidth}
  \includegraphics[width=\linewidth]{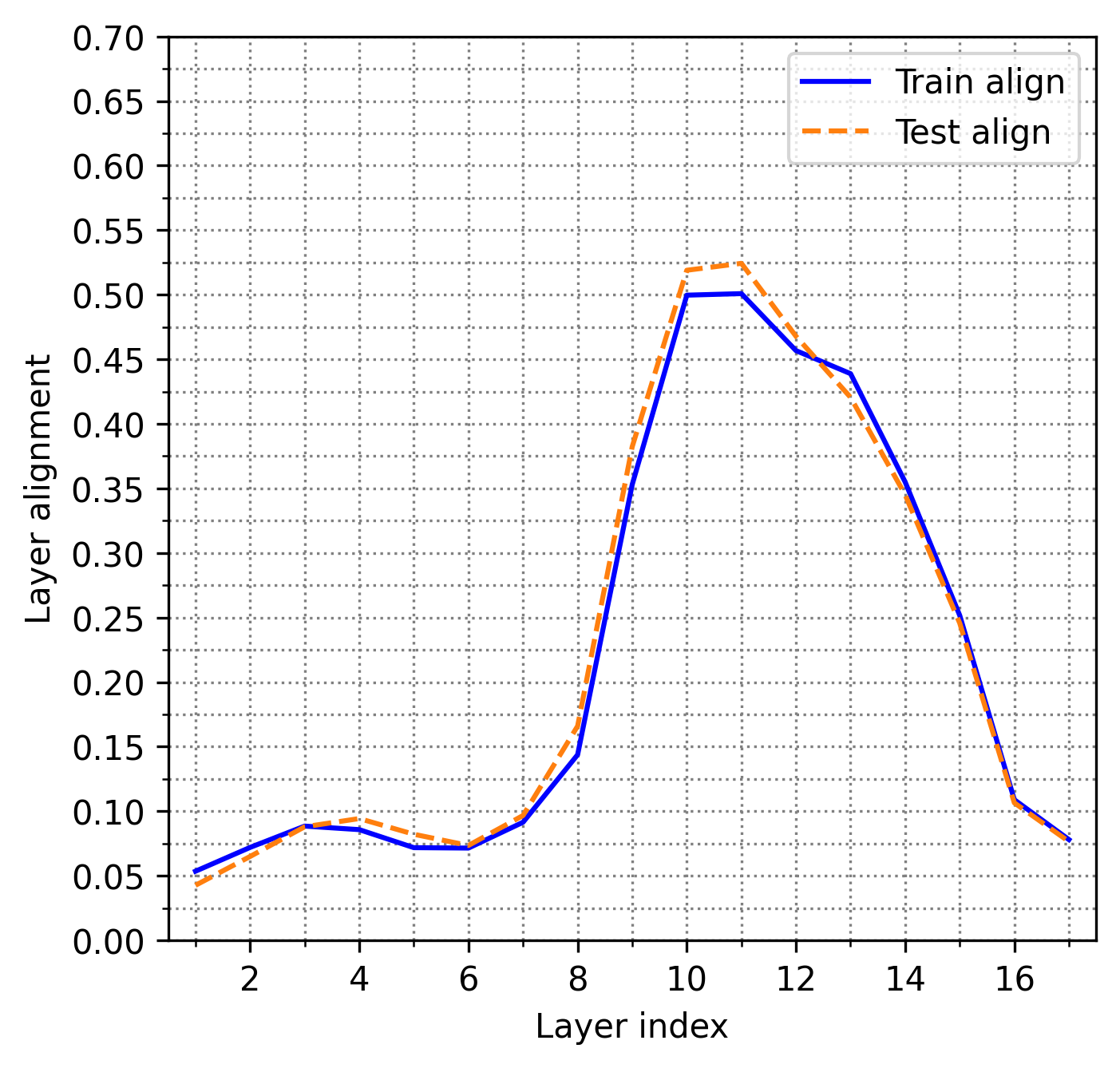}
  \caption{bs 32, Acc Gap  4.7\%, Loss Gap 0.236}
\end{subfigure}
\begin{subfigure}{0.23\linewidth}
  \includegraphics[width=\linewidth]{fig/fig10/cifar10_0-0032.png}
  \caption{bs 128, Acc Gap 8.89\%, Loss Gap 0.357}
\end{subfigure}
  \begin{subfigure}{0.23\linewidth}
  \includegraphics[width=\linewidth]{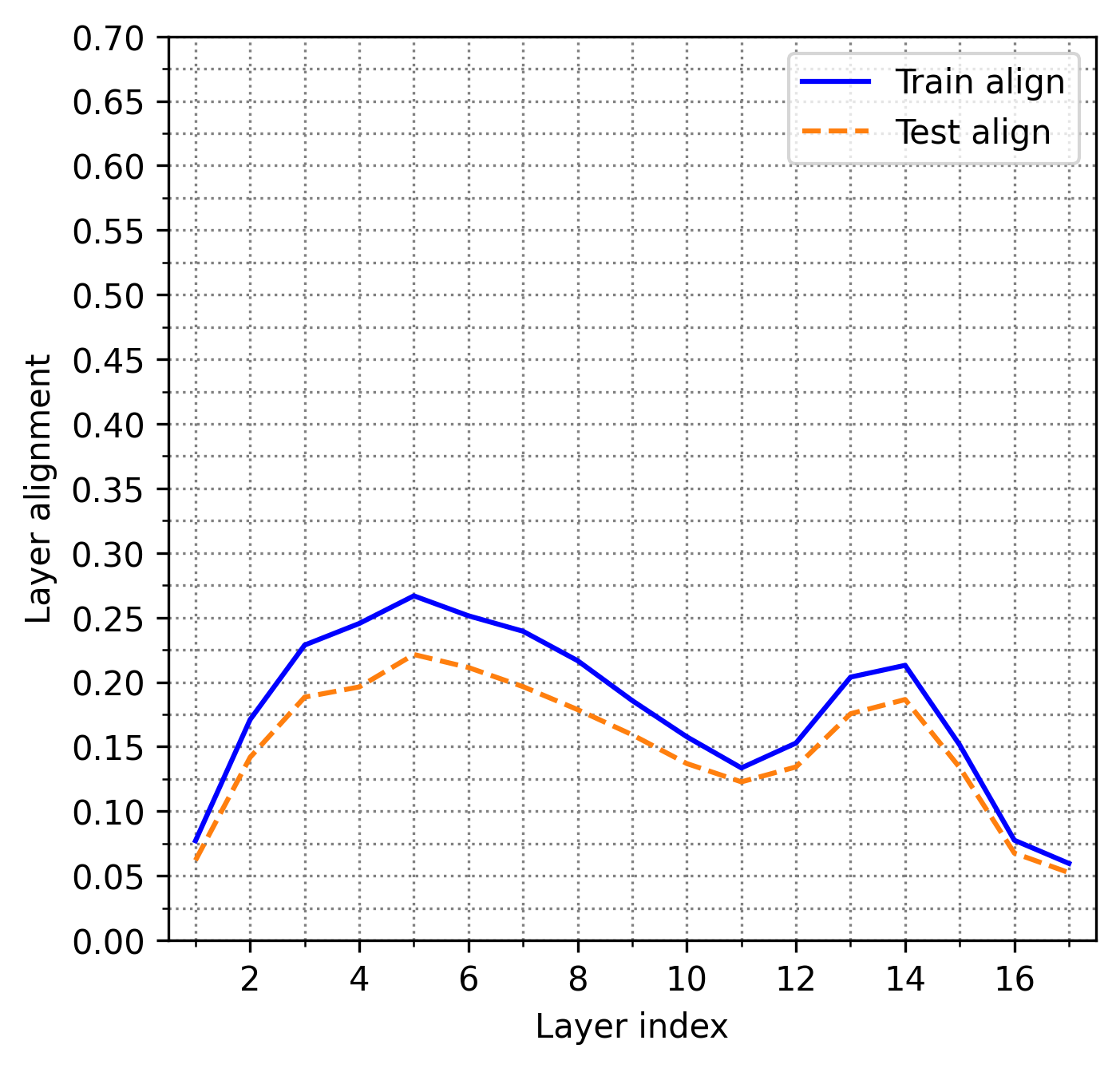}
  \caption{bs 512, Acc Gap 11.3\%, Loss Gap 0.562}
\end{subfigure}
  \begin{subfigure}{0.23\linewidth}
  \includegraphics[width=\linewidth]{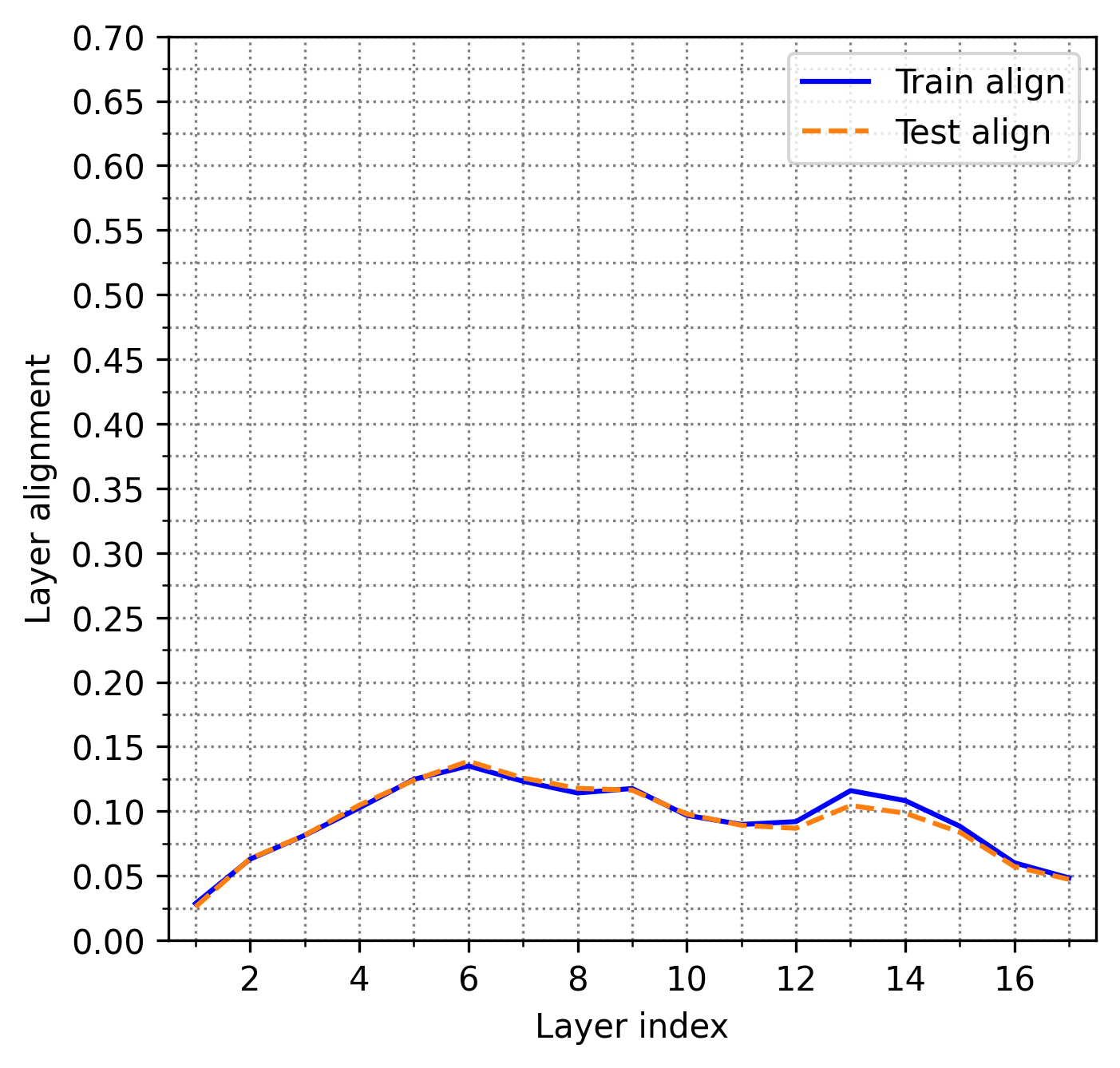}
  \caption{bs 2048, Acc Gap 12.2\%, Loss Gap 0.652}
\end{subfigure}
\caption{(a), (b) and (c) show VGG19 trained with SGD (with momentum and weight decay) until convergence (train loss reaches 0.1) epochs on CIFAR10 dataset with four different batch sizes (bs).}\label{fig:cifar10_vgg19_hierarchy_generalisation_bs}
\end{figure}

\begin{figure}[H]
\centering
\begin{subfigure}{0.23\linewidth}
  \includegraphics[width=\linewidth]{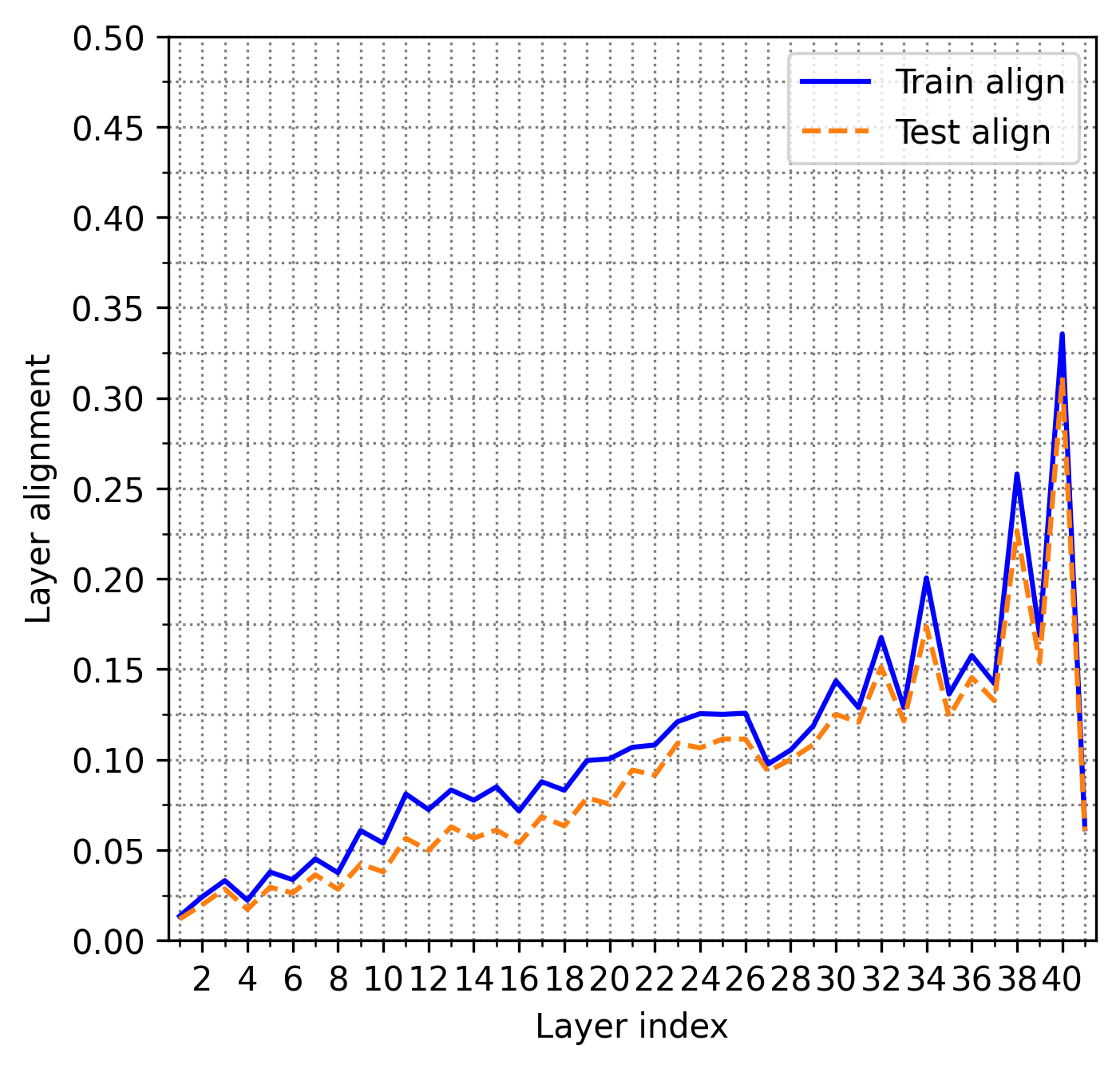}
  \caption{lr 0.001, Acc Gap  10.79\%, Loss Gap 0.395}
\end{subfigure}
\begin{subfigure}{0.23\linewidth}
  \includegraphics[width=\linewidth]{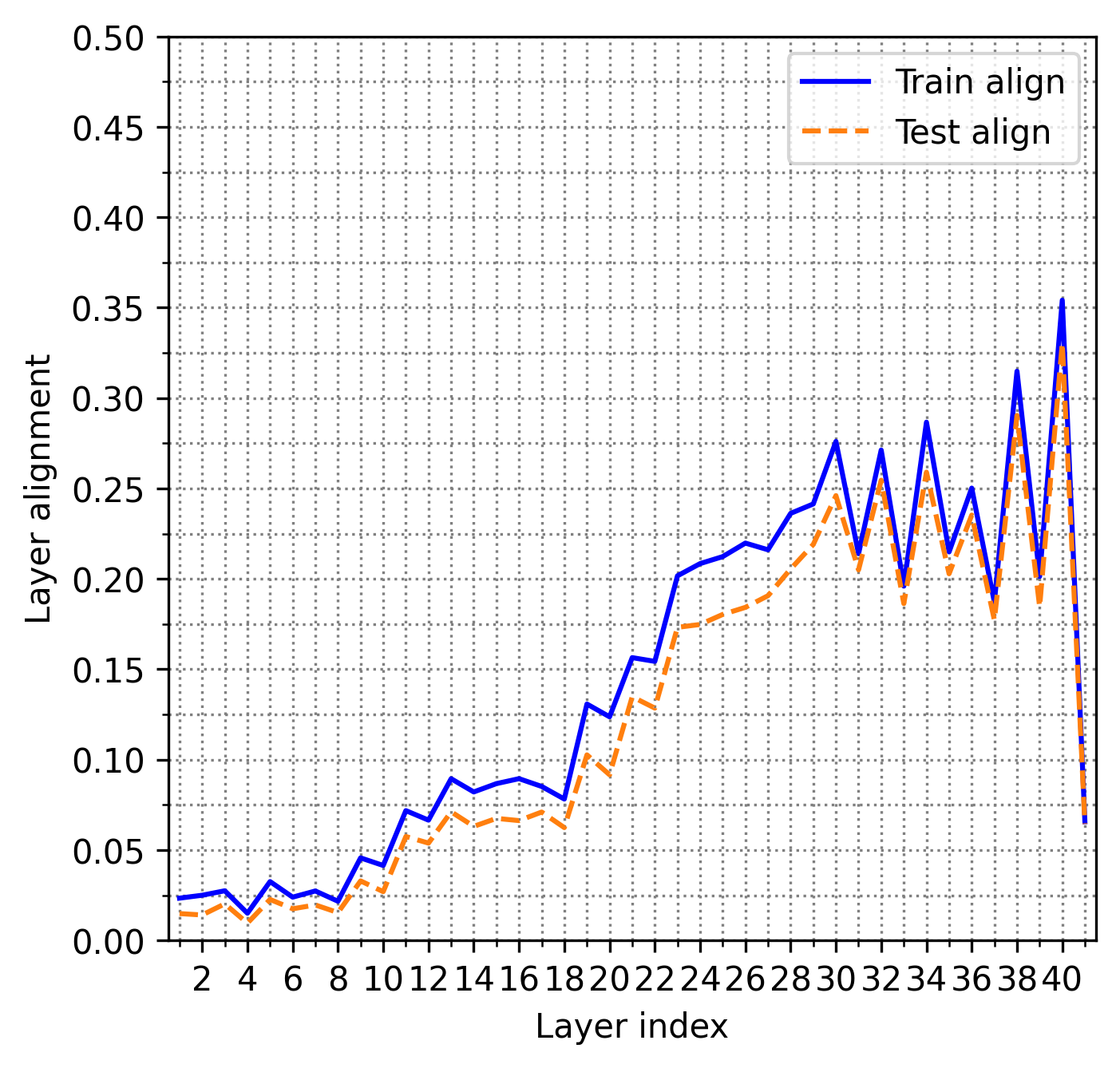}
  \caption{lr 0.003, Acc Gap 9.28\%, Loss Gap 0.397}
\end{subfigure}
  \begin{subfigure}{0.23\linewidth}
  \includegraphics[width=\linewidth]{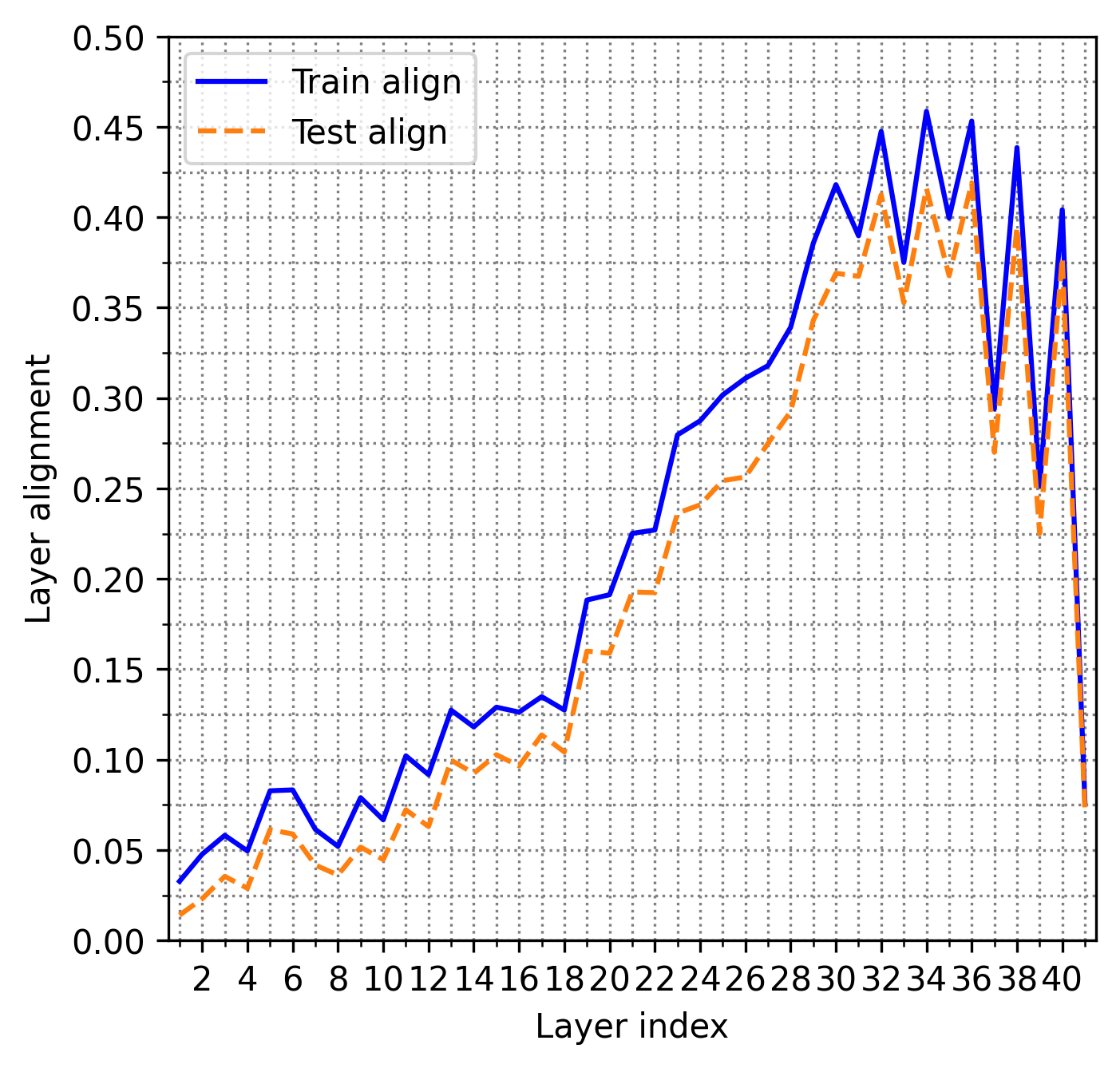}
  \caption{lr 0.01, Acc Gap 4.39\%, Loss Gap 0.152}
\end{subfigure}
  \begin{subfigure}{0.23\linewidth}
  \includegraphics[width=\linewidth]{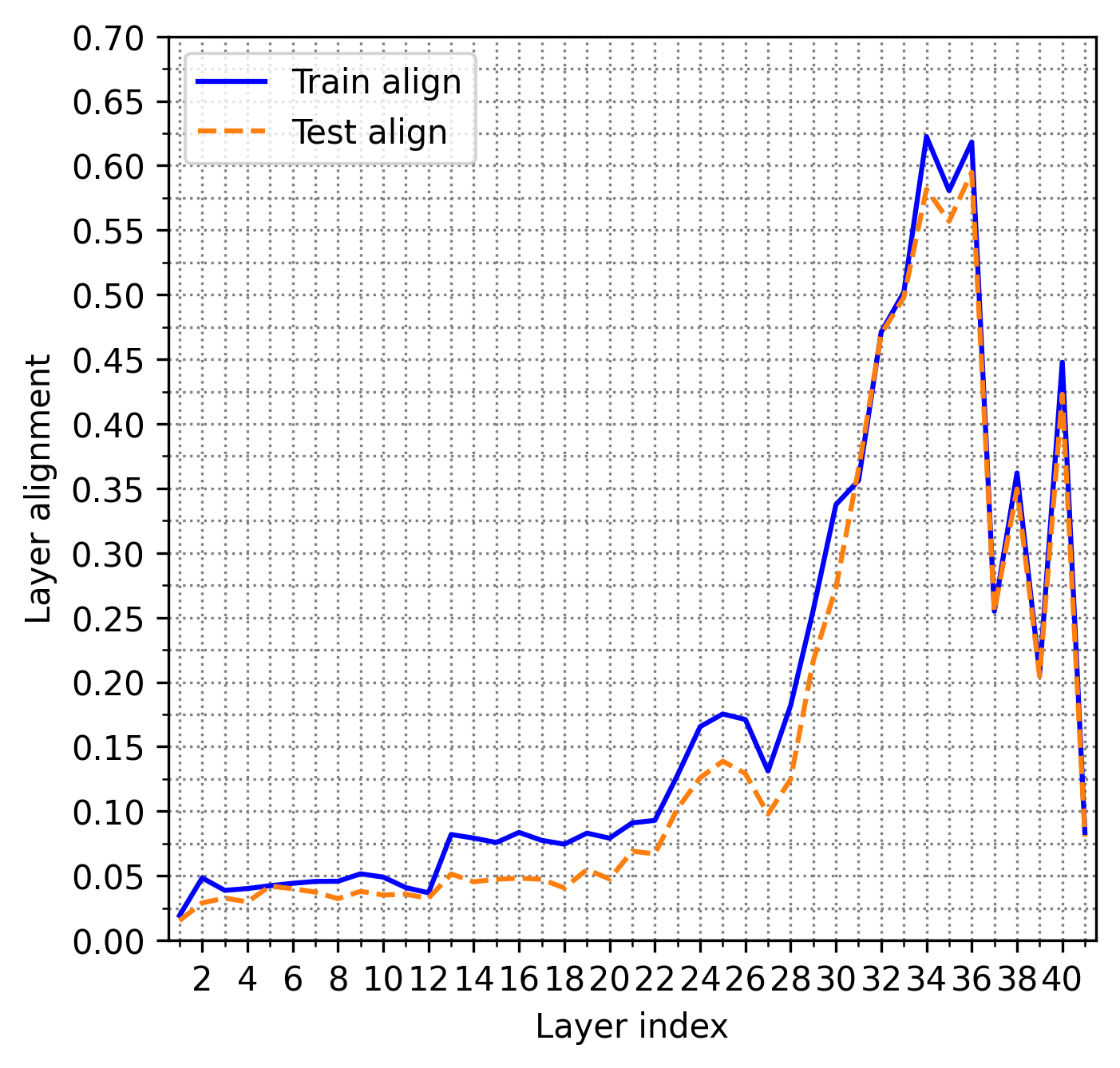}
  \caption{lr 0.02, Acc Gap 5.5\%, Loss Gap 0.182}
\end{subfigure}
\caption{(a), (b) and (c) show ResNet18 trained with SGD (with momentum and weight decay) until convergence (train loss reaches 0.1) epochs on CIFAR10 dataset with four different learning rates (lr).}\label{fig:cifar10_vgg19_hierarchy_generalisation_bs_adam}
\end{figure}

Details of experiments in \cref{fig:EH}:
Feed forward Neural Network on CIFAR10, Fashion MNIST and MNIST using SGD optimizer with momentum and weight decay are included in \cref{tab:cifar10_eh}, \cref{tab:fmnist_eh} and \cref{tab:mnist_eh} resp. The learning rates are chosen as the one that produces best out of sample accuracy. 

\begin{table}[H]
    \centering
    \begin{tabular}{C{1cm} C{1cm} C{2.5cm} C{1cm} C{3cm}}
         depth & width & learning rate & epochs & test accuracy \\
         \hline
         10 & 256 & 0.005 & 100 & 56$\%$\\
         20 & 256 & 0.003 & 100 & 58.3$\%$\\
         30 & 256 & 0.003 & 150 & 57.7$\%$\\
         40 & 256 & 0.001 & 200 & 58.7$\%$\\
         50 & 256 & 0.001 & 250 & 58$\%$\\
         60 & 256 & 0.0007 & 300 & 58.3$\%$\\
         70 & 256 & 0.0005 & 300 & 59.1$\%$\\
         80 & 256 & 0.0005 & 500 & 57.2$\%$\\
         90 & 256 & 0.0001 & 500 & 56.9$\%$\\
         100 & 256 & 0.0001 & 700 & 56$\%$\\
         \hline
    \end{tabular}
    \caption{CIFAR10 FFNN experiments to verify EH (\cref{fig:CIFAR_35}).}
    \label{tab:cifar10_eh}
\end{table}

\begin{table}[H]
    \centering
    \begin{tabular}{C{1cm} C{1cm} C{2.5cm} C{1cm} C{3cm}}
         depth & width & learning rate & epochs & test accuracy \\
         \hline
         10 & 100 & 0.003 & 100 & 88.3$\%$\\
         20 & 100 & 0.004 & 100 & 88.6$\%$\\
         30 & 100 & 0.004 & 100 & 89.6$\%$\\
         40 & 100 & 0.002 & 100 & 88.9$\%$\\
         50 & 100 & 0.001 & 100 & 88.4$\%$\\
         60 & 100 & 0.0007 & 100 & 87.7$\%$\\
         70 & 100 & 0.0003 & 200 & 88.7$\%$\\
         80 & 100 & 0.0002 & 200 & 87.9$\%$\\
         90 & 100 & 0.0001 & 300 & 87.6$\%$\\
         100 & 100 & $7\times10^{-5}$ & 300 & 88$\%$\\
         \hline
    \end{tabular}
    \caption{FMNIST FFNN experiments to verify EH (\cref{fig:FMNIST_35}).}
    \label{tab:fmnist_eh}
\end{table}

\begin{table}[H]
    \centering
    \begin{tabular}{C{1cm} C{1cm} C{2.5cm} C{1cm} C{3cm}}
         depth & width & learning rate & epochs & test accuracy \\
         \hline
         10 & 100 & 0.003 & 100 & 97$\%$\\
         20 & 100 & 0.003 & 100 & 97$\%$\\
         30 & 100 & 0.003 & 100 & 96.9$\%$\\
         40 & 100 & 0.002 & 100 & 97.6$\%$\\
         50 & 100 & 0.001 & 100 & 97.6$\%$\\
         60 & 100 & 0.0007 & 100 & 97.6$\%$\\
         70 & 100 & 0.0002 & 200 & 94.8$\%$\\
         80 & 100 & 0.0001 & 200 & 94.4$\%$\\
         90 & 100 & 0.0002 & 300 & 96.1$\%$\\
         100 & 100 & 0.0001 & 300 & 95.8$\%$\\
         \hline
    \end{tabular}
    \caption{MNIST FFNN experiments to verify EH (\cref{fig:MNIST_35}).}
    \label{tab:mnist_eh}
\end{table}

\section{Layer-wise alignment of the forward feature kernel}

\begin{algorithm}\label{alg:dbm}
\caption{Layer-wise maximisation of features}
\begin{algorithmic}\label{alg:mainexp}
\STATE \textbf{input:} DNN $N$ with $L$ layers, LeakyReLU activations $\phi$, stochastic optimiser $O$, batch size $b$.
\STATE \textbf{input:} Training dataset $S=\{(x_1,y_1), \ldots, (x_{n},y_n)\}$ and validation set $V=\{(x_1,y_1), \ldots, (x_{n'},y_{n'})\}$ with normalised $x_i$ (such that $||x_i||^2 = 1$).
\FOR{layers $l=1,\dots,L-1$}
    \STATE Normalise inputs to $l$, such that $||\phi(z_{l-1}(x))||^2=1$.
    \WHILE{True}
        \FOR{Minibatches $B$ in $S$}
            \STATE With optimiser $O$, update weights and biases in layer $l$ using loss function
            $
            L(B) = \sum_{x_i,x_j \in B}||\overrightarrow{K}_l(x_i,x_j)-\frac{1}{2}\delta_{y_i,y_j}||^2,
            $
            where $\delta_{y_i,y_j}=1$ if $y_i=y_j$ else 0, and the unnormalised forward features, $\overrightarrow{K}_l(x_i,x_j) = \phi(z_{l-1}(x_i)) \cdot \phi(z_{l-1}(x_j))$.
        \ENDFOR
        \STATE End while based on increase/plateau of loss on the validation set, $L(V)$.
    \ENDWHILE
    \STATE Normalise layer $l$. Return $\phi(z_{l}(x))$ as input for layer $l+1$
\ENDFOR
\STATE Train layer $L$ (the final classification layer) with optimiser $O$ and cross entropy loss.
\STATE Return $N$.
\end{algorithmic}
\end{algorithm}

Previous work has studied layer-wise training of neural networks. Here, we adapt a method from \cite{kulkarni2017layer}, which aims to maximise forward feature learning. Here, we confirm that our adapted method generalises well, and show the resulting layerwise CKA on CIFAR10 and MNIST (see \cref{fig:lfm_train_mnist,fig:lfm_train_cifar}). We will call the algorithm layer-wise feature maximisation (LFM) -- see below. Finally, in \cref{fig:frozen}, we show the layerwise CKA for networks with all but the last layer frozen during training. We find non-trivial CKA evolution, due to evolution of the backward feature kernel (despite trivial forward features).

\begin{figure}[H]
\centering

\begin{subfigure}{0.32\linewidth}
  \includegraphics[width=\linewidth]{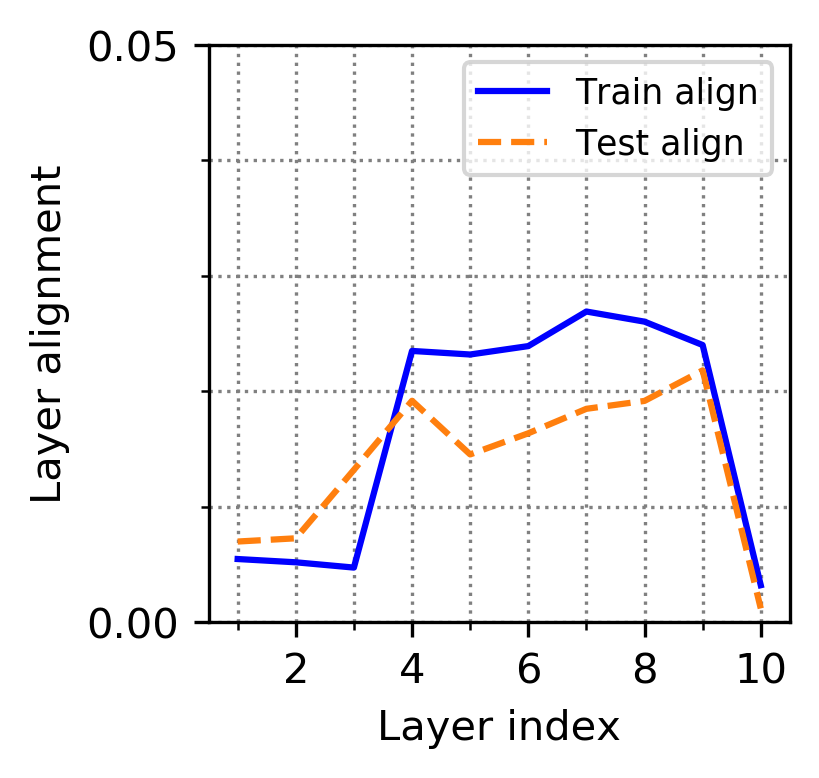}
  \caption{LFM width 32. Acc $93.6\%$}
\end{subfigure}
\begin{subfigure}{0.32\linewidth}
  \includegraphics[width=\linewidth]{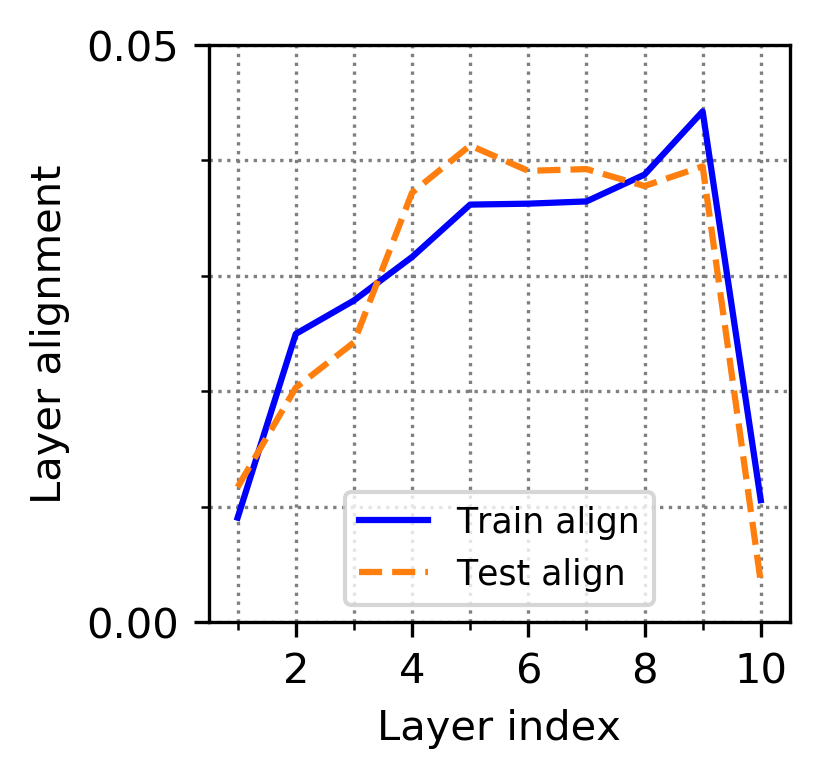}
  \caption{LFM width 256. Acc $95.9\%$}
\end{subfigure}
\begin{subfigure}{0.32\linewidth}
  \includegraphics[width=\linewidth]{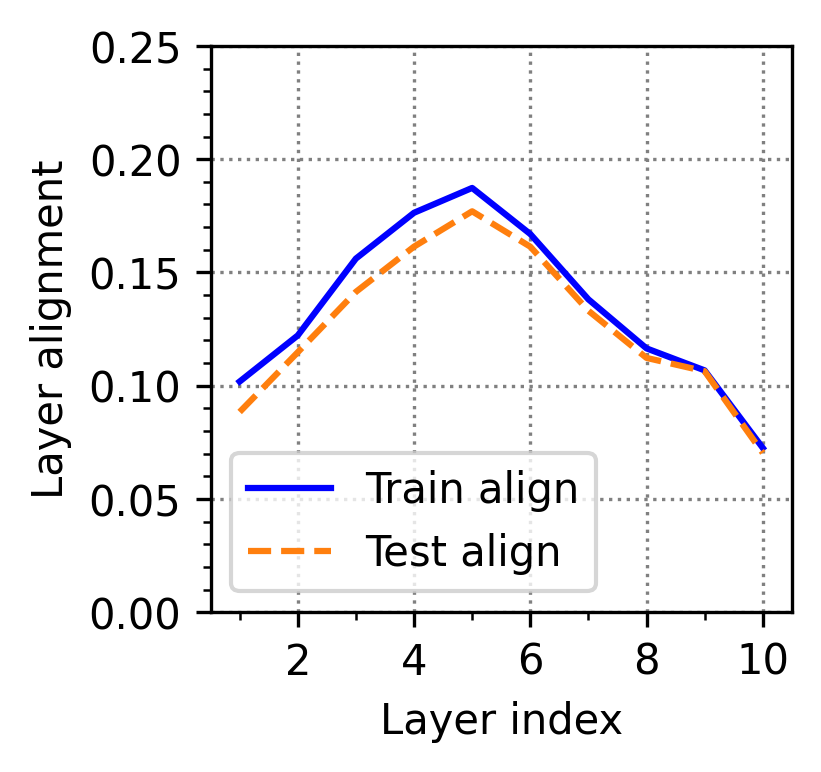}
  \caption{width 32. Acc $96.9\%$}
\end{subfigure}

\caption{MNIST. (a) and (b) are trained with LFM (widths 32 and 256 respectively) with ADAM and a learning rate of 0.01, and (c) trained end-to-end with learning rate 0.003. (a) and (b) used a train/validation set split of 45000/5000, and (c) used 50000 training images with no validation set.
The LFM does not produce a single peak in the same way an end-to-end trained network does (rather, several layers have approximately maximal alignment). The absolute magnitudes of alignment are also lower for the LFM -- determining whether this is an artefact of the layer-wise normalisation scheme or otherwise is a topic of future work.
Furthermore, as with experiments on CIFAR10 (\cref{fig:lfm_train_cifar}), the LFM underperforms relative to the end-to-end neural network. More sophisticated early stopping schemes or different optimisers (Adam was used as other optimisers struggled to converge) may improve generalisation.
}\label{fig:lfm_train_mnist}
\end{figure}

\begin{figure}[H]
\centering

\begin{subfigure}{0.32\linewidth}
  \includegraphics[width=\linewidth]{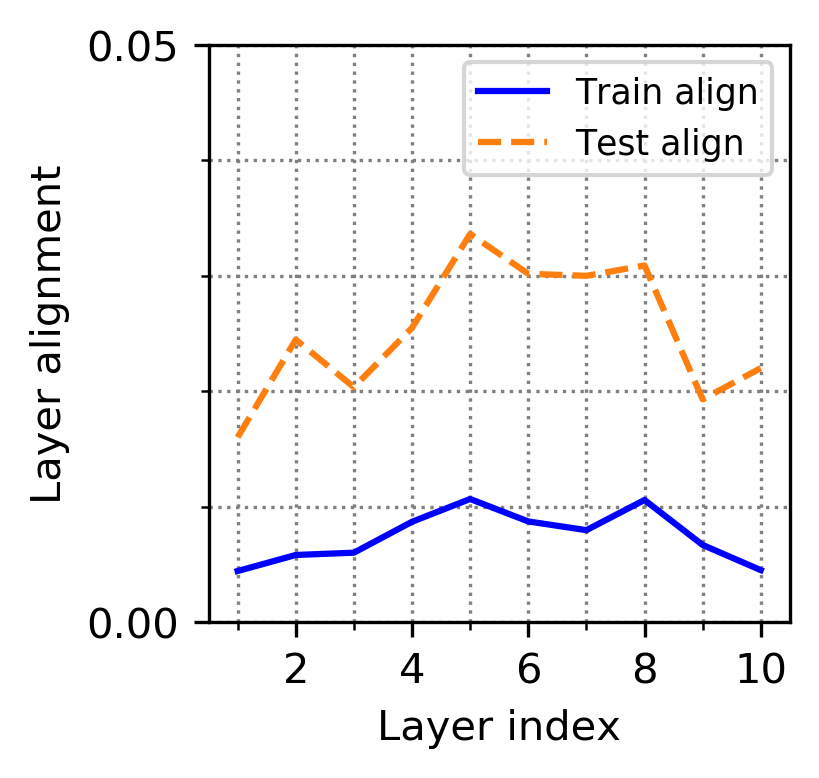}
  \caption{LFM width 256. Acc $48.3\%$}
\end{subfigure}
\begin{subfigure}{0.32\linewidth}
  \includegraphics[width=\linewidth]{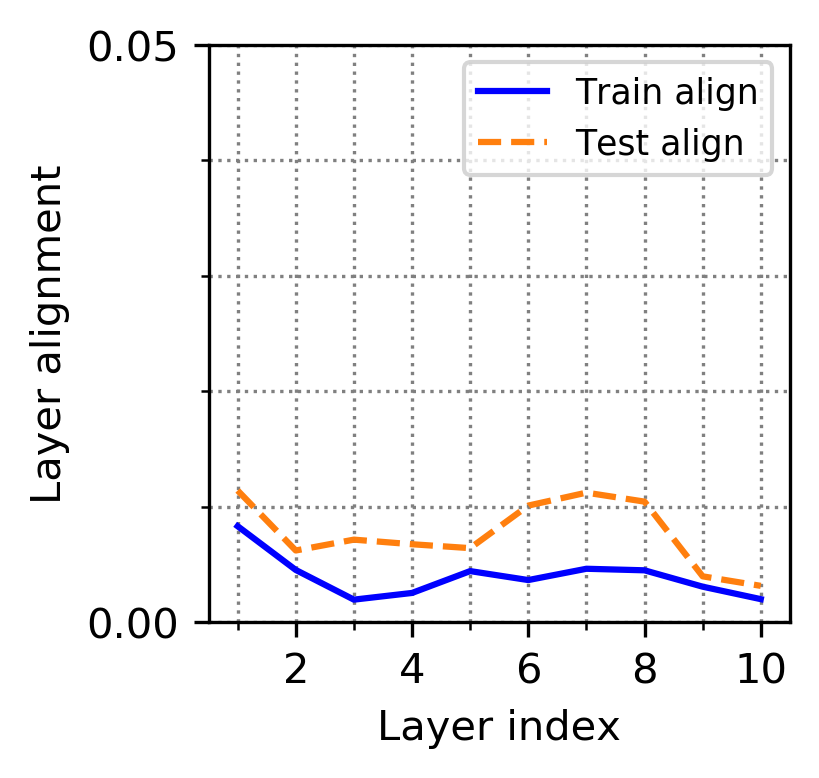}
  \caption{LFM width 1024. Acc $49.2\%$}
\end{subfigure}
\begin{subfigure}{0.32\linewidth}
  \includegraphics[width=\linewidth]{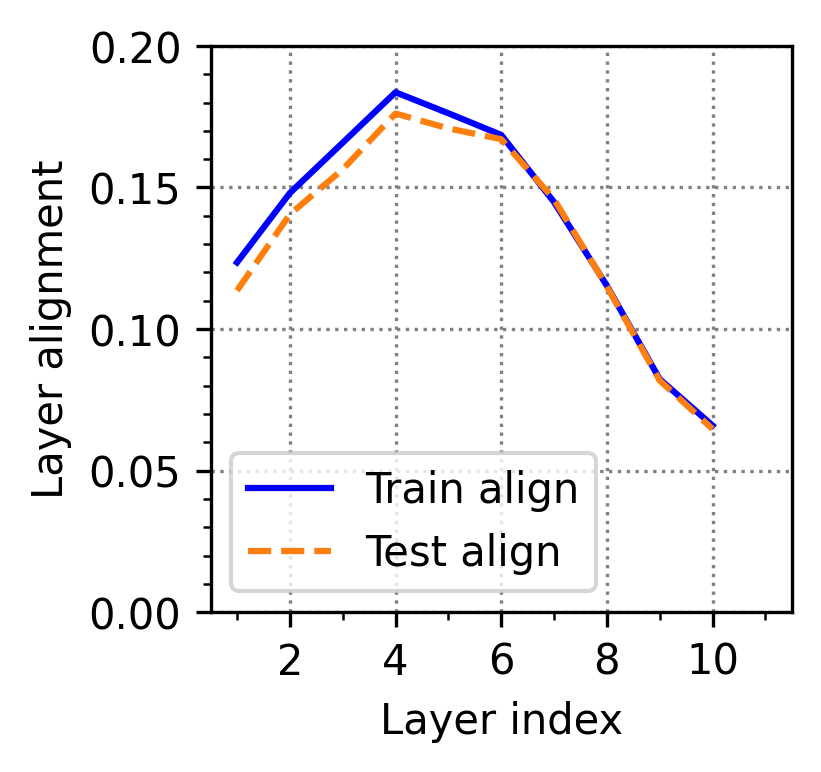}
  \caption{width 256. Acc $56.7\%$}
\end{subfigure}

\caption{CIFAR10. (a) and (b) are trained with LFM (widths 256 and 1024 respectively) with ADAM and a learning rate of 0.01, and (c) trained end-to-end with learning rate 0.003. (a) and (b) used a train/validation set split of 45000/5000, and (c) used 50000 training images with no validation set. The alignment for the LFMs is very different to the end-to-end trained system, and the generalisation error is noticeably worse. Understanding why, or improving results is a topic of future work. }\label{fig:lfm_train_cifar}
\end{figure}

\begin{figure}[H]
\centering

\begin{subfigure}{0.32\linewidth}
  \includegraphics[width=\linewidth]{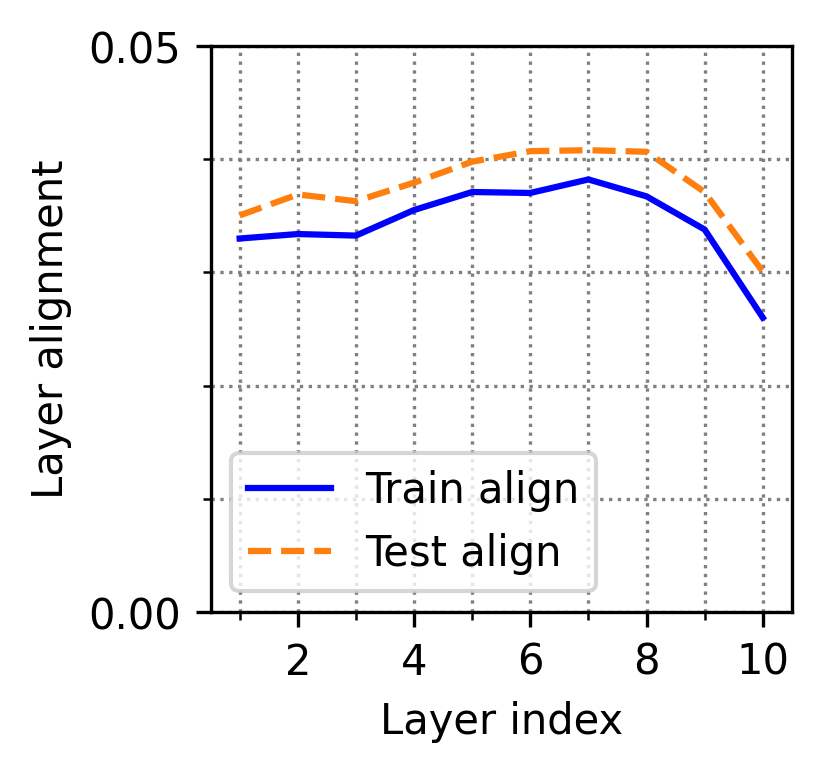}
  \caption{Width 2048. Noise scale: 0.01}
\end{subfigure}
\begin{subfigure}{0.32\linewidth}
  \includegraphics[width=\linewidth]{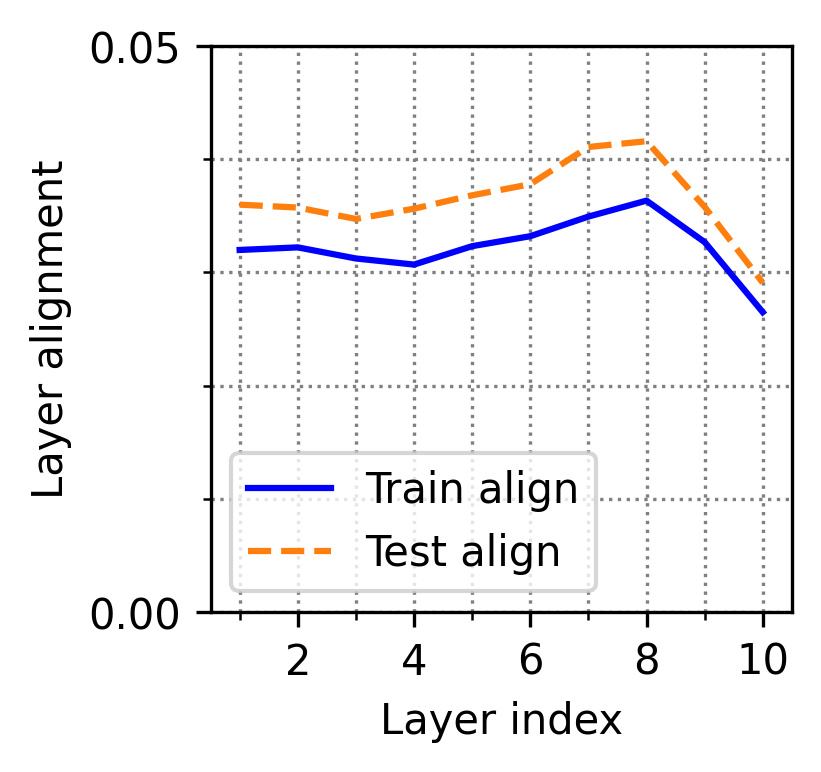}
  \caption{Width 2048. Noise scale: 1.0}
\end{subfigure}

\caption{MNIST. All but the last layer is frozen during training (so no feauture learning occurs). 
In the limit of infinite width, this is equivalent to sampling from an NNGP \citep{Matthews2017}. 
With 10 layers, 2048 is not sufficiently wide to obtain 100\% training accuracy (see \cref{tab:mnist}), but computing the CKA for each layer scales poorly with layer width.
(a) and (b) achieve test accuracies of 91.9\% and 88.1\% respectively (due to comparatively small layer width).
Noise scale determines the scale of the initialisation of the final layer -- parameters are sampled i.i.d.\ from $N(0,s^2\times 2/L_w)$ (for noise scale $s$), where $L_w$ is the width of the layer.
Decreasing the amount of noise in the last layer appears to shift the peak towards the center, although more careful study is required.
Due to the frozen layers, no forward features are learned (so $\protect\overrightarrow{K}_l(x_i,x_j)$ is trivial), but evidently backward features $\protect\overleftarrow{K}_l(x_i,x_j)$ are non trivial. This is unlike neural networks trained end-to-end, which have both non-trivial forward and backward feature kernels.}\label{fig:frozen}
\end{figure}

\begin{table}[H]
    \centering
    \begin{tabular}{c c c c c}
        \hline
         Dataset & Number of layers & Width & Max train acc & Max test acc \\
         \hline
         MNIST & 10 & 75000 & 97.6\% & 96.4\% \\
         MNIST & 10 & 10000 & 96.4\% & 95.6\% \\
         MNIST & 10 & 2048 & 92.6\% & 91.9\% \\
         \hline
    \end{tabular}
    \caption{Best train/test accuracy for frozen neural networks as a function of layer widths. Network parameters are initialised i.i.d.\ from  $N(0,2/L_w)$ where $L_w$ is the width of the layer. Clearly, layers have to be very wide before near parity can be achieved with finite width unfrozen networks.}
    \label{tab:mnist}
\end{table}

\section{Connecting CKA with Effective Rank, NTK, and Fisher information}\label{app:fim}

In the following appendix, we will summarize work related to the effective rank of $\Psi$ and their connection to CKA $A$, and provide some related novel experiments. \cite{baratin2021implicit,oymak2019generalization} independently observed that the following phenomenon holds when DNNs generalize:
\begin{enumerate}
  \itemsep 1em 
  \item $\Psi$ has a small number of large singular values while most other singular values are much smaller.
  \item The label vector $Y$ is aligned with large singular directions in $\Psi$.
\end{enumerate}
The two conditions above are also the conditions for maximizing CKA. As was argued in \cite{baratin2021implicit}, the CKA was introduced as the measure of model compression and feature selection. The first condition can be viewed as the measure of model compression, as effectively only a few directions in parameter space are relevant for changes of the function. The second condition can be interpreted as feature selection because we want the anisotropy of the tangent space to be skewed toward directions that leads to correct labels in function space. 

We will first show that above observations still hold for layerwise CKA in \cref{app:G1}. Then, in \cref{app:G:FIM}, we shed more light into the connection between the two conditions and the generalization via the Fisher information matrix. 

%We will summarize the results empirically and make links to the empirical Fisher Information Matrix $I(\theta)$ and flatness arguments.

%\cite{oymak2019generalization, li2020gradient} have previously studied the connection between the rank of $\Psi$ and generalization. Theoretical arguments and empirically evidence in \cite{oymak2019generalization} suggest that DNNs trained to low loss and generalise well exhibit the following properties: 
%We can interpret these two conditions on the singular values of $\Psi$ in the following way: the outputs of the neural network $f(x;\theta)$ are more robust to perturbations in the parameters. However perturbation in $f(x;\theta)$ may not always lead to perturbation in likelihood, and the robustness of prediction on perturbation of parameters is best understood via the Fisher information matrix.

%Finally, we interpret this result using the Fisher Information Matrix in \cref{app:G:FIM}.

\subsection{Connection between kernel alignment $A$ and the Effective Rank of $\mathbf{\Psi}^T\mathbf{\Psi}$}\label{app:G1}
Let us begin with observation 1: that only a handful of eigenvalues are large. As shown in \cref{fig:eigs_fcn}, the eigenvalues differ in logarithmic scale for both functions that do and do not generalize. Even though both eigenvalue distributions are spread on a logarithmic scale, the magnitude of the spread is different. To quantify this, we introduce the stable rank \citep{rudelson2007sampling}.\footnote{Other measures of effective rank (e.g. \cite{roy2007effective}) are also valid for our study as long as it can represent the exponential gap in eigenvalues.}

\begin{figure}[H]
\centering

\begin{subfigure}{0.4\linewidth}
  \includegraphics[width=\linewidth]{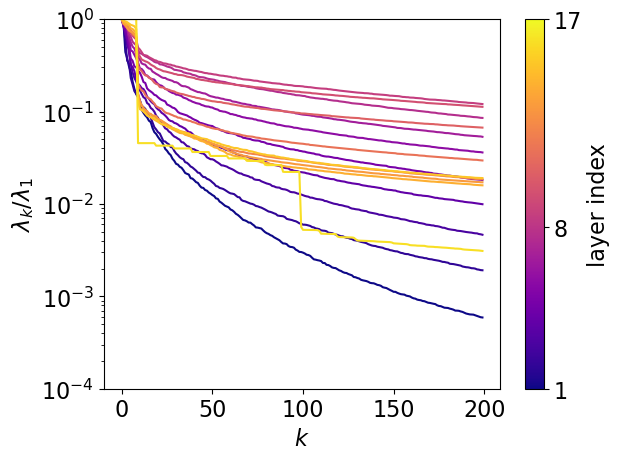}
  \caption{Trained on 20\% random labels, test accuracy = 70.41\%}
  \label{fig:eigs_fcn_bad}
\end{subfigure}
\begin{subfigure}{0.4\linewidth}
  \includegraphics[width=\linewidth]{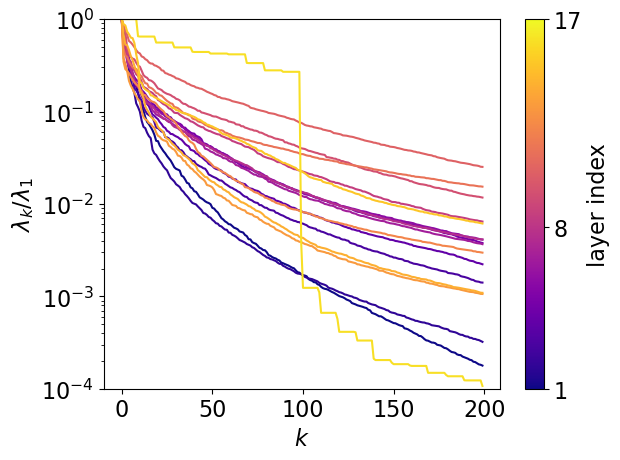}
  \caption{Trained without random labels, test accuracy = 90.59\%}
  \label{fig:eigs_fcn_good}
\end{subfigure}

\caption{The top 200 normalized eigenvalues of $C\mathbf{\Psi_l}^T\mathbf{\Psi_l}C$ for different layers $l$. VGG19 was trained on CIFAR10 (a) with 20\% random labels and (b) without random labels. The relative eigenvalues decay on logarithmic scale for both cases. The logarithmic gap in eigenvalues assert that $C\mathbf{\Psi_l}^T\mathbf{\Psi_l}C$ has low effectively rank. On average, the eigenvalue drops faster for generalizing case (b) compared to less generalizing case (a).}
\label{fig:eigs_fcn}
\end{figure}

For a matrix $W$ of rank $k$, the stable rank is 
\begin{align*}
R(W) = \frac{|| W ||^2_F}{||W||^2_2} = \frac{\sum_i \lambda^2_i}{\lambda_1^2},
\end{align*}
where the numerator and denominator are squares of the Frobenius norm and spectral norm respectively. Singular values $\lambda_i$ are given in descending orders measured by their absolute values. The stable rank is scale invariant and upper bounded by the true rank. Note that stable rank becomes 1 when $\lambda_1 >> \lambda_j$ for all $j > 1$.

Using the definition of stable rank, layerwise CKA can be written as
\begin{align*}
A_l &= \frac{\tilde{Y}^T C\mathbf{\Psi_l}^T\mathbf{\Psi_l}C \tilde{Y}}{||C\mathbf{\Psi_l}^T\mathbf{\Psi_l}C||_F} = \frac{\lambda_1}{||C\mathbf{\Psi_l}^T\mathbf{\Psi_l}C||_F} \frac{\tilde{Y}^T C\mathbf{\Psi_l}^T\mathbf{\Psi_l}C \tilde{Y}}{\lambda_1} \\
&= \frac{1}{\sqrt{R_l}} \left( \sum_i \frac{\lambda_i}{\lambda_1} \left<u_i,\tilde{Y}\right>^2\right),
\end{align*}
where $R_l$ is the stable rank of $C\mathbf{\Psi_l}^T\mathbf{\Psi_l}C$ and $\lambda_i, u_i$ are its eigenvalue and eigenvector respectively. We have assumed $\tilde{Y}$ has been normalized. 
Then, layerwise CKA can be divided into two terms: inverse square root of stable rank of $C\mathbf{\Psi_l}^T\mathbf{\Psi_l}C$ and the weighted correlation between the $u_i$ and $\tilde{Y}$. Each term is related to the observation 1 and 2 respectively.

Notice that $A_l$ is maximized when $R_l$ is minimized and the correlation term $\left( \sum_i \lambda_i/\lambda_1 \left<u_i,\tilde{Y}\right>^2 \right)$ is maximized. This result is trivial because CKA is maximizing the alignment with $\tilde{Y} \tilde{Y}^T$, which is a rank 1 matrix with $\tilde{Y}$ as its eigenvector. 

We observe in \cref{fig:ka_rank_rank} that the stable rank is consistently small (i.e. concentrated eigenvalues) for most of the layers (this relates to Observation 1). In addition, the stable rank is smaller when better generalization gap is achieved. This could be interpreted as signifying that only a few directions in parameter space can meaningfully alter the function, hence a function is more robust against perturbation in parameters, and effectively lower dimensional.

\subsection{Connection between kernel alignment $A$ and the Correlation $\sum_i \lambda_i/\lambda_1 \left<u_i,\tilde{Y}\right>^2$}\label{app:G2}
\begin{figure}[H]
\centering

\begin{subfigure}{0.32\linewidth}
  \includegraphics[width=\linewidth]{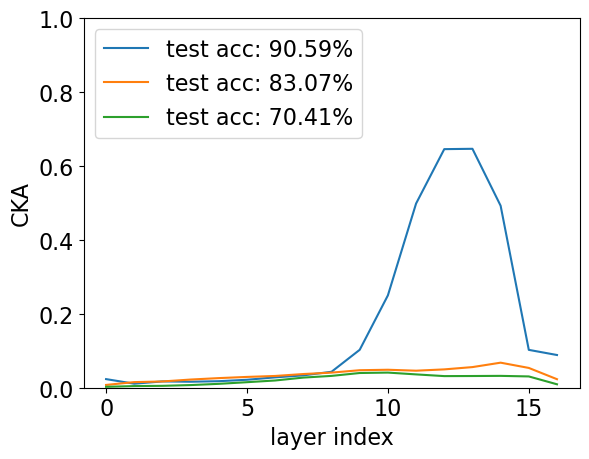}
  \caption{CKA}
  \label{fig:ka_rank_ka}
\end{subfigure}
\begin{subfigure}{0.32\linewidth}
  \includegraphics[width=\linewidth]{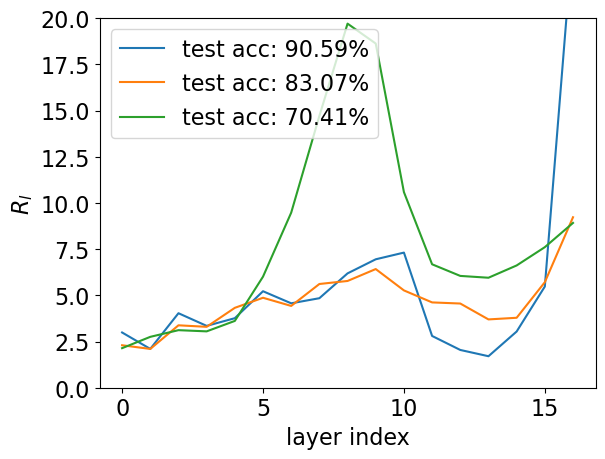}
  \caption{Stable rank}
  \label{fig:ka_rank_rank}
\end{subfigure}
\begin{subfigure}{0.32\linewidth}
  \includegraphics[width=\linewidth]{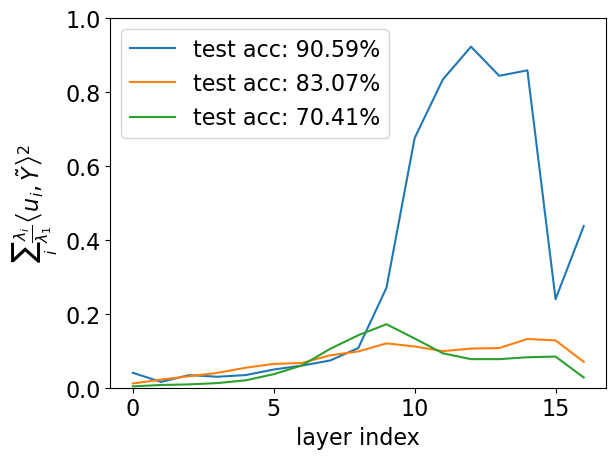}
  \caption{Correlation}
  \label{fig:ka_rank_corr}
\end{subfigure}

\caption{Plots for (a) CKA $A_l$, (b) stable rank $R_l$, and (c) correlation $\sum_i \frac{\lambda_i}{\lambda_1} \left<u_i,\tilde{Y}\right>^2$ after training. VGG19 was trained on CIFAR10 with following conditions to achieve different test accuracy: (blue) 90.59\%, lr=0.02; (orange) 83.07\%, lr=0.001; (green) lr=0.001 and 20\% random labels. (a) Higher test accuracy correlates with larger layer-wise CKA, as suggested by \cref{fig:egs_main} in the main text. (b) The stable rank is on average larger for a model trained with the random labels (green), while smaller for models trained without the random labels (blue, orange). This suggests that low rank is the first condition that must be satisfied for good generalization. The importance of low rank can be seen from the fact that it implies robustness to perturbation in parameters. (c) The correlation is large only for best generalising model (blue), and is small otherwise. This suggests that large correlation may be a second condition for generalization. A large correlation indicates that the only allowed deviation in function space is along $\tilde{Y}$. 
%$\tilde{Y}$ increases (decreases) all wrong labels for all datapoints uniformly.
}
\label{fig:ka_rank}
\end{figure}

As mentioned before, the correlation term measures how much $\tilde{Y}$ is aligned with the large eigenspace of $C\mathbf{\Psi_l}^T\mathbf{\Psi_l}C$: the quantitative measurement for observation 2. The correlation must also be large in order for $A_l$ to be large. When such is the case, we may approximate
\begin{align}
\Psi_l \approx v_l \sqrt{\lambda_1}\tilde{Y}^T,
\end{align}
and thus the effect of perturbation in $\theta$ on $f$,
\begin{align}
\Delta f^T \approx \left<\Delta \theta, v_l\right> \sqrt{\lambda_1}\tilde{Y}^T.
\end{align}
Thus, any infinitesimal change in parameter space for most directions cannot alter $f$. The only allowed direction of change is along $\tilde{Y}$. The direction along $\tilde{Y}$ is special because $f$ cannot be effectively changed to increase the probability of an incorrect label. %$f$ must be moved between incorrect labels, and all datapoints in order to create mis-classification. 
This combined with observation 1 suggests that most directions in parameter space are robust against perturbation, and the only meaningful direction of change is along the direction that uniformly increases the function values of misclassified labels.

The correlation term for models of different generalization can be seen from \cref{fig:ka_rank_corr}. Notice that the correlation is small for all of the layers for the two worse performing models (orange, green). However, for the model with the best generalization error (blue), the correlation increases to near 1 in the intermediate layers. The reason for this alignment hierarchy (due to the training process) was explained in \cref{sec:EH} of the main text.

\subsection{Fisher information matrix}\label{app:G:FIM_intro}

Fisher information is the metric tensor of a statistical manifold \citep{amari2000methods}, and it provides the local measure of how fast a prediction of model changes according to change in parameters.
Because DNNs are commonly trained using the gradients of the negative log likelihood (NLL), and not the gradient of the functions $f_\theta$\footnote{For MSE loss, the two become the same}, Fisher information becomes a natural choice for calculating the flatness of different hypotheses. 
Fisher information $I^{exp}(\theta)$ is 
\begin{align*}
I^{exp}_{ij}(\theta) &= \mathbf{E}_{x \sim q(x)} \left[\mathbf{E}_{y \sim p(y|x;\theta)} \left[ \left. \frac{\partial \log(p(y|x;\theta))}{\partial \theta_i} \frac{\partial \log(p(y|x;\theta))}{\partial \theta_j} \right|_{\theta} \right]\right],
\end{align*}
where $q(x)$ is the true distribution of the inputs, and $p(y|x;\theta)$ is the conditional probability predicted by a statistical model at $\theta$. 
%As a metric, Fisher information measures of how much $p(y|x;\theta)$ will change according to a infinitesimal deviation in parameters. 
To see why it is a measure of sensitivity of the likelihood along the parameters, Fisher information can be equivalently expressed as
\begin{align*}
I^{exp}_{ij}(\theta) &= -\mathbf{E}_{x \sim q(x)} \left[\mathbf{E}_{y \sim p(y|x;\theta)} \left[ \frac{\partial^2 \log(p(y|x;\theta))}{\partial \theta_i\partial \theta_j} \right]\right] = \left. \frac{\partial^2 D_{KL}\left(p(y|x;\theta)||p(y|x;\theta + \Delta\theta)\right)}{\partial \theta_i\partial \theta_j} \right|_{\Delta\theta=0},
\end{align*}
which is the curvature of KL divergence at $\theta$. Thus, quantifying the second order changes of $p(y|x;\theta)$ (i.e. large eigen direction of $I^{exp}_{ij}(\theta)$ leads to sharper change of $p(y|x;\theta)$, while small eigen direction leads to flatter changes). This quantity is more useful near the maximum likelihood estimate when the first order deviation terms disappear.

For our purposes, we will only consider the empirical fisher information
\begin{align*}
I(\theta) = \mathbf{\Psi} \frac{\partial ^2 \mathcal{L}}{\partial F^2} \mathbf{\Psi}^T,
\end{align*}
where $\mathcal{L}$ is the negative log likelihood, and $F \in \reals^{o \times n}$ is the concatenated network output. We have dropped the centering matrix $C$ following the observation from \cite{baratin2021implicit} that quantitatively similar results were obtained for CKA and KA. For the special case when $\mathcal{L}$ is MSE loss\footnote{MSE loss is used for NTK, and can be interpreted as the likelihood being a Gaussian distribution with fixed covariance.},
\begin{align*}
\frac{\partial ^2 \mathcal{L}}{\partial F_\alpha \partial F_\beta} = \frac{\partial ^2 \sum_i^{on}(F_i - y_i)^2}{\partial F_\alpha \partial F_\beta} = \delta_{\alpha\beta},
\end{align*}
and $I(\theta) = \mathbf{\Psi}\mathbf{\Psi^T}$. Therefore, $I(\theta)$ and NTK share same set of non-zero eigenvalues for MSE loss. For more general discussion beyond the empirical case, see Appendix A of \cite{baratin2021implicit}.

\subsection{Fisher Information and Generalization via the Stable Rank}\label{app:G:FIM}

It has been argued that "flatness" (albeit not being the sole factor) is related to good generalization \citep{keskar2016large, wu2017towards, neyshabur2017pac}. The flatness is calculated by the Hessian, which is equal to the Fisher information for MLE. Thus, measuring the stable rank of Fisher information can be a measure of the flatness of the model. Flatness of $f$ may be inadequate, because sharpness on $f$ does not always lead to sharpness on $\mathcal{L}$ (e.g. when the softmax function is saturated).

We will explore the case of cross-entropy loss to see what CKA can infer about the flatness of the loss landscape. For such settings,
\begin{equation*}
\frac{\partial ^2 \mathcal{L}}{\partial F \partial F} = 
\begin{pmatrix}
\frac{\partial ^2 \mathcal{L}}{\partial f(x_1) \partial f(x_1)}&  &  & \\
& \frac{\partial ^2 \mathcal{L}}{\partial f(x_2) \partial f(x_2)} &  &\\
& & \ddots &\\
& & &\frac{\partial ^2 \mathcal{L}}{\partial f(x_n) \partial f(x_n)}\\,
\end{pmatrix} 
\end{equation*}
where 
\begin{align*}
\frac{\partial ^2 \mathcal{L}}{\partial f^i_\theta(x) \partial f^j_\theta(x)} = \delta_{ij}p_i(x) -p_i(x)p_j(x),
\end{align*}
and $p_i(x)$ is  
\begin{align*}
p_i(x) = p(y = i | x;\theta) = \frac{e^{f^i_\theta(x)}}{\sum_j e^{f^j_\theta(x)}}.
\end{align*}

$\frac{\partial ^2 \mathcal{L}}{\partial F \partial F}$ is clearly different from the identity matrix and approaches 0 as the training converges. However, it is still possible to infer the properties of Fisher information from $\mathbf{\Psi}$.

By the property of rank under matrix multiplication, $\textrm{rank}(\mathbf{\Psi} \frac{\partial ^2 \mathcal{L}}{\partial F^2} \mathbf{\Psi}^T) \leq \textrm{rank}(\mathbf{\Psi^T}\mathbf{\Psi})$. Even though the stable rank does not strictly satisfy this condition, we can infer that the stable rank of Fisher information would still be similar or smaller than that of $\mathbf{\Psi^T}\mathbf{\Psi}$. In \cref{fig:fim_rank_comparison}, the rank of Fisher information indeed is observed to be upper bounded by the rank of $\mathbf{\Psi^T}\mathbf{\Psi}$. More surprisingly, the stable rank of Fisher information is consistently small across all layers, which is left as a future direction of investigation.

As seen in \cref{fig:fim_rank_all}, the stable rank of Fisher information not only is upper bounded by the stable rank of CKA, but loosely follows the trend of
CKA. In addition, generalizing models leads to smaller stable rank of Fisher information. We can empirically postulate that larger CKA leads to smaller stable rank of CKA, which in turn informs us about the stable rank of Fisher information.

%large increase along the component of $f$ along the correct label changes $f$, but the probability of the correct label is negligibly changed if $f$ along the true label were already large to start with (i.e saturation in softmax function).
\begin{figure}[htp]
\captionsetup[subfigure]{justification=centering}
\centering

\begin{subfigure}{0.32\linewidth}
  \includegraphics[width=\linewidth]{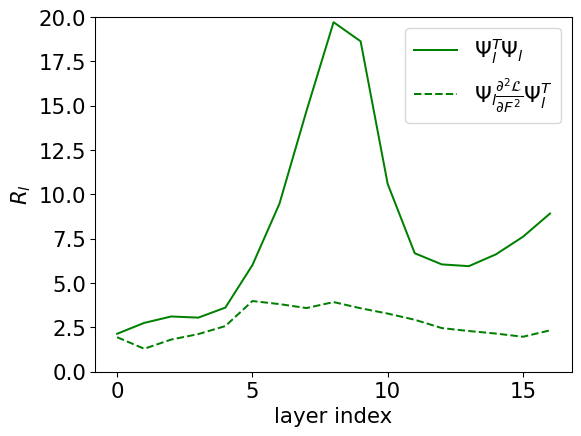}
  \caption{20\% random labels,\\ test accuracy = 70.41\%}
  \label{fig:fim_rank_bad}
\end{subfigure}
\begin{subfigure}{0.32\linewidth}
  \includegraphics[width=\linewidth]{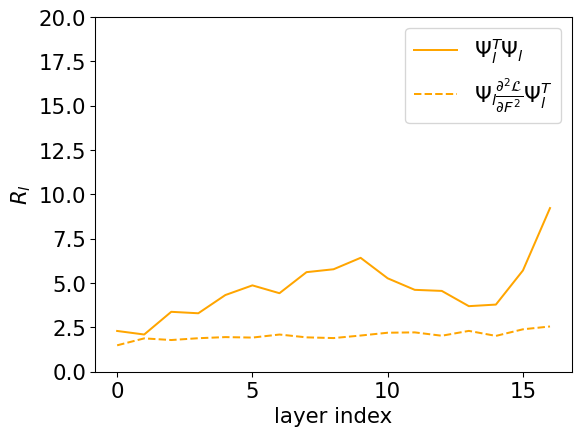}
  \caption{no random labels,\\ test accuracy = 83.07\%}
  \label{fig:fim_rank_neutral}
\end{subfigure}
\begin{subfigure}{0.32\linewidth}
  \includegraphics[width=\linewidth]{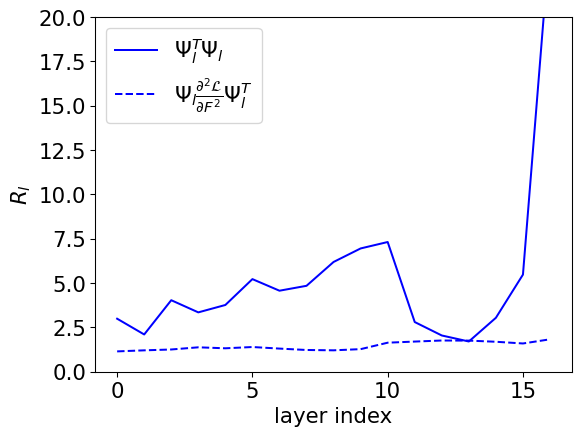}
  \caption{no random labels,\\ test accuracy = 90.59\%}
  \label{fig:fim_rank_good}
\end{subfigure}

\caption{Stable rank $R_l$ for NTK ($\mathbf{\Psi^T}\mathbf{\Psi}$) (solid) and Fisher information $\textrm{rank}(\mathbf{\Psi} \frac{\partial ^2 \mathcal{L}}{\partial F^2} \mathbf{\Psi}^T)$ (dashed) for models trained to different test accuracies. The training condition is equal to that of \cref{fig:ka_rank}, denoted by the same colors. Even though inequality of ranks for matrix multiplication does not hold strictly, it is observed that the rank of NTK upper bounds that of Fisher information. In addition, the stable rank of Fisher information is more consistent over the layers.}
\label{fig:fim_rank_comparison}
\end{figure}

\begin{figure}[H]
\centering

\begin{subfigure}{0.5\linewidth}
  \includegraphics[width=\linewidth]{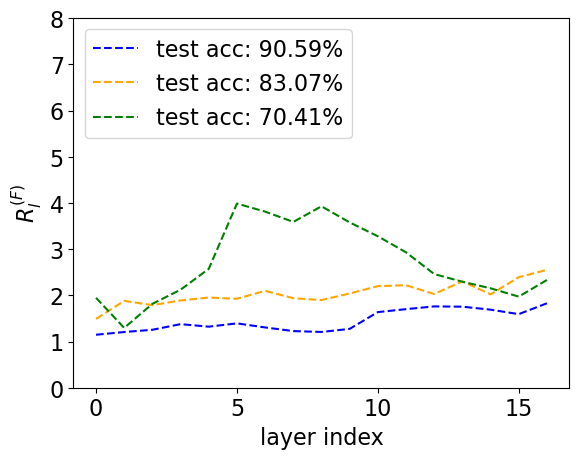}
  \caption{Ranks of Fisher information}
\end{subfigure}

\caption{The comparison of the ranks of Fisher information. Higher test accuracy correlates with smaller stable rank. It can be inferred from the fact that lower stable rank leads to more robustness in the prediction probabilities.}
\label{fig:fim_rank_all}
\end{figure}

\subsection{Fisher Information and Correlation}\label{app:G:FIM_corr}
Even though the empirical relationship between the stable rank of Fisher information and CKA is clear, the correlation term is less straightforward. For the CKA, the correlation is defined between the anisotropy of the tangent space with our direction of interest ($\tilde{Y}$) in function space. We can extend it to measurement of the anisotropy of tangent space introduced by the Fisher information.
First, let us define a square root of Fisher information as
\begin{align*}
\sqrt{I_l} = \sqrt{\frac{\partial ^2 \mathcal{L}}{\partial F \partial F}}\mathbf{\Psi_l} = \sum_i \mu^T_i \sqrt{\lambda^{(F)}_i} \nu_i,
\end{align*}
where $\mu_i$ and $\nu_j$ are left and right singular vectors respectively, and $\lambda^{(F)}_i$ is the corresponding eigenvalue of Fisher information. Then we define the Fisher correlation as
\begin{align*}
&\sum_i \frac{\lambda^{(F)}_i}{\lambda^{(F)}_1} \left<\mu_i,\tilde{Y}\right>^2.
\end{align*}
However, as seen in figure \cref{fig:fim_corr_all}, the correlation is not evident from the experiments. The investigation of why the correlation disappears for Fisher information is left as future work.

\begin{figure}[H]
\centering

\begin{subfigure}{0.5\linewidth}
  \includegraphics[width=\linewidth]{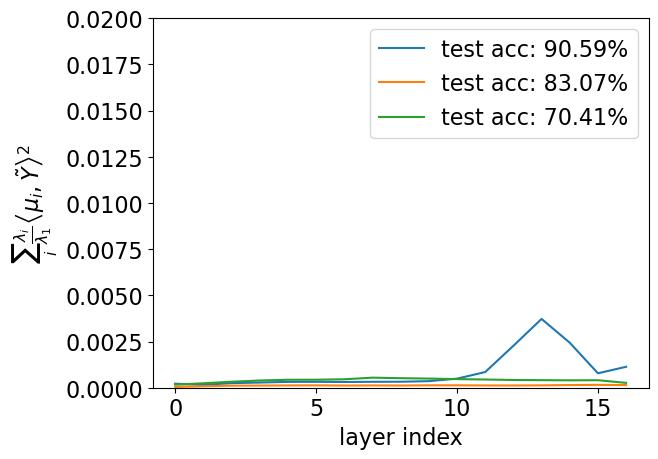}
  \caption{Correlation for Fisher information}
\end{subfigure}

\caption{The comparison of the correlation of Fisher information for models with different accuracies. The experiment conditionas are equal to that of \cref{fig:ka_rank}. The correlation is too small for meaningful argument.}
\label{fig:fim_corr_all}
\end{figure}

\end{document}